\documentclass[11pt, oneside]{article}
\usepackage[utf8]{inputenc}

\usepackage{graphicx}
\usepackage{fancyhdr}
\usepackage{geometry}
\usepackage{amsmath}
\usepackage{tensor}
\usepackage{amssymb}
\usepackage{latexsym}
\usepackage[numbers]{natbib}	
\usepackage{textcomp}
\usepackage{subfig}
\usepackage{multirow}
\usepackage{eufrak}
\usepackage{mathrsfs} 

\usepackage[section]{placeins} 

\usepackage{algpseudocode}

\usepackage{amsthm}
\usepackage{bm}
\usepackage{array}

\usepackage[usenames,dvipsnames]{color}
\usepackage{ifthen}
\usepackage{enumerate}
\usepackage{pgf,tikz}
\usetikzlibrary{arrows}
\DeclareMathAlphabet{\mathpzc}{OT1}{pzc}{m}{it}

\usepackage{tkz-euclide}

\usepackage{hyperref}
\hypersetup{
    bookmarks=true,         
    unicode=false,          
    pdftoolbar=true,        
    pdfmenubar=true,        
    pdffitwindow=false,     
    pdfstartview={FitH},    
    pdftitle={Convex duality for stochastic shortest path problems in known and unknown environments},    
    pdfauthor={Kelli Francis-Staite},     
    pdfsubject={Stochastic Shortest Paths},   
    pdfcreator={KS},   
    pdfproducer={KS}, 
    pdfkeywords={Stochastic Shortest Paths} {Markov Decision Processes} {Convex Duality}, 
    pdfnewwindow=true,      
    colorlinks=true,       
    linkcolor=black,          
    citecolor=black,        
    filecolor=black,      
    urlcolor=black,
    linkcolor=black,          
    citecolor=black,        
    filecolor=magenta,      
    urlcolor=cyan           
}
\usepackage[nameinlink]{cleveref}
\crefname{Program}{{}Program}{{}Programs}

\numberwithin{equation}{section}

\theoremstyle{plain}
\newtheorem{thm}{Theorem}[section]
\newtheorem{corollary}[thm]{Corollary}
\newtheorem{lemma}[thm]{Lemma}
\newtheorem{proposition}[thm]{Proposition}

\newtheorem{conjecture}[thm]{Conjecture}

\theoremstyle{definition}
\newtheorem{defn}{Definition}[section]
\newtheorem{assu}{Assumption}[section]

\theoremstyle{remark}
\newtheorem{example}{Example}[section]

\newtheorem{rmk}{Remark}[section]

\newcommand{\Ac}{\mathcal{A}}
\newcommand{\Sc}{\mathcal{S}}
\newcommand{\Mc}{\mathcal{M}}
\newcommand{\Eb}{\mathbb{E}}

\newcommand{\R}{\mathbb{R}}
\newcommand{\nin}{\notin}
\DeclareMathOperator*{\argmax}{arg\,max}
\DeclareMathOperator*{\argmin}{arg\,min}
\DeclareMathOperator{\CB}{CB}
\DeclareMathOperator{\spn}{span}

\title{Convex duality for stochastic shortest path problems in known and unknown environments}
\author{Kelli Francis-Staite}
\date{\today}

\begin{document}

\maketitle
  \begin{abstract}
This paper studies Stochastic Shortest Path (SSP) problems in known and unknown environments from the perspective of convex optimisation. It first recalls results in the known parameter case, and develops understanding through different proofs. It then focuses on the unknown parameter case, where it studies extended value iteration (EVI) operators. This includes the existing operators used in \citet{rosenbergCohen} and \citet{Tarbouriech2019} based on the $\ell_1$ norm and supremum norm,  as well as defining EVI operators corresponding to other norms and divergences, such as the KL-divergence. This paper shows in general how the EVI operators relate to convex programs, and the form of their dual, where strong duality is exhibited. 

This paper then focuses on whether the bounds from finite horizon research of \citet{neuPikeBurke} can be applied to these extended value iteration operators in the SSP setting. It shows that similar bounds to \cite{neuPikeBurke} for these operators exist, however they lead to operators that are not in general monotone and have more complex convergence properties. In a special case we observe oscillating behaviour. This paper generates open questions on how research may progress, with several examples that require further examination. 
  \end{abstract}
\tableofcontents
\pdfbookmark[1]{Contents}{contents}

\section{Introduction}
This paper focuses on stochastic shortest path problems in known and unknown environments and their relationship to convex duality. Stochastic Shortest Path (SSP) problems are a type of Markov Decision Process (MDP) where agents must trade off reaching a goal state with minimising costs incurred along the way. To do this, agents seek optimal choices of actions (called a \emph{policy}) that depend only on the current state, to move to a new state in the effort to reach a goal state. Each choice incurs a \emph{cost} and agents also seek to minimise their total or expected cost, which is called the \emph{value} of their policy. 
In the SSP setting, an agent may take an unknown and potentially infinite number of actions before reaching the goal state, which means SSPs are considered as a subclass of infinite horizon MDPs.

Shortest paths (SP) problems have been studied as far back as \citet{Wiener1873} in the context of finding paths through mazes. These problems originally considered `deterministic’ actions, where actions have only one known outcome. Such deterministic problems can be solved by dynamic programming techniques, such as Dijkstra’s Algorithm from \citet{Dijkstra1959}.

Stochastic shortest path problems have been studied over the last 40 years, with early research as in \citet{BertsekasTsitsiklis91}. The ‘stochastic’ part of SSPs refers to an agent deciding on an action where the outcome of that action, including the cost incurred and the state it leads to, may depend on realisations random variables. Unlike SPs, SSPs require probability theory to study the properties of these random variables and develop algorithms for determining the optimal policy. 

SSPs have been studied in the \emph{known} and \emph{unknown} settings, where parameters of the MDP are known or unknown respectively. The known setting (also called \emph{planning}) was studied in \citet{BertsekasTsitsiklis91} and \citet{Kallenberg}, which considers using algorithms such as \emph{policy iteration} (PI) and \emph{value iteration} (VI) to solve. Both VI and PI involve iterating certain operators to find their fixed points, which give the optimal policy. These algorithms are dual, in the sense that they correspond to dual linear optimisation programs, as we show in \Cref{subsec:UtoLP}. 

The unknown SSP setting has only recently been studied, as in \citet{Tarbouriech2019}, \citet{rosenbergCohen} and \citet{Tarbouriech2021}. Here an agent faces a trade-off between \emph{exploring} (or \emph{learning/sensing}) the probability distributions versus \emph{exploiting} the current knowledge. The algorithms studied in \cite{Cohen2021,rosenbergCohen,Tarbouriech2019,Tarbouriech2021} are known as optimistic algorithms and are derived from upper confidence bound algorithms. These manage the choice of exploration versus exploitation by using concentration inequalities from probability theory that aim to quantify the trade-off. The approaches of \cite{Cohen2021,rosenbergCohen,Tarbouriech2019,Tarbouriech2021} attempt to estimate the probability distributions directly and based on these estimates, the agent then computes a policy using a version of value iteration. This involves iterating a modified operator to the usual VI. 
These algorithms are termed \emph{model-based}, as one or more parameters are estimated at each step.  In contrast, \emph{model-free} algorithms do not explicitly estimate the parameters and usually rely on estimating the value function such as in Q-learning algorithms, see for example \citet{YuBertsekas2013}.

Aside to this, \citet{neuPikeBurke} demonstrated how convex duality links optimistic model-based value iteration related algorithms (which they deem \emph{value-optimistic}) and model-based policy iteration related algorithms (which they deem \emph{model-optimistic}) in the finite horizon unknown MDP setting, and how this leads to new algorithms. We expected extending their analysis to the infinite horizon setting would lead to new algorithms for SSPs, as well as better understanding of existing algorithms. In particular, we expected it would lead to a deeper understanding of the operators associated with the extended value iteration approach of \citet{rosenbergCohen}. This paper is the result.

\subsection{Main results and layout of this paper} 
We start by recalling SSPs with known parameters as MDPs \Cref{sec:SSPsAsMDPsKnown}. We detail policy and value iteration, their corresponding operators, and recall results on the convergence of these algorithms from \citet{BertsekasTsitsiklis91}. We include several slightly different proofs of known results. In \Cref{subsec:UtoLP} we describe the relationship of policy and value iteration to linear optimisation programs.

\Cref{sec:unknownMDPs} describes SSPs when the transition functions are unknown and must be estimated. We do not consider the case where costs are unknown. General goals of SSP research with unknown parameters are discussed, and extended value iteration (EVI) is described in \Cref{subsubsec:EVI} as in \citet{rosenbergCohen} and \citet{jaksch10a}. The relationship between EVI and convex duality is explored in \Cref{subsec:UtoCP}, which is new research. We show that the EVI operators can be derived from a convex optimisation program, whose solutions is the fixed point of the operators. We show that the form of the dual program and prove there is no duality gap, i.e. strong duality holds.

The main contributions to new research are in \Cref{sec:bounds}. Here we study how using various norms and divergences in our convex programs can lead to approximations and new value iteration-inspired operators. We study the $\ell_1$ norm in depth in \Cref{subsec:l1norm}, and this leads to \Cref{conj:l1dagger} regarding the convergence of its corresponding operator. In \Cref{appn:exploringTheConjecture} we describe intuition for this conjecture and possible proof techniques, as well as giving empirical evidence that the conjecture holds.  

We consider the supremum norm in \Cref{subsec:supnorm}, the KL-divergence in \Cref{subsec:KLdiv}, the Reverse KL-divergence in \Cref{subsec:reverseKLdiv}, the $\chi^2$-divergence in \Cref{subsec:Chisquaredivbound}, and a variance-weighted supremum norm in \Cref{subsec:varweightLinfinitynorm}. In each case we detail one or more approximations and in \Cref{subsec:discussion} we discuss whether we expect these to lead to well-behaved operators. 

In the appendices we include supplementary material, including proofs of several mathematical identities in \Cref{appn: proofBoundCumulant,appn:minAlambdaBoneOnLambda,appn:minofLogfunction,appn:proofSpan,appn:weightedspan}, which are used to determine the bounds in \Cref{sec:bounds}. We also detail alternative approaches for the KL-divergence in \Cref{appn:alternativeKLdiv}, for the $\chi^2$-divergence in \Cref{appn:alternativeChisquared} and for the variance-weighted supremum norm in \Cref{appn:AlternativeVarWeightsupNorm}.

\section{Stochastic shortest paths as MDPs with known parameters}\label{sec:SSPsAsMDPsKnown}
In this section we give the definition of stochastic shortest path problems as Markov decision processes, set our notation and give the basic results from \citet{Kallenberg,BertsekasVol2} and \citet{BertsekasTsitsiklis91} as well as expanding on some of the ideas presented. Note that \citet{Kallenberg} tends to focus on rewards instead of costs. 

Recall that a Markov decision process (MDP) is a tuple $\Mc = (\Sc,\Ac,P,c)$ which governs a discrete stochastic process over some time horizon. Here $\Sc=\{1,2,\ldots, N\}$ is called the state space, $\Ac(s)$ is a set of actions for a state $s$, $P(s'|s,a)\in [0,1]$ is the probability that choosing action $a\in \Ac(s)$ at a state $s$ will lead to state $s'$ at the next time step, and $c(s,a)\in [0,1]$ is the cost of choosing action $a\in \Ac(s)$ at a state $s$. We will often abuse notation and write $a\in\Ac$ when $s$ is obvious. In general MDPs may have $P$ and $c$ depending on the time step, however we do not do this for stochastic shortest paths. 

Often MDPs are studied by generating a selection of actions called a policy $\pi:\Sc\times\mathcal{T}\to \Ac$ that is considered as a rule to pick a certain action $\pi(s_t)$ when the MDP is at state $s_t$ at time step $t\in \mathcal{T}={1,2,3\ldots}$. These policies generate sequences of states $s_1,s_2,\ldots,$ from the MDP with probabilities depending on $P$. Policies are then studied for their properties.

A stochastic shortest path problem (SSP) is an infinite horizon MDP with initial state $s_{init}\in \Sc$ and an additional \emph{goal} state $g\nin \Sc$. Here, the probabilities $P(\cdot|s,a)$ are substochastic so that 
\begin{equation} \label{eqn:substochastic}
\sum_{s'\in \Sc} P(s'|s,a)\le 1
\end{equation} and the remaining probability is considered the probability of reaching the goal state from that state-action pair
\[ P(g|s,a) = 1 - \sum_{s'\in \Sc} P(s'|s,a)\in [0,1].\] 

The aim for SSPs is to find an \emph{optimal} policy $\pi:\Sc\times\mathcal{T}\to \Ac$, so that starting from $s_1=s_{init}$ and applying this policy should reach the goal state with least cost. As the MPD is stochastic we consider the expected cost, called the \emph{cost-to-go} 
\[ J_{\pi}(s_{init}) = \lim_{T\to \infty}\Eb\left [\sum_{t=1}^T c(s_t,\pi(s_t))\right],\]
 as the quantity to minimise. The expectation is taken over the probability of the sequence  $s_1,\ldots, s_T$ given the MDP $\Mc$ and the policy $\pi$. For a stationary policy $\pi$ then writing $J_{\pi}(\cdot)$ as a vector over all possible initial states in $\Sc$ and expanding the previous equation we can write
 \[ J_{\pi} = \lim_{T\to \infty}\Eb\left [\sum_{t=1}^T c(\cdot,\pi(\cdot))\right] = \lim_{T\to \infty}\sum_{t=1} P_{\pi}^{t-1}c_{\pi}.\] Here we write $P_{\pi}$ to refer to the $N\times N$ matrix with $s$-th, $s'$-th entry $P(s'|s,\pi(s))$ and $c_{\pi}$ to refer to the $N\times 1$ column vector $c(\cdot,\pi(\cdot))$.
 
The MDP may be specified such that this expectation does not exist for a given policy, including for stationary policies. Importantly, only \emph{proper} policies are guaranteed to reach the goal state and have a finite cost-to-go.

\begin{defn} \label{defn:proper}
We define a \emph{proper} policy $\pi:\Sc\to\Ac$ to be a policy such that starting at any state $s_1$ and following this policy for at most $N$ stages, then there is a positive probability of reaching the goal state. That is, \[ \min_s\mathbb{P}(s_N=g|s_{init}=s,\pi) >0,\] which for a stationary policy gives \[ \min_{i}\max_{j}[P_{\pi}^N]_{ij} >0.\] Policies that are not proper are called \emph{improper}.
\end{defn}
It may not always be possible to find an optimal policy under the current $\Mc$. For example, such a policy requires at least one state to be able to transition to the goal state under at least one action with positive probability, as well as other conditions on the transition probabilities. We make the following assumption to assist with this. 

\begin{assu}\label{assu:proper}
Each SSP has at least one stationary proper policy, and every improper policy has at least one initial state such that the cost-to-go is infinite.
\end{assu}
Proper stationary policies must reach the goal state in a finite number of time steps, as shown in \citet[\S8.9]{Kallenberg}.

We usually assume that the costs are bounded below by some minimum $c_{\min}>0$ to prevent issues with \emph{zero cycles} occurring, this then implies improper policies have at least one initial state such that the cost-to-go is infinite without \Cref{assu:proper}.

Note that a zero-cycle (or better, a \emph{zero-region}) in an MDP is where a policy $\pi$, a state $s_1$ and an integer $n$ exist such that any instance of following the policy $\pi$ from state $s_1$ for $n$ steps returns to state $s_1$ at least once and with zero costs incurred. Such zero-cycles are problematic as once a MDP enters such a region following this policy it will not leave, and therefore will not reach the goal state, however it will incur no cost by remaining there. This may be unrealistic in practice, given actions usually cost small amounts to the user (e.g. time). In general, SSPs where zero costs are incurred may be transformed by adding a small amount $\epsilon$ to all zero costs and studied using our algorithms as $\epsilon \to 0$. See \citet[\S2]{rosenbergCohen} or \citet[\S5]{Tarbouriech2019} for further details.

We define the Bellman operator $L_\pi$ for a stationary policy $\pi$ and the optimal Bellman operator $U$ such that for $x\in\R^{N}$ we have
\[ L_{\pi}x = c_\pi + P_\pi x \quad \text{ so that } \quad (L_{\pi}x)_s = c(s,\pi(s)) + \sum_{s'\in S}P(s'|s,\pi(s))x_{s'}\] 
and 
\begin{equation}  (Ux)_s = \min_a \left\{ c(s,a) + \sum_{s'\in S} P(s'|s,a)x_{s'}\right\} = (\min_{\pi}L_{\pi}x)_s.\label{eqn:defnU} \end{equation}

We will use the convention that for real vectors $x,y$, then $y\ge x$ (or $y>x$) means that operator $\ge $ (or $>$) is applied elementwise.

Bertsekas and Tsitsiklis established the following, and we have adapted the proof to our own understanding.
\begin{lemma}[\citet{BertsekasTsitsiklis91}] \label{lem:PI}
Under \Cref{assu:proper} then for a stationary policy $\pi$ \begin{itemize}
    \item If there is a vector $x$ such that $L_{\pi}x \le x$ then $\pi$ is proper and $x \ge J_{\pi}$.
    \item For a proper stationary policy $\pi$ then $J_{\pi}$ is the unique fixed point of the equation $L_{\pi}x = x$ and $\lim_{k\to \infty}L_{\pi}^k =J_{\pi}$.
\end{itemize}
\end{lemma}
\begin{proof}
For (i), induction shows $x\ge L^{k}_{\pi}x$ for any $k$. Expanding we see that $L^{k}_{\pi}x = \sum_{i=0}^{k-1}P^i_\pi c_\pi + P^k_\pi x$. We note that $||P^k_\pi x||\leq ||P^k_\pi||||x||\leq ||x||$ by \Cref{eqn:substochastic}, so $P^k_\pi x$ is bounded below and above. We also have that $\sum_{i=0}^{k-1}P^i(\pi)c(\pi)$ is positive, so it is bounded below by $0$, and therefore it must be also bounded above. As $\sum_{i=0}^{k-1}P^i_\pi c_\pi$ is increasing and bounded above, it must converge, and its limit is $J_{\pi}$ the cost-to-go of the policy $\pi$. Hence $\pi$ is proper.

To show that $J_{\pi}\le x$, we need to discuss why $\lim_{k\to\infty}P^k_\pi=0$. As $P_\pi$ is a square, non-negative matrix, we can reorder the rows and corresponding columns of the matrix so that $P_\pi$ consists of  blocks of irreducible matrices along a diagonal corresponding to communicating classes. Here, all elements above the blocks will be zero, and there may be elements below the blocks that correspond to non-communicating elements. See \citet[pg.~15]{Seneta1981} for details. This decomposition does not change the eigenvalues of $P_\pi$, all of which have length less than or equal to one as $P_\pi$ is substochastic. 

Then $P_\pi^k$ will have the blocks to the power of $k$ along the diagonal. If any of these blocks have associated eigenvalues with length equal to one, these block contain cycles such that if the MPD starts in these states, it will never leave these states and not reach the goal state. This would mean $\pi$ is an improper policy. Otherwise the blocks will have eigenvalues with norm less than one, and the powers of these blocks and the matrix itself will eventually converge to the zero matrix. As we know $\pi$ is proper, then $D^k$ must converge to the zero matrix and  $\lim_{k\to\infty}P^k(\pi)=0$. This then implies $J_{\pi}\le x$, and also that $x\ge 0$. This also means that for a proper policy in SSP we must require every eigenvalue of $P^k(\pi)$ to be less than 1.

For (ii), if $\pi$ is proper then from the proof of (i) we have $L^{k}_{\pi}x = \sum_{i=0}^{k-1}P^i_\pi c_\pi + P^k_\pi x$ which converges as $k\to \infty$ to $J_{\pi}$ for all $x$. As $L^{k+1}_{\pi}x=c_\pi+P_\pi L^k_{\pi}x$, taking the limit as $k\to \infty$ gives $J_{\pi}=c_\pi+P_\pi J_{\pi}=L_{\pi}J_{\pi}$. Uniqueness follows by considering any solution $J$ to $L_{\pi}x=x$, and then induction shows $L_{\pi}^kx=x$, and taking the limit as $k\to \infty$ gives $J_{\pi}=x$.
\end{proof}
An alternative way to show that there is a unique fixed point is to show that $L_{\pi}$ is a contraction when $\pi$ is proper. 
\begin{lemma}\label{lem:LpiIsContraction}
When $\pi$ is stationary and proper then $L_{\pi}$ is a contraction under an induced vector norm.
\end{lemma}
\begin{proof}For any $x_1, x_2 \in \R^{N}$ then 
\begin{align*}
    \lVert L_{\pi}x_1 - L_{\pi}x_2 \rVert =  \lVert P_\pi(x_1 -x_2) \rVert \leq \lVert P_\pi\rVert \lVert x_1 -x_2 \rVert.
\end{align*}
Then we know from the proof of \Cref{lem:PI}(i) that the eigenvalue of $P_\pi$ with the largest norm has norm less than one. The norm of this eigenvalue is called the \emph{spectral radius of $P_\pi$}, often denoted $\rho(P_\pi)$. Then by \citet[Thm.~3,pg.~27]{Des2009}, for all $\epsilon >0$ there is a norm on $\R^{N}$ such that the induced matrix norm gives $\lVert P_\pi\rVert < \rho(P_\pi) + \epsilon$. Taking $\epsilon = \frac{1}{2}(1-\rho(P_\pi))$ implies that there exists a norm where $\lVert P_\pi\rVert<1$.

The Banach contraction mapping theorem from \citet{Banach1922} then states that $L_{\pi}$ has a unique fixed point and that the sequence $L_{\pi}^k(x_1)$ converges to this fixed point for any initial $x_1$.
\end{proof}
Here is a direct way we can extend this to the operator $U$, this theorem is known in the literature as referenced but the proof is adapted from \citet{Tseng1988}. See also \citet[pg.~102--103]{BertsekasVol2}.

\begin{lemma}[\citet{Veinott1969} and \citet{Tseng1988}]\label{lem:properContractionU} 
If all stationary policies are proper then $U$ is a contraction.
\end{lemma} 

\begin{proof} 
For any $x_1, x_2 \in \R^{N}$ and $s\in\Sc$ then
\begin{align*}
    \lvert(U x_1)_s - (U x_2)_s \rvert & = \lvert \min_a (c(s,a) + \sum_{s'}P(s'|s,a)x_1(s')) - \min_a (c(s,a) + \sum_{s'}P(s'|s,a)x_2(s'))\rvert\\
    & \leq\max_a  \lvert  (c(s,a) + \sum_{s'}P(s'|s,a)x_1(s')) - (c(s,a) + \sum_{s'}P(s'|s,a)x_2(s'))\rvert\\
    & = \max_a \lvert\sum_{s'}P(s'|s,a)(x_1(s')-x_2(s'))\rvert\\
    & = \max_{\pi}\lvert P(\cdot|s,\pi(s))(x_1-x_2)\rvert\\
    &= \lvert P(\cdot|s,\hat{\pi}(s))(x_1-x_2)\rvert,
\end{align*} where $\hat{\pi}$ is the policy which maximises $|P_{\hat{\pi}}(x_1-x_2)|$. Note that moving from line 1 to line 2 above uses \Cref{lem:maxarrange}.

We define $S_0=\{g\}$ and iteratively define $S_{q}$ for $q=1,\ldots, N$ such that
\[ S_{q} = \{ s\in \mathcal{S}\setminus S_{1}\cup S_{2}\cup\ldots \cup S_{q-1} \mid \forall a\in \mathcal{A},\; \exists s'\in S_{q'} \text{ such that } 0\le q'< q \text{ and } P(s'|s,a) >0\}.\] These sets are clearly disjoint. Using that all policies are proper, then $S_1$ is non-empty and while $S_{1}\cup S_{2}\cup\ldots \cup S_{q-1}\ne \mathcal{S}$ then $S_{q}$ is non-empty. Ignoring all empty $S_{q}$ we see that we can partition $\mathcal{S}$ so that $\mathcal{S}=S_1\cup S_2\cup \ldots \cup S_r$, $r\le N$. We set \[\eta = \min_{s'\in \Sc\cup\{g\},s\in\Sc,a\in\Ac}\{ P(s'|s,a) >0\}\in (0,1)\] and $\gamma = \frac{1-\eta^{2r-1}}{1-\eta^{2r}}\in (0,1)$. For each $s\in S_{q}$ we define $\omega_s = 1-\eta^{2q}\in (0,1)$.

Note that if $s\in\mathcal{S}_q, s'\in\mathcal{S}_{q'}$ with $q'< q$ then $\omega_{s'}< \omega_s $ and that $\frac{1-\eta^{2s'-1}}{1-\eta^{2s'}} \le \frac{1-\eta^{2s-1}}{1-\eta^{2s}}$.

For a given $s\in\Sc_{q}$ and $a\in \mathcal{A}(s)$, take $s'\in S_{0}\cup S_{1}\cup \ldots \cup S_{q-1}$ such that $P(s'|s,a) >0$. Then we have 
\begin{align*}
    \sum_{s''}P(s''|s,a)\frac{\omega_{s''}}{\omega_s} =&\sum_{s''\ne s'}P(s''|s,a)\frac{\omega_{s''}}{\omega_s} + P(s'|s,a)\frac{\omega_{s'}}{\omega_s}\\
    &\le \sum_{s''\ne s'}P(s''|s,a)\frac{1}{\omega_s} + P(s'|s,a)\frac{\omega_{s'}}{\omega_s}\\
    & = (1-P(s'|s,a)-P(g|s,a)) \frac{1}{\omega_s} + P(s'|s,a)\frac{\omega_{s'}}{\omega_s}\\
    & = (1-P(s'|s,a)(1-\omega_{s'}) -P(g|s,a))\frac{1}{\omega_s}\\
    &\le (1-\eta(1-\omega_{s'}))\frac{1}{\omega_s}\\
    &\le \frac{1-\eta^{2q-1}}{1-\eta^{2q}}\le \gamma.
\end{align*}
Then consider that 
\begin{align*} (Ux_1-Ux_2)_s&\le  |P(\cdot|s,\hat{\pi}(s))(x_1-x_2)|\\
&= |\sum_{s'}P(s'|s,\hat{\pi}(s))(x_{1,s'}-x_{2,s'})|\\
&=|\omega_s\sum_{s'}P(s'|s,\hat{\pi})(s))\frac{\omega_{s'}}{\omega_s}(x_{1,s'}-x_{2,s'})\frac{1}{\omega_{s'}}|\\
 &\le \omega_s\sum_{s'}P(s'|s,\hat{\pi})(s))\frac{\omega_{s'}}{\omega_s} \max_{s'}\left[|x_{1,s'}-x_{2,s'}|\frac{1}{\omega_{s'}}\right]\\
 &\le \omega_s\gamma \max_{s'}\left[|x_{1,s'}-x_{2,s'}|\frac{1}{\omega_{s'}}\right].
\end{align*} Note that exchanging $x_1$ and $x_2$ results in the same inequality, so that we have 
\[ \left|(Ux_1-Ux_2)_s\frac{1}{\omega_s}\right| \le \gamma \max_{s'}\left[\left|x_{1,s'}-x_{2,s'}\right|\frac{1}{\omega_{s'}}\right]\] and so 
\[ \lVert Ux_1-Ux_2\rVert_{\mathbf{\omega}} \le \gamma \lVert x_1-x_2\rVert_{\mathbf{\omega}},\] where the norm is the weighted supremum norm with weights $\mathbf{\omega} = (\omega_1,\ldots, \omega_N)$. As $\gamma$ does not depend on $x_1$ nor $x_2$ then $U$ is a contraction.
\end{proof}
Bertsekas and Tsitsiklis use this to prove the following theorem, although \citet[Thm.~8.38]{Kallenberg} has a slightly different proof. Both use the monotonicity of $L_{\pi}$, that is if $x\le y$ then $L_{\pi}x\le L_{\pi}y$.

\begin{lemma}[\citet{BertsekasTsitsiklis91}] \label{lem:VIconverges}
For a SSP under \Cref{assu:proper} then the optimal policy $\pi^*$ exists and is stationary and proper. The optimal cost-to-go $J_{\pi^*}$ is the unique fixed point of the operator $U$, and is the limit of $U^k(x)$ for any $x\in \R^N$. Finally, a stationary policy $\pi$ is optimal if and only if $L_{\pi}J_{\pi} = UJ_{\pi}$. 
\end{lemma}
Note that if the cost-to-go is known for a policy $\pi$ then one can find a corresponding policy with that cost-to-go by acting \emph{greedily}. That is we take \[\pi(s) \in \argmin_{a\in\Ac} c(s,a) + \sum_{s'\in S} P(s'\mid s,a) J_{\pi}.\]

As the optimal policy of an SSP is stationary, we restrict our attention to stationary policies only in the sequel.

We note that solving for the fixed points of $L_{\pi}$ can proceed as follows: \[L_{\pi_0}x=x = c_{\pi_0}+P_{\pi_0} x \quad \Leftrightarrow \quad (I-P_{\pi_0})x = c_{\pi_0}.\] As $P_{\pi_0}$ has spectral radius less than 1 then $(I-P_{\pi_0})$ is invertible and $x = (I-P_{\pi_0})^{-1}c_{\pi_0}$ is the fixed point.

This result and the above lemmas allow for algorithms to solve for the optimal policy $\pi$ and corresponding value vector $v$ in a SSP with known transition probabilities. We summarise the main algorithms here, adapted from \citet{Kallenberg}.

\subsubsection*{Policy Iteration for known SSP}
\begin{algorithmic}
\Require State space \(S\), action space \(A\), proper policy $\pi_0$, instance of a known SSP.

\State \!\!\!\!\!\!{\bf output:} Optimal policy $\pi$, with value vector $v$.
\Ensure Set $x\leftarrow J_{\pi_0}$, which is the unique solution to $L_{\pi_0}x=x$.
\State {\bf Step 1}

Set $y_s \leftarrow \min_{a}\{ c(s,a) +\sum_{s'\in\Sc}P(s'|s,a)x_{s'}\}$ for $s=1,\ldots, N$.

\State {\bf Step 2}
\If{ $||y-x||=0$}
\State Set $v\leftarrow y$
\State Set $\pi(s)\leftarrow \argmin_a\{c(s,a) +\sum_{s'\in\Sc}P(s'|s,a)x_{s'}\}$ for $s=1,\ldots, N$.
\State {\bf STOP}
\Else{}
\State Set $\pi_0(s)\leftarrow \argmin_a\{c(s,a) +\sum_{s'\in\Sc}P(s'|s,a)x_{s'}\}$ for $s=1,\ldots, N$.
\State Set $x\leftarrow J_{\pi_0}$, the unique solution to $L_{\pi_0}x=x$. 
\State {\bf RETURN to Step 1}
\EndIf
\end{algorithmic}

Note that policy iteration is equivalent to the simplex algorithm applied to the linear program in \Cref{LP dualKnown} as in \citet[pg.~621]{Kallenberg}. Notably, variants of policy iteration give variants of solution methods for this linear program.

\subsubsection*{Value Iteration for known SSP}
\begin{algorithmic}
\Require State space \(S\), action space \(A\), proper policy $\pi_0$, instance of a known SSP, scalar $\epsilon>0$.
\State \!\!\!\!\!\!{\bf output:} Near optimal policy $\pi$, with approximate value vector $v$.
\Ensure Set $x\leftarrow J_{\pi_0}$, which is the unique solution to $L_{\pi_0}x=x$.
\State {\bf Step 1}

Set $y_s \leftarrow \min_{a}\{ c(s,a) +\sum_{s'\in\Sc}P(s'|s,a)x_{s'}\}$ for $s=1,\ldots, N$.

\State {\bf Step 2}
\If{ $||y-x||\le \epsilon$}
\State Set $v\leftarrow y$
\State Set $\pi(s)\leftarrow \argmin_a\{c(s,a) +\sum_{s'\in\Sc}P(s'|s,a)x_{s'}\}$ for $s=1,\ldots, N$.
\State {\bf STOP}
\Else{ }
\State Set $y\leftarrow x$
\State {\bf RETURN to Step 1}
\EndIf
\end{algorithmic}

The previous lemmas ensure that both of these algorithms converge. Value iteration is essentially just iterating the operator $U$ to find its fixed point to within a certain tolerance, while policy iteration incrementally finds a series of policies using the operator $L_{\pi}$. There are other, more efficient versions of value iteration such as the Gauss-Seidel variant, see \citet[\S3.6,\S8.9]{Kallenberg}, \citet[\S6.3]{Puterman} and \citet[\S1.3.2]{BertsekasVol2} for further details.

\subsection{Relationship to linear programs}\label{subsec:UtoLP}
In this section we show how considering fixed points of $U$ leads us to a set of primal/dual linear programs and how to interpret them. We are guided by \cite[\S3.5]{Kallenberg} who considers the discounted horizon case, although we deviate with adaptions to costs (rather than rewards) and the SSP setting. See also \citet[\S6.9]{Puterman}. We extend on this in \Cref{subsec:UtoCP}.

Seeking a fixed point of $U$, we see that  
\[(Ux)_s =x_s = \min_a \left\{ c(s,a) + \sum_{s'\in S} P(s'|s,a)x_{s'}\right\} \le  c(s,a) + \sum_{s'\in S} P(s'|s,a)x_{s'} \quad \forall s\in \Sc, a\in \Ac.\]

\begin{defn}
The set of vectors $x\in\R^N$ such that \[x_s \le  c(s,a) + \sum_{s'\in S} P(s'|s,a)x_{s'} \quad \forall s\in \Sc, a\in \Ac\] are called \emph{superharmonic} (with respect to the operator $U$) as in \cite[\S3.5]{Kallenberg}. 
\end{defn}

\begin{lemma}
The fixed point of $U$ can be found by solving the following linear program
\begin{align}
    \max &\sum_{s=1}^N x_s \label[Program]{LP primalKnown}\\\nonumber
    & \text{such that} \\\nonumber
     x_s \le  c(s,a) &+ \sum_{s'\in S} P(s'|s,a)x_{s'} \quad \forall s\in \Sc, a\in \Ac.
\end{align}
\end{lemma}
\begin{proof}
We first show that $x\le J_{\pi^*}$ for all superharmonic $x$.

As there is a norm such that $\rVert P_\pi\lVert <1$ then we see that 
\[\left\lVert \sum_{t=0}^{\infty}P_{\pi}^t\right\rVert \le  \sum_{t=0}^{\infty}\lVert P_{\pi}\rVert = \frac{1}{1-\lVert P_{\pi}\rVert}\] so $\sum_{t=0}^{\infty}P_{\pi}^t$ converges and then $(I-P_\pi)^{-1} = \sum_{t=0}^{\infty}P_{\pi}^t\ge 0$. 
Then for a super harmonic $x$ we must have 
\begin{align*}
   x &\le  c_\pi+  P_\pi x \quad \forall \pi\\
  \Leftrightarrow \quad  c_{\pi}&\ge (I- P_\pi )x\\
  \Leftrightarrow \quad  J_{\pi}= (I-P_\pi)^{-1} c_{\pi}&\ge (I-P_\pi)^{-1} (I- P_\pi )x = x \quad \forall \pi
\end{align*}
Minimising over $\pi$ we see that $x\le J_{\pi^*}$ for all superharmonic $x$.

We know that taking the maximum elementwise over the possible superharmonic $x$ is attainable by the optimal policy cost-to-go $J_{\pi^*}$, and it is clear that any superharmonic $x$ that attains this maximum is the fixed point of $U$, so it must be $J_{\pi^*}$ by uniqueness. 

So finding a fixed point of $U$ is equivalent to finding an superharmonic $x\in \R^N$ that is maximum in every element, and is therefore equivalent to maximising the sum of its elements. This means the fixed point of $U$ is the unique solution to \Cref{LP primalKnown}.
\end{proof}

Applying Lagrange multipliers $q(s,a)$ as in \citet[\S5.1.2]{Boyd} we can determine the Lagrange dual function of \Cref{LP primalKnown} as follows
\begin{align*}
    L(q) &=\max_x\sum_{s=1}^N x_s-\sum_{s,a}q(s,a)(x_s -  c(s,a) - \sum_{s'\in S} P(s'|s,a)x_{s'})\\
    & = \max_x\sum_{s=1}^N x_s(1-\sum_a q(s,a)+\sum_{s',a}q(s',a)P(s|s',a)) +\sum_{s,a}q(s,a)c(s,a)\\
\end{align*} 
Taking the minimum over $q(s,a)\ge 0$ then we have the dual linear program of \Cref{LP primalKnown}
\begin{align}
    \min_{q(s,a)\ge 0} &\sum_{s,a}q(s,a)c(s,a) \label[Program]{LP dualKnown}\\\nonumber
    & \text{such that} \\\nonumber
     \sum_{a}q(s,a)&= 1+\sum_{s',a}q(s',a)P(s|s',a) \quad \forall s\in \Sc.
\end{align}
As \Cref{LP primalKnown} always has an optimal solution and linear programs exhibit strong duality \citet[\S5.2.4]{Boyd} then \Cref{LP dualKnown} also always has an optimal solution. \citet[pg.~621]{Kallenberg} shows that policy iteration is equivalent to the simplex algorithm applied to \Cref{LP dualKnown}.

\subsubsection{Interpretation of the dual program}\label{subsubsec:occupancy}
The dual variables $q(s,a)$ of \Cref{LP dualKnown} can be interpreted as a form of \emph{occupancy measure}, where
\[ q(s,a) = N\sum_{t=0}^\infty\mathbb{P}(\text{being in state $s$ and performing action $a$ at time $t$ under policy $\pi$}).\] Here for any $s\in \Sc$ we consider the probability of $s=s_{init}$ to be equal to $1/N$. Then minimising $\sum_{s,a} c(s,a)q(s,a)$ is equivalent to minimising the expected cost of the sequences of states and actions under policy $\pi$. In addition, the constraint  \[\sum_{a}  q(s,a) = 1+ \sum_{a} \sum_{s'} P(s\mid s',a) q(s',a)\]
can be derived from considering under policy $\pi$ that
\begin{align*}
\frac{1}{N}&\sum_aq(s,a) \\
&=\sum_{t=0}^\infty\mathbb{P}(\text{achieving state $s$ at time $t$})\\
 & =\mathbb{P}(s_0=s)+\\
 &\sum_{s',a'}\sum_{t=1}^\infty\mathbb{P}(\text{achieving state $s$ at time $t$, and achieving state $s'$ and taking action $a'$ at time $t-1$})\\
 & =\mathbb{P}(s_0=s)+\\
 &\sum_{s',a'}\sum_{t=0}^\infty\mathbb{P}(\text{being in state $s$ at time $t$, and being in state $s'$ and taking action $a'$ at time $t$})\\
 &= \frac{1}{N}+\sum_{s',a'}\sum_{t=0}^\infty P(s\mid s',a') \mathbb{P}(\text{achieving state $s'$ and taking action $a'$ at time $t$})\\
  &= \frac{1}{N}\left(1+\sum_{s',a'}P(s\mid s',a')q(s',a')\right)
\end{align*}
 for all $s\in \mathcal{S}, a\in \mathcal{A}$.
 
Note that for any feasible solution $q(s,a)$ to \Cref{LP dualKnown} then we have \[\sum_a q(s,a)= 1+\sum_{s',a}q(s',a)P(s|s',a)\ge 1.\] Then solving \Cref{LP dualKnown} for the optimal $q(s,a)$, the corresponding optimal policy $\pi$ is given by \begin{equation}\pi(s)=\frac{q(s,a)}{\sum_a q(s,a)}.\label{eqn:qtopi}\end{equation} \citet[\S 3.5]{Kallenberg} shows precisely how this can be derived using complementary slackness.

\section{Stochastic shortest paths with unknown transitions}\label{sec:unknownMDPs}

In this section, we focus on finding optimal policies for SSPs when the transitions $P$ are unknown. We assume that an agent may interact with the SSP and record data from the visits to different states, that is the algorithm is performed \emph{online}. In this setting, algorithms are studied that seek to \emph{explore} the SSP to \emph{learn} the model parameters and also \emph{exploit} their findings by applying policies that are optimal with respect to their new found knowledge. 

The algorithms vary by how they trade off between further exploration versus current exploitation. The differences between the algorithms can be subsequently measured by understanding the properties of the \emph{regret}. This measures the distance between expected costs following the true optimal policy over multiple episodes versus the expected costs following the policies implemented by the algorithm over the same number of episodes. By episodes, we mean an agent engaging with the SSP by starting at an initial state $s_{init}$ and continuing until the goal state is found.  

The regret over $K$ episodes is defined as
\[R_{K} = \sum_{k=1}^K \sum_{t=1}^{T^k}c(s_t^k,a_t^k)-K\min_{\pi}J_{\pi}(s_{init}). \] Here $T^k$ is the time taken to complete episode $k$, and $(s_t^k,a_t^k)$ is the state and action taken at time $t$ in episode $k$ under the algorithm studied. Intuitively, algorithms that do not achieve the right balance between exploring and exploiting will incur large regret, and the regret will not reduce efficiently as the number of episodes increases. Algorithms that are improving as they iterate should incur sub-linear regret, showing that the regret of each episode is decreasing on average. This behaviour may not be observed initially as the algorithm may seek to explore without trying to minimise regret before exploiting the knowledge --- however, after this exploration period, the regret becomes sub-linear. 

A particular goal of an SSP algorithm with unknown parameters is to minimise the regret of an algorithm. The regret of any algorithm is not expected to be known explicitly, but usually is able to be bounded with some high probability, with the bound depending on various parameters such as the number of episodes and often parameters dependent on the MDP.  Minimising the regret of an algorithm is often referred to as achieving the minimax bound as in \citet{Cohen2021}, i.e. minimising the maximum/worst-case regret. \citet{Cohen2021} show this bound is $\tilde{O}(\sqrt{(B_*^2+B_*)|\mathcal{S}||\mathcal{A}|K})$ for $B_*>1$ improving on \citet{rosenbergCohen}. 

Other SSP research goals include reducing the dependence of the bound and the algorithm on parameters. These parameters include $c_{\min}>0$, a lower bound on the costs; $B_*$, an upper bound on the expected cost of the optimal policy; the diameter of the policy; the size of the action and state spaces; $T_*$ an upper bound on the time to reach the goal state until the optimal policy and $K$ the number of episodes (iterations) of the algorithm. See \citet[\S1]{Tarbouriech2021} for further discussion on recent SSP research goals.

\subsubsection*{Structure of general SSP algorithm with unknown $P$ but known $c$}
\begin{algorithmic}
\Require State space \(S\), action space \(A\), instance of a known SSP, initial state $s_{init}$ and goals state $g$, maximum episode number $K$, trigger point parameters.
\Ensure Set $s\leftarrow s_{init}$, regret $R_i = 0$ for $i=1,\ldots, K$, $R_0=0$
\For {i = 1,2, \ldots, K}
 Set $s\leftarrow s_{init}$, $R_{k} \leftarrow R_{k-1}$\\
{\bf while} $s\ne g$
\State {\bf Step 1} Determine new position $s'$ by following some fixed policy
\State {\bf Step 2} Record information gathered from moving to this new state (LEARN)
\State {\bf Step} Determine if trigger point is activated
\If{trigger point is activated}
Use all information gathered so far to update the policy used at Step 1. (EXPLOIT)
\EndIf

Update $R_{k} \leftarrow R_{k}+c(s',a)$.

Update $s\leftarrow s'$.\\
{\bf end while}
\EndFor
\end{algorithmic}

A very simple algorithm for SSPs is an ultimate \emph{greedy} algorithm that proceeds as follows:

\subsubsection*{Ultimate greedy algorithm for SSP with unknown $P$ but known $c$}
\begin{algorithmic}
\Require State space \(S\), action space \(A\), instance of a known SSP, initial state $s_{init}$ and goals state $g$, probability $\epsilon \in (0,1)$, maximum episode number $K$.
\Ensure Set $s=s_{init}$, regret $R_i = 0$ for $i=1,\ldots, K$, $R_0=0$
\For {i = 1,2, \ldots, K}
 Set $s=s_{init}$, $R_{k} \leftarrow R_{k-1}$\\
{\bf while} $s\ne g$
\State {\bf Step 1} Set $a_{\min} = \argmin_{a\in \Ac(s)}c(s,a)$
\State {\bf Step 2} Choose action $a$ by selecting $a_{\min}$ with probability $(1-\epsilon)$, and otherwise uniformly select $a$ as one of the other actions to take. Identify the next state $s'$.

Update $R_{k} \leftarrow R_{k}+c(s',a)$.

Update $s\leftarrow s'$.\\
{\bf end while}
\EndFor
\end{algorithmic}
Here, no learning is undertaken at any stage, it is purely exploitative on the current known costs, and this algorithm is not even guaranteed to terminate unless all policies are proper. The regret is linear. No estimates of any model parameters are generated.

In general, algorithms are divided into two categories, model-free and model-based. Model-free algorithms for SSP such as \emph{Q-learning} have been discussed in \citet{YuBertsekas2013} and more recently in \citet{Chen2021} and we do not cover further here. See also \citet[\S6.5]{BartoSutton}, \citet{Watkins1989} and \citet{Watkins1992} for further details.

Since model-based algorithms rely on estimation, there are both Bayesian and frequentists approaches and Bayesian SSP algorithms have been recently studied in \citet{JJahromi2021}. We instead focus on frequentist approaches with known costs, such as in \citet{rosenbergCohen}. One aspect of this research is adapting the Bellman operators to \emph{optimistic} versions that are used to solve for subsequent actions --- these operators allow for some degree of tolerance around the estimated parameters and seek to be optimal within this tolerance. Another aspect involves adapting the algorithm parameters to ensure appropriate amounts of exploration versus exploitation in order to minimise the regret. 

We note that most of the SSP frequentist approaches tend to adapt value iteration approaches with modified operators $U$, along the same lines as the extended value iteration of \citet{jaksch10a}. We have not so far encountered policy iterations methods applied to unknown SSPs, such as those studied for discounted horizon methods in \citet{Kaufman2011} or finite horizons in \citet{Auer2006}.

In general, the costs may be chosen to be known (as in our case and \citet{rosenbergCohen}), unknown and requiring estimation as in \citet{Chen2021}, or adversarial as in \citet{Neu2012,Rosenberg2020, ChenLuo2021}. The adversarial setting is quite distinct, while the known and unknown costs cases have similar approaches.

Motivations for algorithms for unknown SSPs include the upper confidence bounds of bandit algorithms and their extensions, as well as other algorithms devised for MDPs with different horizon settings, including discounted, finite or average horizons. \citet{Cohen2021} and \citet{Chen2021} explicitly show that finite horizon methods satisfying certain properties can be used to create algorithms for SSPs by reducing SSPs to finite-horizon approximations.  

We are also motivated by finite-horizon perspectives, particularly those of \citet{neuPikeBurke}. They consider finite-horizon MDP algorithms with unknown transition functions from the perspective convex-duality. From their viewpoint, finite horizon algorithms' operators differ in the terms of a metric or divergence used to determine the tolerance from the estimated parameters. Using convex optimisation, they show that optimistic model-based value iteration related algorithms (which they deem \emph{value-optimistic}) are dual to model-based policy iteration related algorithms (which they deem \emph{model-optimistic}). They also consider how bounds on these operators can lead to more efficient algorithms with straight-forward regret bounds. 

We now recall extended value iteration from \citet{jaksch10a} and its application to unknown SSPs as in \citet{rosenbergCohen}, before we consider how this relates to convex optimisation in a similar way to \citet{neuPikeBurke}.

\subsection{Extended value iteration} \label{subsubsec:EVI}
The particular algorithms that we are interested are versions of extended value iteration (EVI) algorithms. Extended value iteration was first described in \citet{jaksch10a} for undiscounted infinite horizon problems. Here the operator $U$ is adapted to $\hat{U}$ where a minimum is taken over transition functions in some feasible set. This gives 
\[ (\hat{U}x)_s = \min_{a\in \mathcal{A}} \left\{ c(s,a) + \min_{\tilde{P}\in \mathcal{P}} \left\{ \sum_{s'\in \mathcal{S}} \tilde{P}(s'|s,a) x_{s'}\right\} \right\}.\]

Iterating the operator $\hat{U}$ converges provided the set $\mathcal{P}$ is compact, as this is equivalent to value-iteration for an \emph{extended} SSP MDP (the original MDP with an extended compact action space) as described  \citet[\S3.1]{jaksch10a}. This converges to its fixed point $\hat{J}^*$ by the version of \Cref{lem:VIconverges} for compact actions sets as in \citet{BertsekasTsitsiklis91}. This $\hat{J}^*$ can be considered an \emph{optimistic} cost-to-go, as we discuss below. First we give an outline of the EVI algorithm adapted from \citet{rosenbergCohen}.

\subsubsection*{Extended value iteration for SSPs with unknown $P$ but known $c$}
This algorithm is adapted from \citet[Alg.~1,Alg.~2]{rosenbergCohen}. It is motivated by the UCRL2 algorithm in \citet[Fig.~1]{jaksch10a}. See also \citet[Alg.~1]{Tarbouriech2019}.
\begin{algorithmic}
\Require State space \(S\), action space \(A\), instance of a known SSP, initial state $s_{init}$ and goals state $g$, bound of cost-to-go of optimal policy $B_*$, confidence parameter $\delta$.
\Ensure Set $s\leftarrow s_{init}$, regret $R_i = 0$ for $i=1,\ldots, K$, $R_0=0$. $\forall (s,a,s
')\in S\times A \times S, N(s,a,s')\leftarrow 0, N(s,a) \leftarrow 0$, arbitrary policy $\tilde{\pi}$, $t\leftarrow1$
\For{k = 1,2, \ldots}
 Set $s\leftarrow s_{init}$, $R_{k} \leftarrow R_{k-1}$\\
{\bf while} $s\ne g$
\State {\bf Step 1} Determine action $a$ and new position $s'$ by following $\tilde{\pi}$. 
\State {\bf Step 2} Update $N(s,a,s')\leftarrow 1+N(s,a,s'), N(s,a) \leftarrow 1+ N(s,a)$ (LEARN)
\State {\bf Step} Determine if trigger point is activated, either if $s'=g$ or if $N(s,a)$ too small (using $\delta$ and $B_*$).
\If{trigger point is activated}
Calculate empirical transition functions $\hat{P}$ where $\hat{P}(s,a,s') = N(s,a,s')/\max\{N(s,a),1\}$. Update $\tilde{\pi}$ to the optimistic policy found by iterating $\hat{U}$ over $\tilde{P}\in\mathcal{P}$. (EXPLOIT)
\EndIf

Update $t \leftarrow t+1$.

Update $R_{k} \leftarrow R_{k}+c(s',a)$.

Update $s\leftarrow s'$.\\
{\bf end while}
\EndFor
\end{algorithmic}

The set $\mathcal{P}$ is usually defined as a convex set of the form 
\[\mathcal{P}_{\epsilon} = \{ \tilde{P} \in \Delta \mid D(\tilde{P}(\cdot|s,a),\hat{P}(\cdot|s,a) \le \epsilon(s,a)\}\] where $\Delta$ is the set of all feasible transition probabilities $P$ for the MDP and $\hat{P}$ is some transition probability that has been estimated from available data. 
The function $D:\Delta_{s,a}\times\Delta_{s,a}\to [0,\infty)$ is some appropriate divergence or metric on the space of feasible transition probabilities from the state/action pair $(s,a)$. 

Usually the set $\mathcal{P}_{\epsilon}$ can be considered a confidence region (rather than an interval) of where the true transition probability $P$ is likely to lie, and taking the $\tilde{P}$ that minimises $\sum_{s'\in \mathcal{S}} \tilde{P}(s'|s,a) x_{s'}$ in $\hat{U}x$ is similar to taking an lower bound on the expected cost-to-go. Note that if we formulated our MDP using rewards instead of costs, this would be an upper bound, as is familiar from algorithms such as UCRL2 \cite{jaksch10a}. 

At each episode of the algorithm, an $\epsilon$ is decided upon and a given $\hat{P}$ is determined, then $\hat{U}x$ it iterated until it converges to $\hat{J}^*$ (within some tolerance). Then the ``optimal'' policy is chosen by 
\[\pi^*(s) = \argmin_{a\in \mathcal{A}} \left\{ c(s,a) + \min_{\tilde{P}\in \mathcal{P}} \left\{ \sum_{s'\in \mathcal{S}} \tilde{P}(s'|s,a) \hat{J}^*_{s'}\right\}\right\},\] which determines how next to interact with the MDP. This policy can be considered as choosing the actions that are (almost) optimal with respect to an \emph{optimistic} model of the MDP, which is known as the paradigm of \emph{optimisim in the face of uncertainty}.

This leads to obvious questions around the choice of $D$, how to estimate $\hat{P}$ and how to determine $\epsilon$. The majority of the theory so far tends to estimate $\hat{P}$ as the empirical transition function from prior interaction with the MDP, although may tilt this to \emph{optimistically} favour the goal state, say to ensure $\hat{P}$ has all policies proper, as we discuss in \Cref{rem:modifiedP}. See also \citet[eq.~5]{Tarbouriech2019} and \citet[\S A.5]{neuPikeBurke} for examples. 

Different choices of $D$ tend to change the regret of the algorithm, and have included the $\ell_1$-norm as in  \citet{jaksch10a} and \citet[Alg.~1,3]{rosenbergCohen}, and the $\ell_{\infty}$-norm in \citet[Alg.~2]{rosenbergCohen}. An upper bound on the regret can usually be found with some probability that is dependent on the likelihood that the true transition function $P$ is within the set $\mathcal{P}_{\epsilon}$. This probability is tied to the definition of $\epsilon$. The choice of $D$ also determines which concentration inequalities are appropriate to be used in the proof of the regret, such as the Hoeffding bounds for the $\ell_1$ norm and the Bernstein bounds for the $\ell_{\infty}$ norm, see \citet{Boucheron2012} for further details on concentration inequalities.

At each iteration of $\hat{U}$, minimising over $\mathcal{P}$ is equivalent to solving a linear program. If $\mathcal{P}$ has a piecewise linear boundary then $\mathcal{P}$ is a convex polytope and one can find the optimal solution by finding the minimum of the vertices of the region. As the state space and action space size of the MDP increases, this involves increased computational power and may be a deterrent from using this method.

\subsection{Relationship to convex optimisation}\label{subsec:UtoCP}
Our research involves the relationship of the EVI operator $\hat{U}$ to convex optimisation programs, similar to the relationship of $U$ to linear programs detailed in \Cref{subsec:UtoLP}. While we have been guided by the work of \citet{neuPikeBurke} and the known SSP case as in \Cref{subsec:UtoLP}, we have not seen this presentation before for SSPs.

Here we take \[\mathcal{P}_{\epsilon} = \{ \tilde{P} \in \Delta \mid D(\tilde{P}(\cdot|s,a),\hat{P}(\cdot|s,a)) \le \epsilon(s,a)\}.\] We assume $D$ to be a suitable metric or divergence such that positive homogenenity holds, i.e. $D(aP,aP')=aD(P,P')$ for any $a\ge 0$ and $P,P'$ valid probability transition functions. We note that the EVI operators corresponding the $\ell_1$-norm and the supremum norm have been studied in \citet{rosenbergCohen} and \citet{Tarbouriech2019}, however we have not seen any other EVI operators studied in the SSP literature. 

Then we can write the operator $\hat{U}$ as
\begin{align*} 
(\hat{U}x)_s &= \min_{a\in \mathcal{A}} \left\{ c(s,a) + \min_{\tilde{P}\in \mathcal{P}_{\epsilon}} \left\{ \sum_{s'\in \mathcal{S}} \tilde{P}(s'|s,a) x_{s'}\right\} \right\}\\
& = \min_{a\in \mathcal{A}} \left\{ c(s,a) +  \sum_{s'\in \mathcal{S}} \hat{P}(s'|s,a) x_{s'} + \min_{\tilde{P}\in \mathcal{P}_{\epsilon}} \left\{ \sum_{s'\in \mathcal{S}}(\tilde{P}(s'|s,a) - \hat{P}(s'|s,a)) x_{s'}\right\}\right\}\\
& = \min_{a\in \mathcal{A}} \left\{ c(s,a) +  \sum_{s'\in \mathcal{S}} \hat{P}(s'|s,a) x_{s'} + \CB_{\min}(s,a)(x)\right\}
\end{align*}
where 
\[\CB_{\min}(s,a)(x) = \min_{\tilde{P}\in \mathcal{P}_{\epsilon}} \left\{ \sum_{s'\in \mathcal{S}}(\tilde{P}(s'|s,a) - \hat{P}(s'|s,a)) x_{s'}\right\}.\] Note that as $\epsilon$ tends to $0$ then $\CB_{\min}$ also tends to zero.

Then following \Cref{subsec:UtoLP}, for a fixed point $x$ of $\hat{U}$ we have
\begin{align*}
    (\hat{U}x)_s &=x_s = \min_{a\in \mathcal{A}} \left\{ c(s,a) +  \sum_{s'\in \mathcal{S}} \hat{P}(s'|s,a) x_{s'} + \CB_{\min}(s,a)(x)\right\}\\
    &\le  c(s,a) +  \sum_{s'\in \mathcal{S}} \hat{P}(s'|s,a) x_{s'} + \CB_{\min}(s,a)(x) \quad \forall s\in \Sc, a\in \Ac.
\end{align*}
This again leads to the definition of superharmonic with respect to $\hat{U}$.
\begin{defn}
The set of vectors $x\in\R^N$ such that 
\[x_s \le  c(s,a) + \sum_{s'\in \mathcal{S}} \hat{P}(s'|s,a) x_{s'} + \CB_{\min}(s,a)(x) \quad \forall s\in \Sc, a\in \Ac\] 
are called \emph{superharmonic} with respect to $\hat{U}$. The set of all such superharmonic vectors is convex, as for any two superharmonic $x_1$, $x_2$ and $\alpha \in [0,1]$ then $\alpha x_1+(1-\alpha)x_2$ is also superharmonic (this uses that the minimum in $\CB_{\min}$ obeys the triangle inequality, e.g. $\min_i(y_i)+\min_i(z_i)\le \min_i(y_i+z_i)$ for any vectors $y,z$.)
\end{defn}
As $\mathcal{P}_{\epsilon}\subset\Delta$ is convex, then as in \Cref{subsec:UtoLP} we can show that $x\le \hat{J}^*$ for all superharmonic $x$, and that the unique fixed point $x$ of $\hat{U}$ is unique superharmonic vector such that  $x = \hat{J}^*$. 

\begin{lemma}
Any superharmonic vector with respect to $\hat{U}$ is superharmonic with respect to $U$ with $P=\hat{P}$. If $J^*$ is the fixed point of $U$ with $P=\hat{P}$ then $\hat{x}_s:=J^*_s + \min_{a\in\Ac} \CB_{\min}(s,a)(x)$ is superharmonic for $\hat{U}$. Then we have 
\[ \hat{x} \le \hat{J}^* \le J^*\] and \[\lim_{\epsilon\to0}\hat{x} = \lim_{\epsilon\to 0}\hat{J}^* =  J^*.\]
\end{lemma}
\begin{proof}
As $\CB_{\min}$ is negative, then any superharmonic vector with respect to $\hat{U}$ is superharmonic with respect to $U$ with $P=\hat{P}$. If $J^*$ is the fixed point of $U$ with $P=\hat{P}$ then for each $s\in\Sc$ we have 
\begin{align*}
    J^*_s &\le  c(s,a) + \sum_{s'\in \mathcal{S}} \hat{P}(s'|s,a) J^*_{s'} \quad \forall a\in\Ac,\\
    \Leftrightarrow \quad J^*_s +  \CB_{\min}(s,a)(J^*_s) &\le c(s,a) + \sum_{s'\in \mathcal{S}} \hat{P}(s'|s,a) J^*_{s'} +  \CB_{\min}(s,a)(x) \quad \forall a\in\Ac,\\
        \Rightarrow \quad J^*_s + \min_{a\in\Ac} \CB_{\min}(s,a)(J^*_s) &\le c(s,a) + \sum_{s'\in \mathcal{S}} \hat{P}(s'|s,a) J^*_{s'} +  \CB_{\min}(s,a)(J^*_s) \quad \forall a\in\Ac
\end{align*}
So the vector $\hat{x}_s =J^*_s + \min_{a\in\Ac} \CB_{\min}(s,a)(x)$ is superharmonic for $\hat{U}$ and is a lower bound for $\hat{J}_s^*$, which is a lower bound for $J_s^*$. Then we see that 
\begin{align*}
    \hat{x}_s = \left(J^*_s + \min_{a\in\Ac} \CB_{\min}(s,a)(x)\right) \le \hat{J}^*_s \le J^*_s.
\end{align*}
As $\epsilon$ tends to zero, then $\CB_{\min}$ tends to zero, and therefore $\hat{x}_s$ converges to $J_s^*$ and, by the Squeeze Theorem (see for example \cite[Thm.~3.19]{Rudin1976} or \cite[Chpt.~1~\S~1 Prop.~1.7]{Knapp2016}), so to does $\hat{J}_s^*$.
\end{proof}

Now finding a fixed point of $\hat{U}$ is equivalent to finding an superharmonic $x\in \R^N$ that is maximum in every element, and is therefore equivalent to maximising the sum of its elements. This means that fixed points of $\hat{U}$ can be found by solving the following convex optimisation program
\begin{align}
    \max_{x\in \R^N} &\sum_{s\in\Sc} x_s \label[Program]{CP primalUnknown} \tag{A}\\\nonumber
    & \text{such that}\\\nonumber
     x_s \le  c(s,a) &+ \sum_{s'\in \mathcal{S}} \hat{P}(s'|s,a) x_{s'} + \CB_{\min}(s,a)(x) \quad \forall s\in \Sc, a\in \Ac.
\end{align}
Here we can say that \Cref{CP primalUnknown} is convex by considering the constraint to define a convex region, i.e. the set of all superharmonic vectors.

We can again apply Lagrange multipliers $q(s,a)\ge 0$ as in \citet[\S5.1.2]{Boyd} to determine the Lagrange dual function of \Cref{CP primalUnknown} as follows
\begin{align*}
    L(q) &=\max_{x\in \R^N}\sum_{s=1}^N x_s-\sum_{s,a}q(s,a)\left(x_s -  c(s,a) - \sum_{s'\in S} P(s'|s,a)x_{s'}- \CB_{\min}(s,a)(x)\right)\\
    & = \max_{x\in \R^N}\min_{\tilde{P}\in\mathcal{P}_{\epsilon}}\sum_{s=1}^N x_s\left(1-\sum_a q(s,a)+\sum_{s',a}q(s',a)\tilde{P}(s|s',a)\right) +\sum_{s,a}q(s,a)c(s,a).
\end{align*} 
If we can exchange the maximum and the minimum in this equation (indeed we show we can in \Cref{prop:dualityOPsUnknown}), then taking the minimum over $q(s,a)$ returns the following non-convex program, which we consider to be the dual program to \Cref{CP primalUnknown}
\begin{align}
    \min_{q(s,a)\ge 0,\tilde{P}\in\mathcal{P}_{\epsilon} } &\sum_{s,a}q(s,a)c(s,a) \label[Program]{OP dualUnknown}\tag{B}\\\nonumber
    & \text{such that} \\\nonumber
     \sum_{a}q(s,a)&= 1+\sum_{s',a}q(s',a)\tilde{P}(s|s',a) \quad \forall s\in \Sc.
\end{align} This is non-convex due to the product $q(s',a)\tilde{P}(s|s',a)$ of variables.

We now show that  \Cref{CP primalUnknown} and  \Cref{OP dualUnknown} exhibit strong duality, in the sense that they have the same optimal solution.

\begin{proposition}\label{prop:dualityOPsUnknown}
\Cref{CP primalUnknown} and \Cref{OP dualUnknown} have the same optimal solution whenever there is a feasible solution to \Cref{LP dualKnown} with $P=\hat{P}$. We consider these programs to be dual optimisation programs.
\end{proposition}
\begin{proof}
Following \citet[lem.~10]{neuPikeBurke} we define the variables 
\[M(s',s,a) = q(s',a)\tilde{P}(s|s',a)\ge 0.\] Using \Cref{eqn:substochastic}, we require \[ \sum_{s}M(s',s,a) \le q(s',a)\] for all $s'\in\Sc$ and $a\in\Ac$. 

Then we have that 
\begin{align*}
    D(\tilde{P}(\cdot|s,a),\hat{P}(\cdot|s,a)) \le \epsilon(s,a)
\end{align*}
is equivalent to
\[
D(M(s,a,\cdot), q(s,a)\hat{P}(\cdot | s,a)) \leq q(s,a)\epsilon(s,a)
\] 
whenever $q(s,a)$ is positive using that $D$ is positive homogeneous.

If we now consider the following convex program
\begin{align}
    \min_{q(s,a)\ge 0,M(s',s,a)\ge 0 } &\sum_{s,a}q(s,a)c(s,a) \label[Program]{CP dualKnown2}\\\nonumber
    & \text{such that} \\\nonumber
     \sum_{a}q(s,a)&= 1+\sum_{s',a}M(s',s,a) \quad \forall s\in \Sc,\\\nonumber
     D(M(s,a,\cdot), q(s,a)\hat{P}(\cdot | s,a)) &\leq q(s,a)\epsilon(s,a),\\\nonumber
     \sum_{s}M(s',s,a) &\le q(s',a),
\end{align} 
we can see that this is equivalent to \Cref{OP dualUnknown}  by considering that any solution to \Cref{OP dualUnknown} is a solution to \Cref{CP dualKnown2} using the definition of $M$, and any solution to \Cref{CP dualKnown2} is a solution to \Cref{OP dualUnknown} by setting 
\[ 
\tilde{P}(s'| s,a) = \begin{cases} \frac{M(s,a,s')}{q(s,a)} & q(s,a)>0\\                                            
\hat{P}(s'| s,a) & q(s,a)=0.
\end{cases}
\]

Such a convex problem has zero duality gap if we can establish that a constraint qualification exists. In this case, we can use Slater's condition \cite[eq.~5.27,\S5.2.3, pg.~226-227]{Boyd}, which says that \Cref{CP dualKnown2} has zero duality gap if there exists is $\hat{M}, \hat{q}$ such that the constraints in  \Cref{CP dualKnown2}  are satisfied and \[D(\hat{M}(s,a,\cdot), \hat{q}(s,a)\hat{P}(\cdot \mid s,a)) < \hat{q}(s,a)\epsilon (s,a).\]
This is true for $\hat{q}(s,a)$ any feasible solution to \Cref{LP dualKnown} with $P(s'\mid s,a)=\hat{P}(s'\mid, s,a)$, and $ M(s',a,s) = \hat{P}(s\mid s',a)\hat{q}(s',a)$. This establishes zero duality gap for the convex dual to \Cref{CP dualKnown2}.

If we take the Lagrangian of \Cref{CP dualKnown2} we have
\begin{align*}
   & \underset{x, \gamma \geq 0, \lambda \geq 0}{\max} \, \underset{q \geq 0, M\ge 0}{\min} \Bigg \{ \sum_{s\in\Sc} \sum_{a\in \Ac} c(s,a) q(s,a) - \sum_{s=1}^N x_s \left( \sum_a q(s,a) - \sum_a \sum_{s' = 1}^N M(s', a,s) - 1 \right) + 
   \\&\quad \sum_{s,a} \gamma (s,a) \left(D(M(s,a, \cdot), q(s,a) \hat{P}(\cdot \mid s,a)) - \sum q(s,a) \right) + \sum_{s',a} \lambda (s',a) \left( \sum_sM(s',a,s) - q(s',a) \right)\Bigg\} \end{align*} which is equal to
   \begin{align*}& \underset{x, \gamma \geq 0, \lambda \ge 0}{\max} \,\underset{q \geq 0, \tilde{P}\ge 0}{\min} \Bigg\{ \sum_{s=1}^N \sum_{a\in A(S)} c(s,a) q(s,a) - \sum_{s=1}^N x_s \left( \sum_a q(s,a) - \sum_a \sum_{s' = 1}^N \tilde{P}(s\mid s',a)q(s',a) - 1 \right) + \\ 
    &\quad \sum_{s,a}\gamma(s,a)q(s,a) \left( D( \tilde{P}(\cdot \mid s,a), \hat{P}(\cdot \mid s,a)) - \epsilon(s,a) \right) + \sum_{s',a} \lambda (s',a)q(s',a) \left( \sum_s\tilde{P}(s'\mid a,s) - 1 \right) \Bigg\} .\end{align*}
    We can see that this is equivalent to 
    \begin{align*}
    & \underset{x, \bar{\gamma} \geq 0,\bar{\lambda}\ge0}{\max}\, \underset{q \geq 0, \tilde{P}\in\mathcal{P}}{\min} \Bigg\{ \sum_{s\in\Sc} \sum_{a\in \Ac} c(s,a)q(s,a) - \sum_{s} x_s \left( \sum_a q(s,a) - \sum_a \sum_{s=1}^N \tilde{P}(s \mid s', a)q(s', a) - 1\right) 
    \\ &+ \quad \sum_{s,a} \tilde{\gamma}(s,a) \left( D(\tilde{P}(\cdot \mid s,a), \hat{P}( \cdot \mid s,a) ) - \epsilon (s,a) \right) + \sum_{s',a} \bar{\lambda} (s',a) \left( \sum_s\tilde{P}(s'\mid a,s) - 1 \right) \Bigg\}
\end{align*}
which is the Lagrangian of \Cref{OP dualUnknown}. This equivalence follows by first checking that any solution to the former gives a solution to the latter by setting $\bar{\lambda}(s,a)=\lambda(s,a)q(s,a)$ and $\bar{\gamma}(s,a)=\bar{\gamma}(s,a)q(s,a)$. Then checking vice-versa by considering that we can set
\[ 
\lambda(s,a)  = \begin{cases} \frac{\bar{\lambda}(s,a)}{q(s,a)} & q(s,a)>0\\
0 & q(s,a)=0,
\end{cases}
\]
and
\[ 
\gamma(s,a)  = \begin{cases} \frac{\bar{\gamma}(s,a,s')}{q(s,a)} & q(s,a)>0\\
0 & q(s,a)=0.
\end{cases}
\] 
Here we note that whenever $q(s,a)=0$ we must have $D(\tilde{P}(\cdot \mid s,a), \hat{P}( \cdot \mid s,a) ) - \epsilon (s,a)\le 0$ and $\sum_s\tilde{P}(s'\mid a,s) - 1\le 0$ for a non-infinite solution, in which case, maximising over $\bar{\lambda}$ and $\bar{\gamma}$ would set both of these parameters to zero.

As the Lagrangian of \Cref{OP dualUnknown} is equivalent to the Lagrangian of  \Cref{CP dualKnown2} then their dual programs must be equivalent.

Therefore we have shown that \Cref{OP dualUnknown} is equivalent to the convex program, \Cref{CP dualKnown2}, that \Cref{CP dualKnown2} has zero duality gap and therefore the same optimal solution as its dual, and its dual is equivalent to \Cref{CP primalUnknown}. Therefore there is no duality gap between \Cref{OP dualUnknown} and \Cref{CP primalUnknown}. 

\end{proof}

In general, we can see that convex programming techniques can be used to solve for the fixed points of the operators. The dual variables $q(s,a)$ can again be considered occupancy measures as in \Cref{subsubsec:occupancy} and the corresponding policy given by \Cref{eqn:qtopi}. The requirement of feasibility in \Cref{prop:dualityOPsUnknown} is equivalent to $\hat{P}$ determining an MDP that satisfies \Cref{assu:proper}. 

We remark that the EVI operator $\hat{U}$ is related to solving \Cref{CP primalUnknown}, however there has been no study in the literature of the dual program in \Cref{OP dualUnknown} nor its related policy iteration methods for SSPs. Such methods have been studied in \citet{Kaufman2011} and \citet{Auer2006} for discounted infinite horizon and finite horizon MDPs respectively, and these approaches could be adapted here. This is possible future research that we do not do here.

Note that \Cref{CP primalUnknown} and the operator $\hat{U}$ are complicated by the presence of $\CB_{\min}(s,a)(x)$. If we can find a suitable bound on $\CB_{\min}$ we may be able to develop nice approximations that are less computationally intensive to compute. We do this in \Cref{sec:bounds}.

\section{\texorpdfstring{Bounds on $\CB_{\min}$}{Bounds on CBmin}}\label{sec:bounds}
Here we study bounds of $\CB_{\min}$. Each bound depends on the metric or divergence used to determine $\mathcal{P}_{\epsilon}$. Many of the examples can be considered the SSP versions of the bounds derived by \citet{neuPikeBurke} for finite horizon MDPs.

\subsection{\texorpdfstring{The $\ell_1$-norm}{The l-1 norm}} \label{subsec:l1norm}
Here we take the $\ell_1$ norm, $D(P_1,P_2) = \lVert P_1-P_2 \rVert_1$. For EVI, this case is studied in \citet[Alg.~1]{rosenbergCohen} and \citet{Tarbouriech2019}.

In this case $\CB_{\min}(s,a)$ reduces to checking which $\tilde{P}$ of a set of points minimise $\sum_{s'}P(s'|s,a)x_{s'}$. As with the simplex algorithm for linear programs, these points are the vertices of the feasible region. This region is the intersection the $\ell_1$ epsilon ball and $\Delta$, the set of all valid transition probabilities, where \[ \Delta =\{P\mid \sum_{s'}P(s'|s,a)\le 1, P(s'|s,a)\ge 0, \forall s',s\in \mathcal{S}, a \in \mathcal{A}\}.\] 

However, due to the non-negativity of $x$, and the non-negativity of $\Delta$, only a subset of $N$ of these points needs to be checked, one for each state. Which of these points end up determining $\tilde{P}$ relies heavily on $x$. This is illustrated in \Cref{fig:ell1normregion}. 

\begin{figure}
    \centering
\definecolor{uuuuuu}{rgb}{0.26666666666666666,0.26666666666666666,0.26666666666666666}
\definecolor{zzqqtt}{rgb}{0.6,0.,0.2}
\definecolor{qqttcc}{rgb}{0.,0.2,0.8}
\definecolor{cqcqcq}{rgb}{0.7529411764705882,0.7529411764705882,0.7529411764705882}
\begin{tikzpicture}[line cap=round,line join=round,>=triangle 45,x=8.0cm,y=8.0cm]
\draw [color=cqcqcq,, xstep=0.8cm,ystep=0.8cm] (-0.242979794833089,-0.23537905057446323) grid (1.0718709439830993,1.0659959120020261);
\draw[->,color=black] (-0.242979794833089,0.) -- (1.0718709439830993,0.);
\foreach \x in {-0.2,0.2,0.3,0.4,0.5,0.6,0.7,0.8,0.9,1.}
\draw[shift={(\x,0)},color=black] (0pt,2pt) -- (0pt,-2pt) node[below] {\footnotesize $\x$};
\draw[->,color=black] (0.,-0.23537905057446323) -- (0.,1.0659959120020261);
\foreach \y in {-0.2,0.2,0.3,0.4,0.5,0.6,0.7,0.8,0.9,1.}
\draw[shift={(0,\y)},color=black] (2pt,0pt) -- (-2pt,0pt) node[left] {\footnotesize $\y$};
\draw[color=black] (0pt,-10pt) node[right] {\footnotesize $0$};
\clip(-0.242979794833089,-0.23537905057446323) rectangle (1.0718709439830993,1.0659959120020261);
\fill[line width=2.pt,color=qqttcc,fill=qqttcc,fill opacity=0.10000000149011612] (0.,1.) -- (0.,0.) -- (1.,0.) -- cycle;
\fill[line width=2.pt,color=zzqqtt,fill=zzqqtt,fill opacity=0.10000000149011612] (0.5,0.4) -- (0.2,0.1) -- (0.5,-0.2) -- (0.8,0.1) -- cycle;
\draw [line width=2.pt,color=qqttcc] (0.,1.)-- (0.,0.);
\draw [line width=2.pt,color=qqttcc] (0.,0.)-- (1.,0.);
\draw [line width=2.pt,color=qqttcc] (1.,0.)-- (0.,1.);
\draw [->,line width=2.pt,color=zzqqtt] (0.5,0.1) -- (0.2,0.1);
\draw [->,line width=2.pt,color=zzqqtt] (0.5,0.1) -- (0.5,0.);
\draw [line width=2.pt,color=zzqqtt] (0.5,0.4)-- (0.2,0.1);
\draw [line width=2.pt,color=zzqqtt] (0.2,0.1)-- (0.5,-0.2);
\draw [line width=2.pt,color=zzqqtt] (0.5,-0.2)-- (0.8,0.1);
\draw [line width=2.pt,color=zzqqtt] (0.8,0.1)-- (0.5,0.4);
\draw [line width=3.6pt] (0.2,0.1)-- (0.3,0.);
\begin{scriptsize}
\draw [fill=black] (0.5,0.1) circle (2.5pt);
\draw[color=black] (0.5665293349944471,0.14001757324567793) node {$\hat{P}_{s,a}$};
\draw[color=zzqqtt] (0.388456577541303,0.1409801286913706) node {$\min\{\epsilon(s,a),0\}$};
\draw[color=zzqqtt] (0.6502716587697096,0.04664969501348898) node {$\min\{\epsilon(s,a),0\}$};
\draw [fill=uuuuuu] (0.2,0.1) circle (2.5pt);
\draw[color=uuuuuu] (0.17476926859753014,0.1506056831482973) node {$Q$};
\draw [fill=uuuuuu] (0.3,0.) circle (2.5pt);
\draw[color=uuuuuu] (0.2613992587098705,-0.039980295098851294) node {$T$};
\end{scriptsize}
\end{tikzpicture}
    \caption{This shows the valid probability transition functions $\Delta$ in the shaded blue region for a given state-action pair $(s,a)\in\mathcal{S}\times\mathcal{A}$, where $|\Sc|=2$. Shaded in red is the region such that the $\ell_1$ norm distance from $\hat{P}_{s,a}$ does not exceed $\epsilon(s,a)$. The bottom left-hand corner of the intersection of these two regions  shows the points $Q$ and $T$ which are the candidates for $\tilde{P}$. Which of these points determines $\tilde{P}$ will depend on $x$.}
    \label{fig:ell1normregion}
\end{figure}
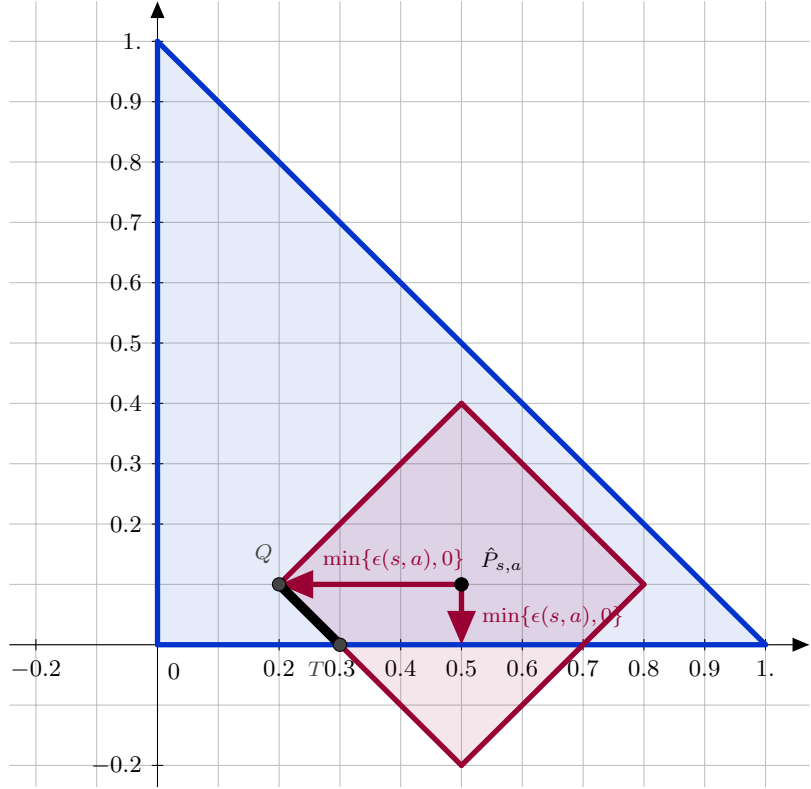

We now describe a lower bound for $\CB_{\min}(s,a)(x)$ when $D(P_1,P_2) = \lVert P_1-P_2 \rVert_1$. We also require that $\tilde{P}\in \Delta$. We then have that
\begin{align*}
\CB_{\min}(s,a)(x) & = \min_{\tilde{P}\in \Delta} \{\langle x,\tilde{P}-\hat{P}\rangle \mid D(\tilde{P},\hat{P}) \leq \epsilon(s,a) \} \\
& =  \max_{\lambda \le 0} \min_{\tilde{P}\geq 0} \{\langle x-\lambda 1,\tilde{P}-\hat{P}\rangle  - \lambda \hat{P}(g|s,a) \mid D(\tilde{P},\hat{P}) \leq \epsilon(s,a) \} \\
& = - \min_{\lambda \le 0} \max_{\tilde{P}\geq 0} \{-\langle x-\lambda 1,\tilde{P}-\hat{P}\rangle  + \lambda \hat{P}(g|s,a)\mid D(\tilde{P},\hat{P}) \leq \epsilon(s,a) \}\\
& \geq - \min_{\lambda \le 0} \max_{\tilde{P}\in \R^S} \{-\langle x-\lambda 1,\tilde{P}-\hat{P}\rangle +  \lambda \hat{P}(g|s,a)\mid D(\tilde{P},\hat{P}) \leq \epsilon(s,a) \}\\
& \geq - \min_{\lambda \le 0} \max_{\tilde{P}\in \R^S, D(\tilde{P},\hat{P})\le \epsilon(s,a)} - - \lVert x-\lambda 1 \rVert_\infty \lVert\tilde{P}-\hat{P}\rVert_1  + \lambda \hat{P}(g|s,a)\\
& \geq -  \min_{\lambda \le 0}\epsilon(s,a) \lVert x-\lambda 1 \rVert_\infty + \lambda \hat{P}(g|s,a)\\
& = - \epsilon(s,a) \ \max(x).
\end{align*}
Note that here we use the notation $\langle,\rangle$ as the euclidean inner product of vectors, and we have dropped the $s,a$ on each of the $\hat{P},\tilde{P}$ so that for example
\[ \langle x,\tilde{P}-\hat{P}\rangle = \sum_{s'\in\Sc} x_{s'}(\tilde{P}(s'|s,a)-\hat{P}(s'|s,a)).\]

To move from line 1 to line 2, we use the constraint that \[\sum_{s'\in\Sc}\tilde{P}(s'|s,a)\le 1 = \sum_{s'\in\Sc}\hat{P}(s'|s,a)+P(g|s,a),\] which results in the term
\[ -\lambda\left(\sum_{s'\in\Sc}\tilde{P}(s'|s,a)  - \sum_{s'\in\Sc}\hat{P}(s'|s,a)-P(g|s,a)\right) = \langle -\lambda 1,\tilde{P}-\hat{P}\rangle - \lambda \hat{P}(g|s,a)  .\] 
To move from line 6 to 7, we use that \[\min_{\lambda\le 0}\lambda \tilde{P}(g|s,a) =\lambda \tilde{P}(g|s,a)\Big{|}_{\lambda = 0} =0\] as $\tilde{P}(g|s,a)\ge 0$ and that  \[\min_{\lambda\le 0} \lVert x - \lambda 1\rVert_{\infty} = \lVert x - \lambda 1\rVert_{\infty}\Big{|}_{\lambda = 0} = \max(x) \] for $x\in \R^N$ and $x\ge 0$ , where the proof of this is in \Cref{appn:proofSpan}. 

\begin{rmk}\label{rmk:includegoalstatel1}
Note that we could include the goal state and define $\Delta$ such that $\sum_{s'}P(s'|s,a)= 1$. Then we would require $\lambda \in \R$ rather than $\lambda \le 0$ in the above equation. This would result in 
\[ \CB_{\min}(s,a)(x) \ge  - \epsilon(s,a) \ \spn(x) ,\]  where $\spn$ is the called the \emph{span} of $x$ and is equal to $\frac{1}{2} (\max(x) - \min(x)) $. This matches the result of \citet[Tab.~1]{neuPikeBurke}. We also consider this case in \Cref{appn:proofSpan}. However, if we are including the goal state, then the cost-to-go of the goal state is zero, and through our analysis we will ensure that $x\ge 0$ so we have $\min(x)=0$ giving the same result as above to a factor of $\frac{1}{2}$.   
\end{rmk}

We define the \emph{inflated exploration bonus} $\CB_{\min}^{\dagger}(s, a)(x)$ by 
\[\CB_{\min}^{\dagger}(s, a)(x)= -\epsilon(s,a) \ \max(x)\] 
for $\alpha>0$ corresponding to the divergence or norm used, as above.

We also note that 
\[\hat{P}(\cdot|s,a)x(\cdot) + {\CB_{\min}}(s, a)(x)=\min_{\tilde{P}\in \Delta_{\hat{P},\epsilon}} \tilde{P}(\cdot|s,a)x(\cdot)\ge 0,\]
so a more appropriate bound for $\CB_{\min}$ is
\begin{equation}\label{eqn:hatDaggerBalloon}
\CB_{\min}^{\dagger,0}(s, a)(x): = \max\left\{{\CB_{\min}}^{\dagger}(s, a)(x),-\hat{P}(\cdot|s,a)x(\cdot)\right\}.
\end{equation}
We can write the operator $U$ adjusted by this bound as
\[(\hat{U}^{\dagger,0}x)_s = \min_{a\in \mathcal{A}} \left\{ c(s,a) +  \max\left\{\sum_{s'\in \mathcal{S}} \hat{P}(s'|s,a) x_{s'} -\epsilon(s,a)\max(x),0\right\}\right\}. \]

\begin{rmk} \label{rmk:whyboundbycosts}
Note that we do require a lower bound on $\CB_{\min}^{\dagger}(s, a)(x)$, else the operator $U$ may be unbounded below. Our approach to bound by $-\hat{P}(\cdot|s,a)x(\cdot)$, which will effectively bound the modified operator $\hat{U}^{\dagger,0}$ by the minimum of the costs, is different to the approach of \citet{neuPikeBurke} and also to \citet{Tarbouriech2021}, who both bound their operators by below 0. We observed bounding by 0 lead to many cases where the operator did not converge but instead oscillates, as we discuss further in \Cref{ex:oscillatingBehaviour}. For \citet{neuPikeBurke} this is not an issue as the horizon is finite.

Unlike our approach, \citet{Tarbouriech2021} consider the unknown costs case, where they allow the bound to be 0 to represent possible zero costs as a optimistic approach. They also prove their operator is monotone and converges regardless. This is not true for our operator $\hat{U}^{\dagger,0}$, and we discuss the property of monotone and whether the operator converges further in \Cref{appn:conjecture2statecase}.
\end{rmk}

The corresponding optimisation program to $\hat{U}^{\dagger,0}x$ is as follows:
\begin{align}
    \max_{x\in \R^N} &\sum_{s\in\Sc} x_s \label[Program]{CP primalUnknowndaggerl1} \\\nonumber
    & \text{such that}\\\label{eqn:superharmonicl1dagger}
     x_s \le  c(s,a) &+ \max\left\{\sum_{s'\in \mathcal{S}} \hat{P}(s'|s,a) x_{s'} -\epsilon(s,a)\max(x),0\right\} \quad \forall s\in \Sc, a\in \Ac.
\end{align}
Note that vectors that satisfy \Cref{eqn:superharmonicl1dagger} must also be superharmonic with respect to $\CB_{\min}$, so that the optimal value must be bounded above by the optimal value of \Cref{CP primalUnknown}. Also, the vector $x$ with $x_s = \min_{a\in \mathcal{A}}c(s,a)$ is feasible, and its sum is a lower bound on the optimal value of \Cref{CP primalUnknowndaggerl1}. This means there is a unique optimal value of this optimisation program, however unlike \Cref{CP primalUnknown} it is not clear that there is a unique $x$ corresponding to this optimal value. Also, \Cref{CP primalUnknowndaggerl1} has a convex feasible region if restricted to any of the sets 
\[S_{U\cup V}=\{x\in\R^N|x_{u}> \min_{a\in\Ac}c(u,a),x_{v}= \min_{a\in\Ac}c(v,a),\;\forall u\in U,v\in V \}, \] where $U$ and $V$ are disjoint sets that partition $\mathcal{S}$. This allows convex programming techniques can be used on each $S_{U\cup V}$, of which there are exactly $2^N$ regions.
 
We have that any fixed point of $\hat{U}^{\dagger,0}$ is automatically superharmonic with respect to \Cref{eqn:superharmonicl1dagger}, so \Cref{CP primalUnknowndaggerl1} must have optimal value greater than or equal to the sum of any fixed point's elements.

However, in the case where $\epsilon = 0$, this is equivalent to \Cref{LP primalKnown} with $P=\hat{P}$, where there is a unique optimal solution. And if $\epsilon$ is greater than 1, then the constraint in \Cref{eqn:superharmonicl1dagger} reduces to $x_s\le c(s,a)$ giving the unique optimal solution of $x_s = \min_{a\in \mathcal{A}}c(s,a)$. These cases result in same outcome as solving \Cref{CP primalUnknown}. So it is only the cases where $0<\epsilon <1$ that needs to be addressed. The same considerations apply to the operator $\hat{U}^{\dagger,0}$, where for each $x$ with $x_s\ge \min_{a\in \mathcal{A}}c(s,a)\in \R^n$ we have 
\[ 0\le \min_{a\in \mathcal{A}}c(s,a) \le (\hat{U}^{\dagger,0}x)_s \le (\hat{U}x)_s \le (Ux)_s \] where $P=\hat{P}$ in the operator $U$. The first, second and fourth inequality also hold for $x\in \R^N$. This means that the fixed points of $\hat{U}^{\dagger,0}$, where they exist, must be between $\min_{a\in \mathcal{A}}c(s,a)$ and $\hat{J}^*$, the fixed point of $\hat{U}$.

\begin{lemma} \label{lem:existenceFixedPoint}
There exists a fixed point of $\hat{U}^{\dagger,0}$ in the region \[ X = \{x\in\R^N \mid  \min_{a\in \mathcal{A}}c(s,a)\le x_s\le (\hat{J}^*)_s ,\quad \forall s=1,2,\ldots,N\}.  \]
\end{lemma}
\begin{proof}
Apply the Brouwer fixed point theorem \cite[Cor.~1.1.1]{Florenzano2003} to $\hat{U}^{\dagger,0}$ on $X$.
\end{proof}

\begin{conjecture}\label{conj:l1dagger}
There exists a unique fixed point of $\hat{U}^{\dagger,0}$, which is equal to the unique optimal solution to \Cref{CP primalUnknowndaggerl1}. In almost all cases, this can be found by iterating the operator $\hat{U}^{\dagger,0}$ starting at any point in $\R^N$.
\end{conjecture}
We have already discussed that this holds when $\epsilon =0$ or when $\epsilon\ge 1$. In practice, the author has seen this result hold true empirically for a wide range of examples and edge cases and we discuss this further in \Cref{appn:conjecture2statecase}. In \Cref{appn:conjecture2statecase}, exploration of the case where $N=2$ still determined a single fixed point and this was equal to optimising \Cref{CP primalUnknowndaggerl1}, however the operator did not converge to fixed point, instead oscillating between two points. Note that $\epsilon$ must be large enough for this to occur, while shrinking epsilon leads to the operator again contracting to a fixed point. Other than this, \emph{this conjecture is an open question and we appreciate any thoughts or comments on this (or indeed a proof or counter-example)!}

\begin{rmk} \label{rem:modifiedP}
As an aside, in proving \Cref{conj:l1dagger} we may want to make assumptions such as all policies are proper. In practice, this means we may need to adjust $\hat{P}$ to $\hat{P}^*$, where we tilt $\hat{P}$ towards the goal state by setting
\begin{equation*} 
\hat{P}^*(s'|s,a) = \begin{cases} \mathbb{I}\{ \hat{P}(s'|s,a)==0\}\frac{1}{n(s,a)+1}+ \mathbb{I}\{ \hat{P}(s'|s,a)\ne 0\}\hat{P}(s'|s,a) & \text{ for } s'=g,\\
 \hat{P}(s'|s,a)\frac{n(s,a)}{n(s,a)+\mathbb{I}\{ \hat{P}(s'|s,a)==0\}} & \text{ for } s'\ne g.\\
\end{cases} 
\end{equation*} 
Here $\mathbb{I}$ is the indicator function that is equal to 1 when applied to a true statement and zero otherwise, and $n(s,a)$ is the number of visits to the state $(s,a)$ that has occurred by running the algorithm. In doing this, we may want to increase $\epsilon(s,a)$ to $\epsilon^*(s,a)$ by considering that 
\begin{align*}\lVert \tilde{P}(\cdot|s,a) - \hat{P}^*(\cdot|s,a)\rVert_1 &= \lVert \tilde{P}(\cdot|s,a) -\hat{P}(\cdot|s,a) + \hat{P}(\cdot|s,a)- \hat{P}^*(\cdot|s,a)\rVert_1 \\
& \le  \lVert \tilde{P}(\cdot|s,a) -\hat{P}(\cdot|s,a)\rVert + \lVert\hat{P}(\cdot|s,a)- \hat{P}^*(\cdot|s,a)\rVert_1\\
& \le \epsilon(s,a) + \frac{1}{1+n(s,a)}\\
& =: \epsilon^*(s,a).\end{align*} Using $\epsilon^*$ instead of $\epsilon$ ensures that enough learning can take place in the algorithm by not shrinking the region around $\hat{P}^*$ too quickly.  Using $\hat{P}^*$ ensures that $\sum_{s'\ne g} \hat{P}^*(s'|s,a)<1$ for all $s\in\Sc,a\in\Ac$, and would then also ensure that all policies are proper, which may be helpful for proofs including for \Cref{conj:l1dagger}.

We may additionally want to ensure all elements of $\hat{P}$ are non-zero, which we require for divergences as in \Cref{subsec:KLdiv}. In this case for each $s\in\Sc,a\in\Ac$ (or $s\in\Sc\cup\{g\},a\in\Ac$) we define 
\begin{equation*} 
\hat{P}^+(s'|s,a) =  \mathbb{I}\{ \hat{P}(s'|s,a)==0\}\frac{1}{n(s,a)+z(s,a)}+ \mathbb{I}\{ \hat{P}(s'|s,a)\ne 0\}\hat{P}(s'|s,a)\frac{n(s,a)}{n(s,a)+z(s,a)}
\end{equation*} 
where for each $s\in\Sc,a\in\Ac$ then $z(s,a)$ is the number states $s'\in\Sc$ (or $s'\in\Sc\cup\{g\}$) such that $ \hat{P}(s'|s,a)=0$. This ensures that $\sum_{s'\ne g} \hat{P}^*(s'|s,a)<1$ and $\hat{P}^+(s'|s,a)\ge 0$ for all $s,s'\in\Sc,a\in\Ac$. If we include $g$ when defining $\hat{P}^+$ then this also ensures all policies are proper. We then would define $\epsilon^+$ by considering
\begin{align*}\lVert \tilde{P}(\cdot|s,a) - \hat{P}^+(\cdot|s,a)\rVert_1 &= \lVert \tilde{P}(\cdot|s,a) -\hat{P}(\cdot|s,a) + \hat{P}(\cdot|s,a)- \hat{P}^+(\cdot|s,a)\rVert_1 \\
& \le  \lVert \tilde{P}(\cdot|s,a) -\hat{P}(\cdot|s,a)\rVert + \lVert\hat{P}(\cdot|s,a)- \hat{P}^+(\cdot|s,a)\rVert_1\\
& \le \epsilon(s,a) + \frac{2z(s,a)-\mathbb{I}\{ \hat{P}(s'|s,a)==0\}}{z(s,a)+n(s,a)}\\
& =: \epsilon^+(s,a).\end{align*}
\end{rmk}

\subsection{Supremum norm}\label{subsec:supnorm}
If we use the supremum norm $D(P_1,P_2) = \lVert P_1-P_2 \rVert_{\infty} = \sup_{s'}| P_1(s')-P_2(s')|$, we can find the exact value of $\tilde{P}$ that gives the minimum of $\hat{U}$. For EVI, this case is studied in \citet[Alg.~2]{rosenbergCohen}. It is straightforward to show that 
\[ \tilde{P}(s'|s,a) = \hat{P}(s'|s,a)-\min\{0,\epsilon(s,a)\},\] and we show how $\tilde{P}$ relates to $\hat{P}$ in \Cref{fig:supnormregion}. Note that unlike \Cref{subsec:l1norm}, this minimum does not depend on $x$. 

We then have that \begin{equation}\CB_{\min}(s,a) = \sum_{s'}\max\{-\epsilon(s,a)x_{s'},-\hat{P}(s'|s,a)x_{s'}\}.\label{eqn:CBminSupnorm}\end{equation} This matches the expression in \citet[eqn.~5]{rosenbergCohen}, provided $\epsilon(s,a)$ is defined as $28A(s,a)+4\sqrt{\hat{P}(s'|s,a)A(s,a)}$, using the notation of that paper. 

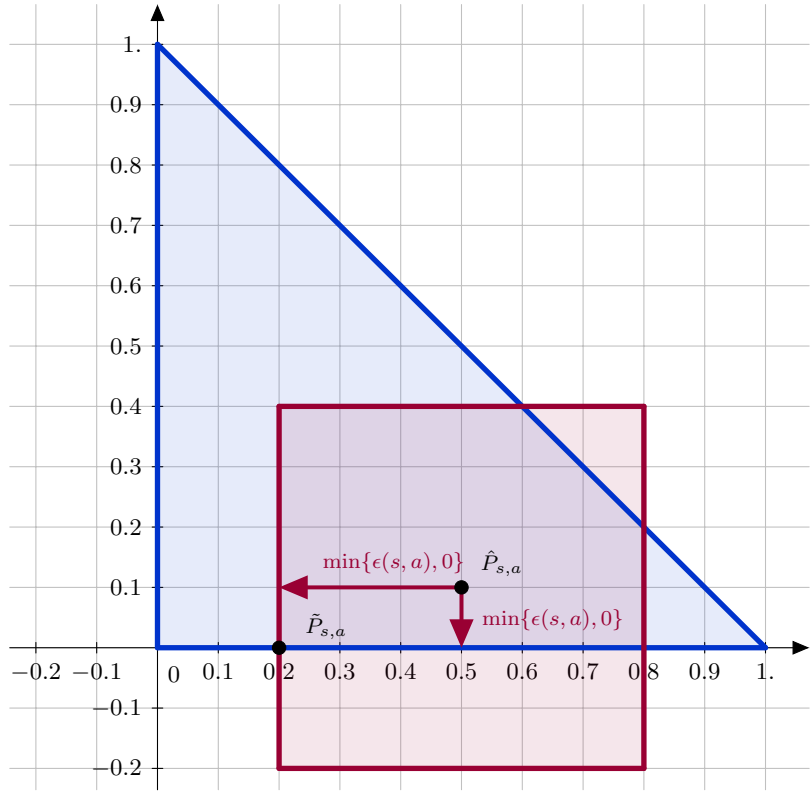
\begin{figure}
    \centering
\definecolor{zzqqtt}{rgb}{0.6,0.,0.2}
\definecolor{qqttcc}{rgb}{0.,0.2,0.8}
\definecolor{cqcqcq}{rgb}{0.7529411764705882,0.7529411764705882,0.7529411764705882}
\begin{tikzpicture}[line cap=round,line join=round,>=triangle 45,x=8.0cm,y=8.0cm]
\draw [color=cqcqcq,, xstep=0.8cm,ystep=0.8cm] (-0.242979794833089,-0.23537905057446323) grid (1.0718709439830993,1.0659959120020261);
\draw[->,color=black] (-0.242979794833089,0.) -- (1.0718709439830993,0.);
\foreach \x in {-0.2,-0.1,0.1,0.2,0.3,0.4,0.5,0.6,0.7,0.8,0.9,1.}
\draw[shift={(\x,0)},color=black] (0pt,2pt) -- (0pt,-2pt) node[below] {\footnotesize $\x$};
\draw[->,color=black] (0.,-0.23537905057446323) -- (0.,1.0659959120020261);
\foreach \y in {-0.2,-0.1,0.1,0.2,0.3,0.4,0.5,0.6,0.7,0.8,0.9,1.}
\draw[shift={(0,\y)},color=black] (2pt,0pt) -- (-2pt,0pt) node[left] {\footnotesize $\y$};
\draw[color=black] (0pt,-10pt) node[right] {\footnotesize $0$};
\clip(-0.242979794833089,-0.23537905057446323) rectangle (1.0718709439830993,1.0659959120020261);
\fill[line width=2.pt,color=qqttcc,fill=qqttcc,fill opacity=0.10000000149011612] (0.,1.) -- (0.,0.) -- (1.,0.) -- cycle;
\fill[line width=2.pt,color=zzqqtt,fill=zzqqtt,fill opacity=0.10000000149011612] (0.2,0.4) -- (0.8,0.4) -- (0.8,-0.2) -- (0.2,-0.2) -- cycle;
\draw [line width=2.pt,color=qqttcc] (0.,1.)-- (0.,0.);
\draw [line width=2.pt,color=qqttcc] (0.,0.)-- (1.,0.);
\draw [line width=2.pt,color=qqttcc] (1.,0.)-- (0.,1.);
\draw [line width=2.pt,color=zzqqtt] (0.2,0.4)-- (0.8,0.4);
\draw [line width=2.pt,color=zzqqtt] (0.8,0.4)-- (0.8,-0.2);
\draw [line width=2.pt,color=zzqqtt] (0.8,-0.2)-- (0.2,-0.2);
\draw [line width=2.pt,color=zzqqtt] (0.2,-0.2)-- (0.2,0.4);
\draw [->,line width=1.5pt,color=zzqqtt] (0.5,0.1) -- (0.2,0.1);
\draw [->,line width=1.5pt,color=zzqqtt] (0.5,0.1) -- (0.5,0.);
\begin{scriptsize}
\draw [fill=black] (0.5,0.1) circle (2.5pt);
\draw[color=black] (0.5665293349944471,0.14001757324567793) node {$\hat{P}_{s,a}$};
\draw[color=zzqqtt] (0.388456577541303,0.1409801286913706) node {$\min\{\epsilon(s,a),0\}$};
\draw[color=zzqqtt] (0.6502716587697096,0.04664969501348898) node {$\min\{\epsilon(s,a),0\}$};
\draw [fill=black] (0.2,0.) circle (2.5pt);
\draw[color=black] (0.27776270128664593,0.03606158511086961) node {$\tilde{P}_{s,a}$};
\end{scriptsize}
\end{tikzpicture}
    \caption{This shows the valid probability transition functions $\Delta$ in the shaded blue region for a given state-action pair $(s,a)\in\mathcal{S}\times\mathcal{A}$, where $|\Sc|=2$. Shaded in red is the region such that the supremum norm distance from $\hat{P}_{s,a}$ does not exceed $\epsilon(s,a)$. The bottom left-hand corner of the intersection of these two regions  is $\tilde{P}_{s,a}$.}
    \label{fig:supnormregion}
\end{figure}
This means no extra computational power is needed to determine $\tilde{P}$ nor $\CB_{\min}$ in this case. We note that \citet{neuPikeBurke} do not consider bounding this norm, potentially for this reason.

We can still follow \Cref{subsec:l1norm} and describe a lower bound for $\CB_{\min}(s,a)(x)$. Recall we require that $\tilde{P}\in \Delta$, where $\Delta$ is the set of all valid transition probabilities, so \[ \Delta =\{P\mid \sum_{s'}P(s'|s,a)\le 1, P(s'|s,a)\ge 0, \forall s',s\in \mathcal{S}, a \in \mathcal{A}\}.\] We then have that
\begin{align*}
\CB_{\min}(s,a)(x) & = \min_{\tilde{P}\in \Delta} \{\langle x,\tilde{P}-\hat{P}\rangle \mid D(\tilde{P},\hat{P}) \leq \epsilon(s,a) \} \\
& =  \max_{\lambda \le 0} \min_{\tilde{P}\geq 0} \{\langle x-\lambda 1,\tilde{P}-\hat{P}\rangle  - \lambda \hat{P}(g|s,a) \mid D(\tilde{P},\hat{P}) \leq \epsilon(s,a) \} \\
& = - \min_{\lambda \le 0} \max_{\tilde{P}\geq 0} \{-\langle x-\lambda 1,\tilde{P}-\hat{P}\rangle  + \lambda \hat{P}(g|s,a)\mid D(\tilde{P},\hat{P}) \leq \epsilon(s,a) \}\\
& \geq - \min_{\lambda \le 0} \max_{\tilde{P}\in \R^S} \{-\langle x-\lambda 1,\tilde{P}-\hat{P}\rangle +  \lambda \hat{P}(g|s,a)\mid D(\tilde{P},\hat{P}) \leq \epsilon(s,a) \}\\
& \geq - \min_{\lambda \le 0} \max_{\tilde{P}\in \R^S, D(\tilde{P},\hat{P})\le \epsilon(s,a)} - - \lVert x-\lambda 1 \rVert_1 \lVert\tilde{P}-\hat{P}\rVert_\infty  + \lambda \hat{P}(g|s,a)\\
& \geq -  \min_{\lambda \le 0}\epsilon(s,a) \lVert x-\lambda 1 \rVert_1 + \lambda \hat{P}(g|s,a)\\
& = - \epsilon(s,a) \lVert x\rVert_1.
\end{align*}
To move line 1 to line 2, we use the constraint that \[\sum_{s'\in\Sc}\tilde{P}(s'|s,a)\le 1 = \sum_{s'\in\Sc}\hat{P}(s'|s,a)+P(g|s,a),\] which results in the term
\[ -\lambda\left(\sum_{s'\in\Sc}\tilde{P}(s'|s,a)  - \sum_{s'\in\Sc}\hat{P}(s'|s,a)-P(g|s,a)\right) = \langle -\lambda 1,\tilde{P}-\hat{P}\rangle - \lambda \hat{P}(g|s,a)  .\] 
To move from line 6 to 7, we use that \[\min_{\lambda\le 0}\lambda \tilde{P}(g|s,a) =\lambda \tilde{P}(g|s,a)\Big{|}_{\lambda = 0} =0\] as $\tilde{P}(g|s,a)\ge 0$ and that  \[\min_{\lambda\le 0} \lVert x - \lambda 1\rVert_{1} = \lVert x - \lambda 1\rVert_{1}\Big{|}_{\lambda = 0} = \Vert x\rVert_1 \] for $x\in \R^N$ and $x\ge 0$ , where the proof of this is in \Cref{appn:weightedspan}. 
Given that we know $\CB_{\min}(s,a)\ge -\sum_{s'}P(s'|s,a)x_{s'}$ we can then take the bound 
\[\CB_{\min}^{\dagger,0} = \max\{- \epsilon(s,a) \lVert x\rVert_1, -\sum_{s'}P(s'|s,a)x_{s'}\},\]
similar to \Cref{eqn:hatDaggerBalloon}.

Note that if $\epsilon(s,a)$ is less than $\min_{s'}P(s'|s,a)$, then $\CB_{\min}(s,a)$ from \Cref{eqn:CBminSupnorm} is equal to $\CB_{\min}^{\dagger,0}$, and similarly if $\epsilon(s,a)$ is greater than $\max_{s'}P(s'|s,a)$, then $\CB_{\min}(s,a)$  is also equal to $\CB_{\min}^{\dagger,0}$. So they are only not equal when $\epsilon(s,a)$ is between $\min_{s'}P(s'|s,a)$ and $\max_{s'}P(s'|s,a)$, and this then depends on $x$. For example, if we have the case of \Cref{fig:supnormregion}, with $\hat{P}(s_1|s,a) = 0.5$, $\hat{P}(s_2|s,a) = 0.1$, $\epsilon(s,a) = 0.3$, and we let $x=(1,0.5)$ then $\CB_{\min}^{\dagger,0} = -0.45$ while $\CB_{\min}=-0.35$. If however $x=(1,1)$ then $\CB_{\min}^{\dagger,0}=\CB_{\min}(s,a)=-0.45$.

\subsection{KL-divergence}\label{subsec:KLdiv}

Here we consider the KL-divergence, 
\[ D_{KL}(\tilde{P}(\cdot|s,a),P(\cdot|s,a)) = \sum_{s'}\tilde{P}(s'|s,a)\log\frac{\tilde{P}(s'|s,a)}{P(s'|s,a)}. \] We have not seen the corresponding EVI operator that uses this divergence studied in the SSP literature. 

We note that this divergence requires $\tilde{P}$ to be zero whenever $P$ is to be well defined. As we do not want to make this restriction, we need to modify the definition of $\hat{P}$. Here we use $\hat{P}^+$ defined in \Cref{rem:modifiedP}, and let $\mathcal{S}_{0}\subseteq\mathcal{S}$ be the corresponding states where $\hat{P}$ is zero for a given $s,a$. In the appendix we consider a different modification used by \citet[pg.~21]{neuPikeBurke}. 

We have that 
\begin{align*}
    &D_{KL}(\tilde{P}(\cdot|s,a),\hat{P}^+(\cdot|s,a)) \\
    &= \sum_{s'\in\Sc}\tilde{P}(s'|s,a)\log\frac{\tilde{P}(s'|s,a)}{\hat{P}^+(s'|s,a)}\\
    &=\sum_{s'\in\Sc\setminus \Sc_0}\tilde{P}(s'|s,a)\log\frac{\tilde{P}(s'|s,a)}{\hat{P}(s'|s,a)\frac{n(s,a)}{n(s,a)+z(s,a)}}+\sum_{s'\in\Sc_0}\tilde{P}(s'|s,a)\log\frac{\tilde{P}(s'|s,a)}{\frac{1}{n(s,a)+z(s,a)}}\\
     &=\sum_{s'\in\Sc\setminus \Sc_0}\tilde{P}(s'|s,a)\log\frac{\tilde{P}(s'|s,a)}{\hat{P}(s'|s,a)} + \sum_{s'\in\Sc\setminus \Sc_0}\tilde{P}(s'|s,a)\left(\log\frac{\tilde{P}(s'|s,a)}{\frac{n(s,a)}{n(s,a)+z(s,a)}} + \log\frac{\tilde{P}(s'|s,a)}{\frac{1}{n(s,a)+z(s,a)}}\right)\\
     &=\epsilon +\sum_{s'\in\Sc_0}\tilde{P}(s'|s,a)\log\frac{\tilde{P}(s'|s,a)}{n(s,a)}\\
     &\le \epsilon +\sum_{s'\in\Sc_0}\tilde{P}(s'|s,a)\log\frac{1}{n(s,a)}\\
     &\le \epsilon,
\end{align*}
so we do not need to modify epsilon.

We now show three separate ways to bound $ \CB_{\min}(s,a)(x) $.

\subsubsection{Using Pinkser's inequality}\label{subsubsec:Pinkser}
We first use Pinkser's inequality to bound $ \CB_{\min}(s,a)(x) $. We have that
\begin{align*}
&\CB_{\min}(s,a)(x) \\
& = \min_{\tilde{P}\in \Delta} \{\langle x,\tilde{P}-\hat{P}^+\rangle \mid D_{KL}(\tilde{P},\hat{P}^+) \leq \epsilon(s,a) \} \\
& =  \max_{\lambda \le 0} \min_{\tilde{P}\in \Delta} \{\langle x,\tilde{P}-\hat{P}^+\rangle - \lambda\left( D_{KL}(\tilde{P},\hat{P}^+)-\epsilon(s,a)\right) \} \\
& = - \min_{\lambda \le 0} \max_{\tilde{P}\in \Delta} \{-\left(\langle x,\tilde{P}-\hat{P}^+\rangle - \lambda \left(D_{KL}(\tilde{P},\hat{P}^+)-\epsilon(s,a)\right)\right) \}\\
& \ge - \min_{\lambda \ge 0} \max_{\tilde{P}\in \Delta} \{\left(\langle x,\hat{P}^+-\tilde{P}\rangle  - \lambda \left(\frac{1}{2\log2}\lVert\tilde{P}-\hat{P}^+\rVert_{1}^2-\epsilon(s,a)\right)\right) \}  \\
& = - \min_{\lambda \ge 0} \max_{\tilde{P}\in \Delta} \{\left(\langle \frac{1}{\sqrt{\lambda}}x,\sqrt{\lambda}(\hat{P}^+-\tilde{P})\rangle  -  \frac{1}{2}\lVert\sqrt{\lambda}(\tilde{P}-\hat{P}^+)\rVert_{1}^2+\lambda\epsilon(s,a)\log2\right) \} \\
& \ge - \min_{\lambda \ge 0}  \{\left(\frac{1}{2\lambda}\lVert x\rVert_{\infty}^2+\lambda\epsilon(s,a)\log2\right) \} \\
& = - 2\lVert x\rVert_{\infty}\sqrt{\frac{\log2}{2}\epsilon(s,a)}.
\end{align*}
We use Pinkser's inequality (see \citet[pg.~26]{Yeung2008}) to move from line 4 to line 5, and we use the relationship between convex conjugates of norms from \cite[Ex.~3.27]{Boyd} to move from line 5 to line 6. Note that Pinkser's inequality requires that $\tilde{P}$ and $\hat{P}$ are stochastic (not substochastic) so we are implicitly using the goal state in each line, and noting that component of $x$ corresponding to the goal state is constrained to have value zero. We also use the result 
that the minimum with respect to $\lambda$ of $a\lambda + \frac{b}{\lambda}$ is $\sqrt{ab}$ occurring at $\lambda = \sqrt{\frac{b}{a}}$ as in \Cref{appn:minAlambdaBoneOnLambda}. 

\subsubsection{\texorpdfstring{Using the cumulant function and \Cref{appn: proofBoundCumulant}}{Using the cumulant function and the cumulant bound}}\label{subsubsec:cumulantAppendixD}
Instead of using Pinkser's inequality, we can use that the convex conjugate of a KL divergence is the cumulant function as in \citet[Cor.~4.14]{Boucheron2012} and the result in \Cref{appn: proofBoundCumulant}. Note that the result for \citet[Cor.~4.14]{Boucheron2012} uses a probability distribution, not the substochastic version that we have, so we need to add in the probability of the goal state as we do below.  We can then proceed from line 4 above as follows
\begin{align*}
&\CB_{\min}(s,a)(x) \\
& = - \min_{\lambda \ge 0} \max_{\tilde{P}\in \Delta} \{\langle -x,\tilde{P}-\hat{P}^+\rangle  - \lambda \left(D_{KL}(\tilde{P},\hat{P}^+)-\epsilon(s,a)\right) \} \\
& \ge - \min_{\lambda \ge 0}  \{\lambda\log\left(\sum_{s'}\hat{P}^+(s')e^{\frac{-1}{\lambda}(x_{s'}-\langle \hat{P}^+,x\rangle)}\right)+\lambda\epsilon(s,a) +\lambda\max_{\tilde{P}\in \Delta}\tilde{P}(g|s,a)\log\left(\frac{\tilde{P}(g|s,a)}{\hat{P}^+(g|s,a)}\right) \} \\
& \ge - \min_{\lambda \ge 0, \lambda \ge \lVert x-\langle \hat{P}^+,x\rangle \rVert_{\infty})}  \{\frac{1}{\lambda}\sum_{s'}\hat{P}^+(s')(x_{s'}-\langle \hat{P}^+,x\rangle)^2 + \lambda\epsilon(s,a) -\lambda \log\left(\hat{P}^+(g|s,a)\right)\} \\
& \ge - \min_{\lambda \ge 0, \lambda \ge \lVert x-\langle \hat{P}^+,x\rangle \rVert_{\infty})}  \{\frac{1}{\lambda}\sum_{s'}\hat{P}^+(s')(x_{s'}-\langle \hat{P}^+,x\rangle)^2 + \lambda\epsilon(s,a) \} \\
& = - \begin{cases} 2\sqrt{\hat{V}^+(x)\epsilon(s,a))} & \text{ if } \epsilon(s,a)\le f(x,\hat{P}^+),\\
\frac{1}{\lVert x-\langle \hat{P}^+,x\rangle \rVert_{\infty}}\hat{V}^+(x)+\lVert x-\langle \hat{P}^+,x\rangle \rVert_{\infty}\epsilon(s,a) & \text{ if } \epsilon(s,a)> f(x,\hat{P}^+).
\end{cases}
\end{align*}
Here \begin{equation}\hat{V}^+(x) = \sum_{s'}\hat{P}^+(s')(x_{s'}-\langle \hat{P}^+,x\rangle)^2\label{eqn:defineVplus}\end{equation} could be considered the variance of $x$ with respect to $\hat{P}^+(\cdot|s,a)$. To move from line 2 to line 3 we use result that $\lambda \log \mathbb{E}^+e^{X/\lambda}\le \mathbb{E}^+(X)+ \frac{1}{\lambda}\mathbb{E}^+(X^2)$ whenever $|X|\le \lambda$ with probability 1 (see \Cref{appn: proofBoundCumulant} for details), with $X=x - \langle \hat{P}^+,x\rangle$ considered as a finite random variable with probability distribution $\hat{P}^+$. Then $|X|\le \lambda$ with probability 1 translates to 
\[\lambda \ge \sup_{s'\in\Sc}\{|x_{s'}-\langle \hat{P}^+(s'|s,a),x\rangle| : \hat{P}^+(s'|s,a)>0\}\] which we write as $\lambda \ge \lVert x-\langle \hat{P}^+,x\rangle \rVert_{\infty}$.
The function $f(x,\hat{P}^+)$ is 
\[ f(x,\hat{P}^+) = \hat{V}^+(x)\frac{1}{\lVert x - \langle \hat{P}^+,x\rangle \rVert_{\infty}^2}.\] This comes from the location of the minimum of the function $a\lambda + b\frac{1}{\lambda}$ as in \Cref{appn:minAlambdaBoneOnLambda}.

\begin{rmk}
In the case where $\epsilon(s,a)\le f(x,\hat{P}^+)$ this results in the same bound as \citet[Tab.~1]{neuPikeBurke}.
\end{rmk}

\subsubsection{Using the cumulant function and Hoeffding's lemma}\label{subsubsec:cumulantHoeffding}
Alternatively, we can still use the cumulant function however we could proceed using Hoeffding's lemma as in \citet[Lem.~2.2]{Boucheron2012} instead of the result in \Cref{appn: proofBoundCumulant}. Then proceeding again from line 4 we  have 
\begin{align*}
\CB_{\min}(s,a)(x) & = - \min_{\lambda \ge 0} \max_{\tilde{P}\in \Delta} \{\langle -x,\tilde{P}-\hat{P}^+\rangle  - \lambda \left(D_{KL}(\tilde{P},\hat{P}^+)-\epsilon(s,a)\right) \} \\
& \ge - \min_{\lambda \ge 0}  \{\lambda\log\left(\sum_{s'}\hat{P}^+(s')e^{\frac{-1}{\lambda}(x_{s'}-\langle \hat{P}^+,x\rangle)}\right)+\lambda\epsilon(s,a) \} \\
& \ge - \min_{\lambda \ge 0}  \{\frac{1}{2\lambda}(\spn(x-\langle \hat{P}^+,x\rangle))^2 +\lambda\epsilon(s,a) \} \\
& =-\frac{2}{\sqrt{2}}\spn(x-\langle \hat{P}^+,x\rangle)\sqrt{\epsilon(s,a)}. 
\end{align*}

\begin{rmk}
This is very similar to bound for the reverse KL-divergence considered in \citet[Tab.~1]{neuPikeBurke}.

Note that we could use any of \Cref{subsubsec:Pinkser,subsubsec:cumulantAppendixD,subsubsec:Pinkser} as $\CB_{\min}^{\dagger}$, or alternatively the maximum of the three. We would also suggest bounding by $-\hat{P}^+(\cdot|s,a)x(\cdot)$ as we did in \Cref{eqn:hatDaggerBalloon}.

We give some alternative approaches to these bounds in \Cref{appn:alternativeKLdiv}.
\end{rmk}

\subsection{Reverse KL-divergence} \label{subsec:reverseKLdiv}

Here we consider the reverse KL-divergence $D_{KL}(P(\cdot|s,a),\tilde{P}(\cdot|s,a))$. This does not require any modification of $\hat{P}$. We have not seen the corresponding EVI operator that uses this divergence studied in the SSP literature. 
We then have
\begin{align*}
\CB_{\min}(s,a)(x) & = \min_{\tilde{P}\in \Delta} \{\langle x,\tilde{P}-\hat{P}\rangle \mid D_{KL}(\hat{P},\tilde{P}) \leq \epsilon(s,a) \} \\
&=\max_{\lambda \le 0}\min_{\tilde{P}\in \Delta} \{\langle x,\tilde{P}-\hat{P}\rangle - \lambda (D_{KL}(\hat{P},\tilde{P}) - \epsilon(s,a)) \} \\
&=-\min_{\lambda \le 0}\max_{\tilde{P}\in \Delta} \{\langle x,\hat{P}-\tilde{P}\rangle + \lambda (D_{KL}(\hat{P},\tilde{P}) - \epsilon(s,a)) \} \\
&=-\min_{\lambda \ge 0}\max_{\tilde{P}\in \Delta} \{\langle x,\hat{P}-\tilde{P}\rangle - \lambda (D_{KL}(\hat{P},\tilde{P}) - \epsilon(s,a)) \} \\
&\ge-\min_{\lambda \ge 0}\max_{\tilde{P}\in \Delta} \{\langle x,\hat{P}-\tilde{P}\rangle - \lambda (\frac{1}{2\log2}\lVert\hat{P}-\tilde{P}\rVert_1^2 - \epsilon(s,a))\}\\
&=-\min_{\lambda \ge 0}\max_{\tilde{P}\in \Delta} \{\langle \frac{1}{\sqrt{\lambda}}x,\sqrt{\lambda}(\hat{P}-\tilde{P})\rangle - \lambda (\frac{1}{2}\lVert\sqrt{\lambda}(\hat{P}-\tilde{P})\rVert_1^2 - \epsilon(s,a)\log2) \} \\
&\ge-\min_{\lambda \ge 0}\{\frac{1}{2\lambda}\lVert x\rVert_{\infty}^2 + \lambda\epsilon(s,a)\log2 \} \\
&=-2\lVert x\rVert_{\infty} \sqrt{\frac{\log2}{2}\epsilon(s,a)}.
\end{align*}
This uses Pinsker's inequality to move from line 4 to line 5 \cite[pg.~26]{Yeung2008}, and then the convex dual of the $\ell_1$ norm to move from line 6 to line 7 \cite[Ex.~3.27]{Boyd}. Note that Pinkser's inequality requires that $\tilde{P}$ and $\hat{P}$ are stochastic (not substochastic) so we are implicitly using the goal state in each line, and noting that component of $x$ corresponding to the goal state is constrained to have value zero.  The final equality uses the minimum of $a\lambda + \frac{1}{\lambda}b$ as in \Cref{appn:minAlambdaBoneOnLambda}. 

\begin{rmk} Note that this is the same bound as the KL-divergence bound in \Cref{subsubsec:Pinkser}. 
This gives a very similar bound to \citet[Tab.~1]{neuPikeBurke} as the maximum (supremum norm) is equal to the span if we include the goal state, as in \Cref{rmk:includegoalstatel1}.
\end{rmk}

\subsection{\texorpdfstring{$\chi^2$-divergence}{Chi-squared divergence}}\label{subsec:Chisquaredivbound}
Here we consider the Pearson's $\chi^2$ distance
\[ D_{\chi^2}(P,\hat{P}) = \sum_{s'} \frac{(P(s')-\hat{P}(s'))^2}{\hat{P}(s')} = \left\lVert \frac{P-\hat{P}}{\sqrt{\hat{P}}} \right\rVert_2^2.\] We have not seen the corresponding EVI operator that uses this divergence studied in the SSP literature. 

This is distance is again only well-defined where $\hat{P}(s')>0$, so we again modify $\hat{P}$ to $\hat{P}^+$ from \Cref{rem:modifiedP}. Then we have that 
\begin{align*}
    &D_{\chi^2}(\tilde{P},\hat{P}^+)  \\
    &=\sum_{s'\in\Sc} \frac{(\tilde{P}(s')-\hat{P}^+(s'))^2}{\hat{P}^+(s')}\\
    &= \sum_{s'\in\Sc} \frac{(\tilde{P}(s')-\hat{P}(s')+\hat{P}(s') - \hat{P}^+(s'))^2}{\hat{P}^+(s')}\\
    &\le \sum_{s'\in\Sc} \frac{(\tilde{P}(s')-\hat{P}(s'))^2}{\hat{P}^+(s')} + \sum_{s'\in\Sc} \frac{(\hat{P}(s')-\hat{P}^+(s'))^2}{\hat{P}^+(s')}\\
    &=\sum_{s'\in\Sc\setminus\Sc_0} \frac{(\tilde{P}(s')-\hat{P}(s'))^2}{\hat{P}^+(s')} + \sum_{s'\in\Sc_0} \frac{(\tilde{P}(s')-\hat{P}(s'))^2}{\hat{P}^+(s')} + \sum_{s'\in\Sc\setminus\Sc_0} \frac{(\hat{P}(s')-\hat{P}^+(s'))^2}{\hat{P}^+(s')} + \sum_{s'\in\Sc_0} \frac{(\hat{P}(s')-\hat{P}^+(s'))^2}{\hat{P}^+(s')}\\
    &=\sum_{s'\in\Sc\setminus\Sc_0} \frac{(\tilde{P}(s')-\hat{P}(s'))^2}{\hat{P}(s')\frac{n}{n+z}} + \sum_{s'\in\Sc_0} \frac{(\tilde{P}(s'))^2}{\frac{1}{n+z}} + \sum_{s'\in\Sc\setminus\Sc_0} \frac{(\hat{P}(s')\frac{z}{n+z})^2}{\hat{P}(s')\frac{n}{n+z}} + \sum_{s'\in\Sc_0} \frac{(\frac{1}{n+z})^2}{\frac{1}{n+z}}\\
    &\le \frac{n+z}{n}\epsilon + (n+z)\sum_{s'\in\Sc_0} (\tilde{P}(s'))^2 + \frac{z^2}{n(n+z)} + \frac{z}{n+z}
\end{align*}
To ensure this is $\mathcal{O}(\epsilon)$ as $n$ increases, then we need to require that $\sum_{s'\in\Sc_0} (\tilde{P}(s'))^2$ is $\mathcal{O}(\frac{1}{n^2})$. To do this, we add the constraint that $\sum_{s'\in\Sc_0} (\tilde{P}(s'))^2\le \frac{1}{n^2}$ and we have that 
\[ D_{\chi^2}(\tilde{P},\hat{P}^+) \le 
(1+\frac{z}{n})\epsilon + \frac{n+z}{n^2} + \frac{z^2}{n(n+z)} + \frac{z}{n+z} =:\epsilon^+= \mathcal{O}(\epsilon) .\] Note that as the algorithm progresses and $n$ increases, $z$ either stays the same or shrinks (as it is equivalent to the number of states $s'$ that have never been reached from $s\in\Sc,a\in\Ac$ so far in the algorithm). 

We have that 
\begin{align*}
&\CB_{\min}(s,a)(x) \\
& = \min_{\tilde{P}\in \Delta} \{\langle x,\tilde{P}-\hat{P}^+\rangle \mid D_{\chi^2}(\hat{P}^+,\tilde{P}) \leq \epsilon^+(s,a) \} \\
& =\max_{\lambda \le 0} \min_{\tilde{P}\ge 0} \{\langle x-\lambda 1,\tilde{P}-\hat{P}^+\rangle \mid \left\lVert \frac{\tilde{P}-\hat{P}^+}{\sqrt{\hat{P}^+}} \right\rVert_2^2 \leq \epsilon^+(s,a) \} \\
&= \max_{\lambda \le 0} \min_{\tilde{P}\ge 0} \{\langle x-\lambda 1,\tilde{P}-\hat{P}^+\rangle + \lambda \hat{P}^+(g|s,a)\mid \left\lVert \frac{\tilde{P}-\hat{P}^+}{\sqrt{\hat{P}^+}} \right\rVert_2^2 \leq \epsilon^+(s,a) \} \\
& =-\min_{\lambda \le 0} \max_{\tilde{P}\ge 0 } \{-\left\langle \sqrt{\hat{P}^+}(x-\lambda 1),\frac{\tilde{P}-\hat{P}^+}{\sqrt{\hat{P}^+}}\right\rangle - \lambda \hat{P}^+(g|s,a) \mid \left\lVert \frac{\tilde{P}-\hat{P}^+}{\sqrt{\hat{P}^+}} \right\rVert_2 \leq \sqrt{\epsilon^+(s,a)}\} \\
& \ge -\min_{\lambda \le 0} \max_{\tilde{P}\ge 0 } \{\left\lVert \sqrt{\hat{P}^+}(x-\lambda 1)\right\rVert_2\left\lVert\frac{\tilde{P}-\hat{P}^+}{\sqrt{\hat{P}^+}}\right\rVert_2 - \lambda \hat{P}^+(g|s,a) \mid \left\lVert \frac{\tilde{P}-\hat{P}^+}{\sqrt{\hat{P}^+}} \right\rVert_2 \leq \sqrt{\epsilon^+(s,a)}\}\\
& \ge -\min_{\lambda \le 0}\{\left\lVert \sqrt{\hat{P}^+}( x-\lambda 1)\right\rVert_2\sqrt{\epsilon^+(s,a)} - \lambda \hat{P}^+(g|s,a)\}.
\end{align*}
Here at each line until the last line we are also requiring that $\sum_{s'\in\Sc_0} (\tilde{P}(s'))^2\le \frac{1}{n^2}$.We also used the Cauchy-Swartz inequality to move from line 5 to line 6 (see for example \citet{Hardy1988}.) 

At this point, we can take the derivative with respect to $\lambda$ and set equal to zero to find the location of the minimum. Note that for the second term, this would occur at the boundary where $\lambda =0$ where it has the value of $0$. For the first term, this is convex in $\lambda$ so the global minimum occurs at $\lambda = \langle x,\hat{P}^+\rangle$ which is positive when $x\ge 0$, which would give 
\begin{equation}- \sqrt{\epsilon^+(s,a) \hat{V}^+(x)},\label{eqn:boundIfSumPis1}\end{equation} with $\hat{V}^+$ as in \Cref{eqn:defineVplus}. However, $\lambda$ must be less than zero, so the minimum also occurs at $\lambda = 0$ and we have 
\begin{equation}\CB_{\min}(s,a)(x) \ge - \sqrt{\epsilon^+(s,a) \sum_{s'}\hat{P}^+(s'\mid s,a)x_{s'}^2}.\label{eqn:boundChiSquared}\end{equation} Note that if we required $ \sum_{s'\in\Sc}{P}(s'|s,a) = 1$ then \Cref{eqn:boundIfSumPis1} would be the appropriate bound instead, which is smaller than \Cref{eqn:boundChiSquared}. We are effectively paying a penalty to relax $ \sum_{s'\in\Sc}{P}(s'|s,a) = 1$ to the less restrictive $ \sum_{s'\in\Sc}{P}(s'|s,a) \le 1$.

\begin{rmk} \label{rem:includingGoalStateBoundChiSq}
We could instead consider including the goal state in all of the analysis in this section. Then we indeed have $ \sum_{s'\in\Sc\cup\{g\}}{P}(s'|s,a) = 1$ and we would need to use the definition of $\hat{P}^+$ that includes the goal state as in \Cref{rem:modifiedP}. We would also require that the element of $x$ corresponding to $g$ would be set to zero (as there the goal state always has zero cost-to-go). We could then use the bound in \Cref{eqn:boundIfSumPis1} where $\hat{V}^+$ is defined by summing over $g$ as well. This would bound would be between the bounds in \Cref{eqn:boundIfSumPis1} and \Cref{eqn:boundChiSquared} that do not include the goal state. 

Note that \Cref{eqn:boundIfSumPis1} is the first term in the bound for this divergence for \citet[A.5.4]{neuPikeBurke}, however the second term in their bound does not appear due to our definition of $\hat{P}^+$.

We discuss an alternative approach to this bound in \Cref{appn:alternativeChisquared}.
\end{rmk}

\subsection{\texorpdfstring{Variance-weighted $\ell_\infty$-norm}{Variance-weighted l-infinity-norm}}\label{subsec:varweightLinfinitynorm}
Here we consider the norm
\[ D_{\infty}(P,\hat{P}) = \max_{s'} \frac{(P(s')-\hat{P}(s'))^2}{\hat{P}(s')} = \left\lVert \frac{P-\hat{P}}{\sqrt{\hat{P}}} \right\rVert_\infty^2.\] We have not seen the corresponding EVI operator that uses this divergence studied in the SSP literature. 

This is again only well-defined where $\hat{P}(s')>0$, so we again modify $\hat{P}$ to $\hat{P}^+$ from \Cref{rem:modifiedP}. Then we have that 
\begin{align*}
    &D_{\chi^2}(\tilde{P},\hat{P}^+)  \\
    &=\max_{s'\in\Sc} \frac{(\tilde{P}(s')-\hat{P}^+(s'))^2}{\hat{P}^+(s')}\\
    &= \max_{s'\in\Sc} \frac{(\tilde{P}(s')-\hat{P}(s')+\hat{P}(s') - \hat{P}^+(s'))^2}{\hat{P}^+(s')}\\
    &\le \max_{s'\in\Sc} \frac{(\tilde{P}(s')-\hat{P}(s'))^2}{\hat{P}^+(s')} + \max_{s'\in\Sc} \frac{(\hat{P}(s')-\hat{P}^+(s'))^2}{\hat{P}^+(s')}
    \end{align*}
The first term is 
\begin{align*}
&\max_{s'\in\Sc} \frac{(\tilde{P}(s')-\hat{P}(s'))^2}{\hat{P}^+(s')}\\
&= \max\left\{ \max_{s'\in\Sc\setminus\Sc_0}\frac{(\tilde{P}(s')-\hat{P}(s'))^2}{\hat{P}^+(s')},  \max_{s'\in\Sc_0}\frac{(\tilde{P}(s')-\hat{P}(s'))^2}{\hat{P}^+(s')}\right\}\\
&=\max\left\{ \max_{s'\in\Sc\setminus\Sc_0}\frac{(\tilde{P}(s')-\hat{P}(s'))^2}{\hat{P}(s')\frac{n}{n+z}},  \max_{s'\in\Sc_0}\frac{(\tilde{P}(s')-\hat{P}(s'))^2}{\frac{1}{n+z}}\right\}\\
&\le\max\left\{ \frac{n+z}{n}\epsilon,  (n+z)\max_{s'\in\Sc_0}(\tilde{P}(s'))^2\right\}
\end{align*}
The second term is 
 \begin{align*}
& \max_{s'\in\Sc} \frac{(\hat{P}(s')-\hat{P}^+(s'))^2}{\hat{P}^+(s')}\\
 & = \max\left\{\max_{s'\in\Sc\setminus\Sc_0} \frac{(\hat{P}(s')-\hat{P}^+(s'))^2}{\hat{P}^+(s')} + \max_{s'\in\Sc_0} \frac{(\hat{P}(s')-\hat{P}^+(s'))^2}{\hat{P}^+(s')}\right\}\\
 & = \max\left\{\max_{s'\in\Sc\setminus\Sc_0} \frac{(\hat{P}(s')\frac{z}{n+z})^2}{\hat{P}(s')\frac{n}{n+z}} + \max_{s'\in\Sc_0} \frac{(\frac{1}{n+z})^2}{\frac{1}{n+z}}\right\}\\
 & = \max\left\{\max_{s'\in\Sc\setminus\Sc_0} \frac{(\hat{P}(s')\frac{z}{n+z})^2}{\hat{P}(s')\frac{n}{n+z}} + \max_{s'\in\Sc_0} \frac{(\frac{1}{n+z})^2}{\frac{1}{n+z}}\right\}\\
  & = \max\left\{\frac{z^2}{n(n+z)}, \frac{z}{n+z}\right\} = \mathcal{O}(\frac{1}{n})
\end{align*}
Then we need to again bound $\max_{s'\in\Sc_0}(\tilde{P}(s'))^2$ to ensure the bound is $\mathcal{O}(\epsilon)$. So we can impose the constraint that $\max_{s'\in\Sc_0}(\tilde{P}(s'))^2\le \frac{1}{n^2}.$

We now follow the same reasoning as the $\chi^2$ case, however we apply H\"{o}lder's inequality for $p=1$, $q=\infty$. We have that
\begin{align*}
&\CB_{\min}(s,a)(x) \\
& = \min_{\tilde{P}\in \Delta} \{\langle x,\tilde{P}-\hat{P}\rangle \mid D_{\infty}(\hat{P},\tilde{P}) \leq \epsilon(s,a) \} \\
& =-\min_{\lambda \le 0} \max_{\tilde{P}\ge 0 } \{-\left\langle \sqrt{\hat{P}^+}(x-\lambda 1),\frac{\tilde{P}-\hat{P}^+}{\sqrt{\hat{P}^+}}\right\rangle - \lambda \hat{P}^+(g|s,a)\rangle \mid \left\lVert \frac{\tilde{P}-\hat{P}^+}{\sqrt{\hat{P}^+}} \right\rVert_\infty \leq \sqrt{\epsilon(s,a)} \} \\
& \ge -\min_{\lambda \le 0} \max_{\tilde{P}\ge 0 } \{\left\lVert \sqrt{\hat{P}^+}(x-\lambda 1)\right\rVert_1\left\lVert\frac{\tilde{P}-\hat{P}^+}{\sqrt{\hat{P}^+}}\right\rVert_\infty - \lambda \hat{P}^+(g|s,a)\mid \left\lVert \frac{\tilde{P}-\hat{P}^+}{\sqrt{\hat{P}^+}} \right\rVert_\infty \leq \sqrt{\epsilon(s,a)}\}\\
& \ge -\min_{\lambda \le 0}\{\left\lVert \sqrt{\hat{P}^+}( x-\lambda 1)\right\rVert_1\sqrt{\epsilon(s,a)} -\lambda \hat{P}^+(g|s,a)\rangle\}.
\end{align*}
As in \Cref{appn:weightedspan}, the first term to be minimised over is convex and has minimum at $\lambda\ge 0$. As we require $\lambda\le 0$ then the minimum occurs at $\lambda =0$. The second term is linear and also has minimum at $\lambda =0$. Therefore we have 
\begin{align}
&\CB_{\min}(s,a)(x) \ge -\left\lVert \sqrt{\hat{P}^+}( x)\right\rVert_1\sqrt{\epsilon(s,a)}. \label{eqn:boundWeightedLinfinty}\end{align}

\begin{rmk}
As in \Cref{rem:includingGoalStateBoundChiSq}, we could include the goal state in the above analysis. In this case we would also need to modify $\hat{P}^+$ to include the goal state as in \Cref{rem:modifiedP}, and then use the exact location of the minimum from \Cref{appn:weightedspan}. This would give a slightly tighter bound than \Cref{eqn:boundWeightedLinfinty}.

As we do not include the goal state, and as we have a different definition of $\hat{P}^+$, we have a different bound to \citet[Tab.~1]{neuPikeBurke}.

We discuss an alternative approach to this bound in \Cref{appn:AlternativeVarWeightsupNorm}.
\end{rmk}

\subsection{General norms and divergences}
Note that all norms in $\mathbb{R}^n$ are equivalent, so there are relations of the form 
\[\alpha_1\lVert \cdot\rVert_a \le \lVert \cdot \rVert_b \le \alpha_2 \lVert \cdot \rVert_a\] for any two norms in $\R^n$ $\lVert \cdot\rVert_a$ and $\lVert \cdot \rVert_a$, where $\alpha_1,\alpha_2$ are positive real numbers that depend only on the norms chosen. Therefore there is a relation 
\[ \lVert \cdot \rVert_a \ge \alpha \lVert \cdot \rVert_1\] for all norms $\lVert \cdot \rVert_a$ in $\R^n$. This allows us to use the bound of one norm to bound another. 

There are also many inequalities relating norms and divergences, such as detailed in \citet{SansonVerdu2016} that will lead to similar bounds.

\subsection{Discussion of results}\label{subsec:discussion}
This section discusses calculations of $\CB_{\min}$ as defined in \Cref{sec:unknownMDPs} and how to bound these as inspired by \citet{neuPikeBurke}. In the $\ell_1$ and sup-norm case, \citet{rosenbergCohen} studied their corresponding EVI algorithms, and we discuss bounds in these cases in \Cref{subsec:l1norm} and \Cref{subsec:supnorm}. The bound in the sup-norm case, while interesting to observe how it differs from the true $\CB_{\min}$ does not reduce computational complexity, so we do not study this further. 

In the $\ell_1$ case, we consider the bound $\CB_{\min}^{\dagger,0}$ and conjecture that using this in $\hat{U}^{\dagger,0}$ would lead to a less computationally intense but good approximation to $\hat{U}$. We show how the corresponding optimisation program \Cref{CP primalUnknowndaggerl1} relates to \Cref{CP primalUnknown}, and discuss the existence of fixed points of $\hat{U}^{\dagger, 0}$ in \Cref{conj:l1dagger}. Finally, we conjecture that iterating $\hat{U}^{\dagger,0}$ gives a fixed point to $\hat{U}^{\dagger,0}$ and the solution to \Cref{CP primalUnknowndaggerl1} in all but a special case, as in \Cref{conj:l1dagger}. 

In \Cref{appn:exploringTheConjecture} we go on to discuss in detail the behaviour of $\hat{U}^{\dagger,0}$ in the $ell_1$ case, showing that it is not monotonic, however that it the conjecture indeed appears to hold in all but a special case. This special case is detailed further for the case where $N=2$ in \Cref{appn:conjecture2statecase}, where the operator shows oscillating behaviour, while its fixed point is equal to the solution to \Cref{CP primalUnknowndaggerl1}. We also give a potential alternative method of finding the fixed points in \Cref{conj:FPprocedure}, that has helped us understand the operator $\hat{U}^{\dagger,0}$. Other than these results, \Cref{appn:exploringTheConjecture} remains an open question.

In \Cref{subsec:KLdiv,subsec:reverseKLdiv,subsec:Chisquaredivbound,subsec:varweightLinfinitynorm} we describe the EVI operators and derived specific bounds for their optimisation programs and relate these to similar bounds derived in \citet{neuPikeBurke}. In each case, we can define an optimisation program related to \Cref{CP primalUnknowndaggerl1}, we then expect a similar conjecture to \Cref{appn:conjecture2statecase} to hold, and that similar results to \Cref{appn:exploringTheConjecture} should hold. Note that these EVI operators have also not yet been studied for SSPs, which is an avenue of future research.

At this stage, given the open questions generated by this research, we have not gone on to consider the regret of the algorithms that may arise if the bounds in the previous sections were used for the operator $\hat{U}^{\dagger, 0}$.

\section{Conclusion}\label{sec:conclusion}
In this paper we have first recalled and introduced the known SSP setting in \Cref{sec:SSPsAsMDPsKnown}. This includes summarising policy and value iteration following \citet{BertsekasTsitsiklis91} and \citet{Kallenberg}, with several of our own proofs for the known results on convergence of operators. We particularly studied the relationship to linear programming, where we note the simplex algorithm applied to \Cref{LP dualKnown} identifies with policy iteration, while value iteration results from iterating the operator deduced from the dual program, \Cref{LP primalKnown}. 

We then consider the case of unknown parameters in \Cref{sec:unknownMDPs}, focusing on the unknown transition function case. In particular we recall extended value iteration for SSPs from \citet{jaksch10a} and \citet{rosenbergCohen}, and then we show how this can be considered as a convex program, \Cref{CP primalUnknown}, which is new research. We show that the unique solution to this convex program give fixed points of the operator $\hat{U}$. We study this convex program and its dual, \Cref{OP dualUnknown}; in particular we give the form of the dual program and deduce that the duality gap is zero, i.e. we have strong duality.

We do not study the dual program further, or its convex form \Cref{CP dualKnown2}, however we note that these programs may be useful for considering policy iteration methods for SSPs with unknown transitions, along the lines of the discounted infinite horizon version of \citet{Kaufman2011} and the finite horizon version in \citet{Auer2006}. This may be an avenue of future research. 

In \Cref{sec:bounds}, we aim to determine whether similar bounds to those achieved by \citet{neuPikeBurke} for finite horizon MDPs will be applicable for SSPs. These bounds are applied to the EVI operators from the primal program,  \Cref{subsec:UtoCP}. We note that unlike the finite setting, we are only aware of the EVI operators from the $\ell_1$ norm  and the supremum norm, as defined in \Cref{subsec:l1norm,subsec:supnorm}, being studied in the SSP literature as in \citet{rosenbergCohen} and \citet{Tarbouriech2019}. Studying the EVI operators in \Cref{subsec:KLdiv,subsec:reverseKLdiv,subsec:Chisquaredivbound,subsec:varweightLinfinitynorm} is future work. 

We indeed find appropriate bounds depending on the various norms and divergences used, with subtle differences to the bounds found in \citet{neuPikeBurke}, which lead us to define new value-iteration based operators. However, while finite horizon MDPs value iteration methods usually rely on backwards induction and other finite methods, infinite horizon methods usually rely on iterating an operator that converges. We are left with an open question on the convergence of the operator deduced from bounds on the $\ell_1$ norm in \Cref{conj:l1dagger}, and its relationship to the convex programs of \Cref{sec:unknownMDPs}. 

Exploring the convergence of these operators, we consider results from simulations and the specific case where there are only two states in \Cref{appn:exploringTheConjecture}; these give affirmative evidence for this conjecture, showing numerical evidence that all but a special case converge to the fixed point and this fixed point is the optimal solution of the convex program \Cref{CP primalUnknowndaggerl1}. However, we note that any formal proof of convergence is expected to be highly non-trivial as the operator is not monotone, and this requires further study. We expect similar results for the operators deduced from the other norms and divergences. 

We overall conclude that the idea to extend \citet{neuPikeBurke} to the SSP setting holds appeal for the understanding SSP EVI operators from the perspective of convex optimisation. This approach leads to avenues of future research in studying the properties of the SSP EVI operators associated to \Cref{subsec:KLdiv,subsec:reverseKLdiv,subsec:Chisquaredivbound,subsec:varweightLinfinitynorm}, as well as generating the optimisation programs that policy iteration methods for SSP would be based on. 

However, while the operators we developed from bounding EVI operators involve less computations, these operators may not be monotone and have more complicated convergence behaviour. Given the difficulties encountered so far with proving convergence, this suggests a more nuanced approach for continuing along this path. One example would be to adapt the bounds to ensure that the operators are monotone and have the corresponding convergence properties. We note that similar research along these lines has been recently studied in \citet{Tarbouriech2021}, where an EVI operator is adapted to a monotone piecewise operator that achieves minimax regret. This approach may benefit from being studied with the convex optimisation lens we have applied here. We leave this for future work.

\section{Acknowledgements}
We acknowledge Dr George Stamastecu for his enthusiasm, for inspiring the idea to combine \citet{neuPikeBurke} and \citet{rosenbergCohen}, and for many helpful discussions. We thank Professor Langford White and Professor Hung Nguyen for many helpful conversations. We also thank Dr Gergly Neu for email correspondence on the proof in \Cref{appn: proofBoundCumulant}. 

\appendix

\newpage
\addcontentsline{toc}{section}{References}
\bibliographystyle{abbrvnat}
\bibliography{references.bib}

\newpage

\label{appendix}

\section{\texorpdfstring{Further discussion on \Cref{conj:l1dagger}}{Further discussion on the conjecture}} \label{appn:exploringTheConjecture}

Here we discuss \Cref{conj:l1dagger} and the operator \[(\hat{U}^{\dagger,0}x)_s = \min_{a\in\Ac} c(s,a) +\max\left\{\sum_{s'\in\Sc}\hat{P}(s'|s,a)x_{s'} -\epsilon(s,a)\max(x),0\right\}.\]
We also consider fixing a policy $\pi$ and defining the operator 
\[(\hat{L}_{\pi}^{\dagger,0}x)_s = c(s,\pi(s)) +\max\left\{\sum_{s'\in\Sc}\hat{P}(s'|s,\pi(s))x_{s'} -\epsilon(s,\pi(s))\max(x),0\right\}.\]
Note that if there is only one action for each task, then  $\hat{L}_{\pi}^{\dagger,0}$ is equal to  $\hat{U}^{\dagger,0}$, and in general $\hat{U}^{\dagger,0} = \min_{\pi}\hat{L}_{\pi}^{\dagger,0}$. For this reason, we find studying  $\hat{L}_{\pi}^{\dagger,0}$ gives insight into the behaviour of  $\hat{U}^{\dagger,0}$. 

\begin{rmk}\label{rem:epsilonlessthanone}
For $x\ge 0$ note that
\begin{align*}
    \sum_{s'}\hat{P}(s'|s,a)x_{s'} - \epsilon(s,a)\max_{s''}x_{s''}&\le \sum_{s'}\hat{P}(s'|s,a)\max_{s''}x_{s''} - \epsilon(s,a)\max_{s''}x_{s''} \\ &\le(1-\epsilon(s,a))\max_{s''}x_{s''} \end{align*}
This is less than or equal to $0$ whenever $\epsilon(s,a)\ge1$. Hence we can assume the operator $\hat{L}_{\pi}^{\dagger,0}$ is equal to the costs in this case. We then find that all interesting behaviour occurs when $\epsilon(s,a)<1$, which we often assume to hold.
\end{rmk}

\subsection{How we would have liked the proof to go}
We would like to be able to show that $\hat{U}^{\dagger,0}$ is a contraction using a similar argument to \Cref{lem:properContractionU}. This proof would proceed as follows.

We assume that $\epsilon(s,a)<1$ as per \Cref{rem:epsilonlessthanone}.
For any $x \in \R^{N}$ and $s\in\Sc, s''\in\Sc,a\in\Ac$ then let 
\[ b_1(x,s,a,s'') = \sum_{s'}\hat{P}(s'|s,a)x_{s'} - \epsilon(s,a)x_{s''}, \quad b_2(x,s,a,s'') = 0.\]
Then \[(\hat{U}^{\dagger,0}x)_s = \min_a\max_{i}\min_{s''} (c(s,a) + b_{i}(x,s,a,s''))\]
For any $x, y \in \R^{N}$ and $s\in\Sc$ then
\begin{align}\nonumber
    &\lvert(\hat{U}^{\dagger,0} x)_s - (\hat{U}^{\dagger,0} y)_s \rvert \\\nonumber
    & =\left \lvert \min_a\max_{i}\min_{s''} (c(s,a) + b_{i}(x,s,a,s''))  - \min_a\max_{i}\min_{s''} (c(s,a) + b_{i}(y,s,a,s'')) \right\rvert\\\nonumber
    & \le \max_a\left\lvert \max_{i}\min_{s''} (c(s,a) + b_{i}(x,s,a,s'')) - \max_{i}\min_{s''} (c(s,a) + b_{i}(y,s,a,s'')) \right \rvert\\\nonumber
    & \le \max_a\max_{i}\left\lvert \min_{s''} (c(s,a) + b_{i}(x,s,a,s'')) - \min_{s''} (c(s,a) + b_{i}(y,s,a,s''))\right) \rvert\\\nonumber
    & \le \max_a\max_{i}\max_{s''}\left\lvert c(s,a) + b_{i}(x,s,a,s'') - c(s,a) - b_{i}(y,s,a,s'')\right\rvert\\\nonumber
    & = \max_a\max_{i}\max_{s''}\left\lvert b_{i}(x,s,a,s'')  - b_{i}(y,s,a,s'') \right\rvert\\\nonumber
    & = \max_a\max\left\{\max_{s''}\left\lvert \sum_{s'}\hat{P}(s'|s,a)(x_{s'}-y_{s'}) - \epsilon(s,a)(x_{s''}-y_{s''})\right\lvert,0\right\}\\ \nonumber
    & = \max_a\max_{s''}\left\lvert \sum_{s'}\hat{P}(s'|s,a)(x_{s'}-y_{s'}) - \epsilon(s,a)(x_{s''}-y_{s''})\right\lvert\\\nonumber
    &\le \max_a\max_{s''}\left\lvert\left(\sum_{s'}\hat{P}(s'|s,a)-\epsilon(s,a)\delta_{s',s''}\right)x_{s''}\right\rvert
\end{align} Note that we have repeatedly used \Cref{lem:maxarrange}. 

We would then like to be able to bound the operators $\hat{P}(\cdot|s,a)-\epsilon(s,a)\delta_{\cdot,s''}$ by a suitable norm, as in \Cref{lem:properContractionU}. This norm needs to work for all choices of $s''$. 

Unfortunately there does not appear to be a way to further adapt the proof \Cref{lem:properContractionU}, as it uses the following inequality
\begin{align*}
    \left\lvert\left(\sum_{s'}\hat{P}(s'|s,a)-\epsilon(s,a)\delta_{s',s''}\right)x_{s''}\right\rvert
    & = \sum_{s'}\left\lvert\hat{P}(s'|s,a)-\epsilon(s,a)\delta_{s',s''}\right\rvert\max_{s''}x_{s''}.
\end{align*}
This tends to result in a sum that gives $\sum_{s'}\hat{P}(s'|s,a)+\epsilon$ which may be greater than 1, and is not a sharp bound in general. This suggests using other proof techniques.

In the next sections we consider how the operators $\hat{P}(\cdot|s,a)-\epsilon(s,a)\delta_{\cdot,s''}$ behave and whether we can gain insight from this.

\subsection{The 1-state case}
Let us first consider the $1\times 1$ case, where there is only 1 state $s$ so the vector $x$ degenerates to a single element $x$. 
\begin{lemma}
\Cref{conj:l1dagger} holds for the 1 state case.
\end{lemma}
\begin{proof} For a fixed policy $\pi$ we have 
\[(\hat{L}_{\pi}^{\dagger,0}x)_s = c(s,\pi(s)) +\max\left\{\left(\hat{P}(s|s,\pi(s)) -\epsilon(s,\pi(s))\right)x,0\right\}.\]
If $\hat{P}(s|s,\pi(s)) -\epsilon(s,\pi(s))<0$ then $\hat{L}_{\pi}^{\dagger,0}$ is a contraction with contraction factor $0$ and fixed point $x^* =  c(s,\pi(s))$. If $\hat{P}(s|s,\pi(s)) -\epsilon(s,\pi(s))>0$ then $\hat{L}_{\pi}^{\dagger,0}$ is a contraction with contraction factor $\hat{P}(s|s,\pi(s)) -\epsilon(s,\pi(s))$ and we have the fixed point
\[ x^* = \frac{c(s,\pi(s))}{1 - \hat{P}(s|s,\pi(s)) +\epsilon(s,\pi(s))}>c(s,\pi(s)).\]
The superharmonic vector with respect to $\hat{L}_{\pi}^{\dagger,0}$ are all elements less than $x^*$, and therefore $x^*$ is the solution to the optimisation program in \Cref{CP primalUnknowndaggerl1} when a policy is fixed.

When the policy is not fixed and we return to the operator $\hat{U}^{\dagger,0}$ then we have 
\[(\hat{U}^{\dagger,0}x)_s = \min_ac(s,a) +x\max\left\{\hat{P}(s|s,a) -\epsilon(s,a),0\right\}.\] We can consider this as the known case operator $U$ from \Cref{eqn:defnU} where we let \[P(s|s,a) = \max\left\{\hat{P}(s|s,a) -\epsilon(s,a),0\right\}.\] Therefore this is a contraction where the fixed point can be found by iterating the operator $\hat{U}^{\dagger,0}$ from any starting point, and the fixed point is equal to optimal solution to the program in \Cref{CP primalUnknowndaggerl1}.

We have shown that \Cref{conj:l1dagger} holds for the 1 state case.
\end{proof}

\subsection{Properties in the general case}

\begin{lemma}
The operator $\hat{U}^{\dagger,0}$ with \[(\hat{U}^{\dagger,0}x)_s = \min_{a\in\Ac} c(s,a) +\max\left\{\sum_{s'\in\Sc}\hat{P}(s'|s,a)x_{s'} -\epsilon(s,a)\max(x),0\right\}\] is 
\begin{itemize} 
\item[(1)] non-montonic, in that $x\le y$ does not necessarily mean $\hat{U}^{\dagger,0}x\le \hat{U}^{\dagger,0}y$.
\item[(2)] in general not a contraction in the $\lVert \cdot\rVert_p$ norms for $p=1,2,\ldots, \infty.$
\end{itemize}
\end{lemma}
\begin{proof}
For (1) note that in the $|\Sc|=1$ case, the operator is monotonic. However, unlike the 1 state case, all other cases cannot be reduced to the known case operator $U$ from \Cref{eqn:defnU}. This is due to the behaviour of the $\max x$ in the operator. For example, consider the following counterexample. Let $|\mathcal{S}|=2$, $|\mathcal{A}| = 1$ and let
\[ P = \begin{bmatrix} 0.45 & 0.45\\ 0.45 & 0.45 \end{bmatrix}, \quad \epsilon = \begin{bmatrix} 0.5 \\ 0.5 \end{bmatrix}, \quad x = \begin{bmatrix} 1\\ 0.9 \end{bmatrix}, \quad y = \begin{bmatrix} 1\\ 2 \end{bmatrix}, c = \begin{bmatrix} 0.5 \\ 0.5 \end{bmatrix}.\] Then we have $x\le y$ however 
\[ \hat{U}^{\dagger,0}x = c + \max\{Px-\epsilon \max{x},0\} = \begin{bmatrix} 0.855\\ 0.855 \end{bmatrix} \ge \begin{bmatrix} 0.85\\ 0.85 \end{bmatrix} = c + \max\{Py-\epsilon \max{y},0\} = \hat{U}^{\dagger,0}y.\]
This example used a large enough $\epsilon$ combined with the maximum of $x$ and $y$ occurring at different locations to produce the counterexample, which is not possible in the 1 state case.

For (2) refer to \Cref{prop:whencontraction} and \Cref{ex:oscillatingBehaviour} where we show that in the $|\Sc|=2$ case non-contractive behaviour can occur.
\end{proof}

We may wish to write the operator $\hat{L}_{\pi}^{\dagger,0}$ as follows:
\[ \hat{L}_{\pi}^{\dagger,0}x = c(\cdot,\pi(\cdot))+\max\{(\hat{P}_{\pi}-\epsilon_{\pi}\delta_{s''})x,0\}.\] Here $s''=\argmax_{s\in\Sc}(x)$, and $\delta_{s''}$ is equal to $1$ on the element $x_{s''}$ and zero elsewhere. As a matrix, we have 
\[ \hat{P}_{\pi}-\epsilon_{\pi}\delta_{s''} = \begin{bmatrix}p_{11} & p_{12} &\cdots & p_{1N}\\ p_{21} & p_{22}& \cdots & p_{2N}\\ \vdots&\vdots&\ddots&\vdots\\p_{N1} &p_{N2}& \cdots & p_{NN}\end{bmatrix} - \begin{bmatrix}  0 & \cdots & 0 & \epsilon(1,\pi(1)) & 0 & \cdots & 0  \\ 0 & \cdots & 0 & \epsilon(2,\pi(2)) & 0 & \cdots & 0  \\ \vdots & \ddots & \vdots & \vdots & \vdots & \ddots & \vdots  \\ 0 & \cdots & 0 & \epsilon(N,\pi(N)) & 0 & \cdots & 0  \end{bmatrix}.\] Here the $s''$ column of $\epsilon_{\pi}\delta_{s''}$ is the non-zero column. So this has the effect of editing $\hat{P}_{\pi}$ by a non-zero column. 

In the case where $ \epsilon(s,\pi(s)) \ge p_{s,s''}$ for all $s$, then $ \hat{P}_{\pi}-\epsilon_{\pi}\delta_{s''}\le \hat{P}_{\pi}$ is a non-negative sub-stochastic matrix. If we consider this as the matrix $P$ in \Cref{lem:PI}, then we can see it must be a contraction as it will have all eigenvalues with norm less than 1 whenever $\hat{P}_{\pi}$ has this property, i.e. when \Cref{assu:proper} holds for $\hat{P}_{\pi}$.

\subsection{The 2-state case} \label{appn:conjecture2statecase}

In the $2\times 2$ case we can determine the contractive properties of each piece. That is, if we fix a policy $\pi$ then the possible pieces correspond to the 7 matrices
\begin{align*} P_1= \begin{bmatrix} p_{11} -\epsilon(1,\pi(1)) & p_{12}\\ p_{21}- \epsilon(2,\pi(2)) & p_{22}\end{bmatrix}, \quad&P_2 = \begin{bmatrix} p_{11}  & p_{12}-\epsilon(1,\pi(1))\\ p_{21} & p_{22}- \epsilon(2,\pi(2))\end{bmatrix}\\
P_{11}= \begin{bmatrix} p_{11} -\epsilon(1,\pi(1)) & p_{12}\\ 0 & 0\end{bmatrix}, \quad&P_{21} = \begin{bmatrix} p_{11}  & p_{12}-\epsilon(1,\pi(1))\\0 & 0\end{bmatrix}\\
P_{12}= \begin{bmatrix} 0 & 0\\ p_{21}- \epsilon(2,\pi(2)) & p_{22}\end{bmatrix}, \quad&P_{22} = \begin{bmatrix} 0  & 0\\ p_{21} & p_{22}- \epsilon(2,\pi(2))\end{bmatrix}\\
P_0 = \begin{bmatrix} 0 & 0\\ 0 & 0 \end{bmatrix}.\quad&
\end{align*}
If $ \hat{L}_{\pi}^{\dagger,0}x = c+P_{k}x$ for a given $k=1,2,11,12,21,22,0$, then we say $P_k$ is the \emph{active} piece.

In this section, we consider whether each active piece is a contraction by examining the eigenvalues. If the active piece is a contraction, then that each active piece has a fixed point. We expect exactly one of these active pieces' fixed points to be the fixed point for $\hat{L}_{\pi}^{\dagger,0}$. We will then show that only one of the fixed points of $P_1$ and $P_2$ can be in the active piece in $\hat{L}_{\pi}^{\dagger,0}$, and the same for the pairs $(P_{11},P_{21})$ and $(P_{21},P_{22})$. 

\begin{rmk}\label{rmk:epsilongreaterthan1}
Note that if $\epsilon(s,\pi(s))>1$ for all $s$ then the only active piece is $P_0$ with fixed point $c(\cdots,\pi(\cdot))$, which is the fixed point of $\hat{L}_{\pi}^{\dagger,0}$. In the case where $\epsilon(s,\pi(s))>1$ only for $s=s'$, say $s'=1$ without loss of generality, then the only possible active pieces are $P_0,P_{21},P_{22}$ and the fixed points must have $x_1=c(1,\pi(1))$. As only one of $(P_{21},P_{22})$ is possibly active, the the fixed point of $\hat{L}_{\pi}^{\dagger,0}$is either the one that is active of these two, or if neither are then $P_0$ fixed point must be active and be the fixed point of $\hat{L}_{\pi}^{\dagger,0}$.
\end{rmk}

Note that $P_0$ corresponds to a contraction, however the other pieces may or may not be contractions. We let $\epsilon_1=\epsilon(1,\pi(1))$ and $\epsilon_2=\epsilon(2,\pi(2))$ for the sequel. Also, regardless of whether the active piece is a contraction, a fixed point to the active piece exists and is equal to $(I-P_i)^{-1}c$.

\begin{proposition}\label{prop:whencontraction}
The active piece $P_2$ of $\hat{L}_{\pi}^{\dagger,0}$  is a contraction unless it holds that $p_{11}+p_{22} < \epsilon_2$ and
    \begin{equation}1+p_{11}(p_{22}-\epsilon_2)+p_{11}+p_{22}-\epsilon_2 <  (p_{12}-\epsilon_1)p_{21}.\label{eqn:conditionsOnP2}\end{equation}
    For this to occur, we must necessarily have $\epsilon_2\ge\epsilon_1$, $p_{12}\ge \epsilon_1$ and $\epsilon_2\ge p_{22}$.
\end{proposition}
\begin{proof}
Here we consider only the cases where $ \epsilon(s,\pi(s))<1$ for all $s$ otherwise this reduces as in \Cref{rmk:epsilongreaterthan1} to the pieces where $ \epsilon_i<1$. To determine whether $P_0$ corresponds to a contraction, we determine whether the eigenvalues are less than 1.

The eigenvalues of $P_2$ are the roots of the following polynomial in $\lambda$
\begin{align*} \det(P-\epsilon\delta_{s''}-\lambda I) &= (p_{11}-\lambda)(p_{22}-\epsilon_2 -\lambda) - (p_{12}-\epsilon_1)p_{21} \\
&= \lambda^2 -(p_{11}+p_{22}-\epsilon_2)\lambda + p_{11}(p_{22}-\epsilon_2)-(p_{12}-\epsilon_1)p_{21}.
\end{align*}
This has roots
\begin{align*}\lambda_{\pm} &= \frac{(p_{11}+p_{22}-\epsilon_2)\pm \sqrt{(p_{11}+p_{22}-\epsilon_2)^2-4(p_{11}(p_{22}-\epsilon_2)-(p_{12}-\epsilon_1)p_{21})}}{2}\\
&= \frac{(p_{11}+p_{22}-\epsilon_2)\pm \sqrt{(p_{11}-p_{22}+\epsilon_2)^2+4(p_{12}-\epsilon_1)p_{21}}}{2}\\
\end{align*}
We can assume at least one of $\epsilon_1-p_{12}>0$ or $\epsilon_2-p_{22}>0$ is true, (as if both are true then we could redefine $\epsilon_1=\epsilon_2=0$, $p_{12}=p_{12}-\epsilon_1 >0$, and $ p_{22}=p_{22}-\epsilon_2 >0$, and then $P_2$ is a substochastic matrix which is a contraction already by \Cref{lem:LpiIsContraction}.)

We now consider various cases.

\noindent{\bf Imaginary roots} 

Assume both roots $\lambda_{\pm}$ are imaginary, then the norm of $\lambda_{\pm}$ (regardless of which root is chosen) is 
\begin{align*}
    &\frac{1}{2}\sqrt{(p_{11}+p_{22}-\epsilon_2)^2 + 4(p_{11}(p_{22}-\epsilon_2)-(p_{12}-\epsilon_1)p_{21}) - (p_{11}+p_{22}-\epsilon_2)^2} \\
    &= \sqrt{p_{11}(p_{22}-\epsilon_2)-(p_{12}-\epsilon_1)p_{21}}\\
    &\le\sqrt{(1-p_{12})(1-p_{21}-\epsilon_2) - (p_{12}-\epsilon_1)p_{21}} \\
    &= \sqrt{ 1-\epsilon_2 +p_{12}(\epsilon_2 - 1) + p_{21}(\epsilon_1 -1 )}\\
    &\le \sqrt{1-\epsilon_2} \le 1,
\end{align*}
with equality occurring in the last line only when $\epsilon_2 =1$, and equality occurring in the second line only if $P$ is row stochastic (not substochastic).  Hence the spectral radius is less than 1 and $P_2$ is a contraction in this case.

\noindent{\bf Real roots and $p_{12}-\epsilon_1<0$} 

In the case where both $\lambda_{\pm}$ are real and $p_{12}-\epsilon_1<0$, then the norm is 
\begin{align*} \frac{1}{2}\left|(p_{11}+p_{22}-\epsilon_2)\pm \sqrt{(p_{11}-p_{22}+\epsilon_2)^2+4(p_{12}-\epsilon_1)p_{21}}\right| &< \frac{1}{2}\left|(p_{11}+p_{22}-\epsilon_2)\pm \sqrt{(p_{11}-p_{22}+\epsilon_2)^2}\right|\\
& \le \frac{1}{2}\max\{2p_{11},2|p_{22}-\epsilon_2|\} \le 1.
\end{align*}
 Hence the spectral radius is less than 1 and $P_2$ is a contraction in this case.

\noindent{\bf Real roots, $p_{12}-\epsilon_1\ge 0$ and $\epsilon_2\le\epsilon_1$}

In the case where we have $p_{12}-\epsilon_1\ge 0$, then we can assume $p_{22}-\epsilon_2 \le 0$. Then $(p_{12}-\epsilon)p_{21}\le (1-p_{11}-\epsilon)(1-p_{22})$ using that $P$ is substochastic by row. 
We can then rewrite the surd in $\lambda$ as follows 
\begin{align*}(p_{11}-p_{22}+\epsilon_2)^2+4(p_{12}-\epsilon_1)p_{21}& \le (p_{11}-p_{22}+\epsilon_2)^2+4(1-p_{11}-\epsilon_1)(1-p_{22})\\
& = (p_{11}+p_{22}+\epsilon_2 -2)^2 + 4(1-p_{22})(\epsilon_2-\epsilon_1).
\end{align*}
where equality only occurs in the first line if $P$ is row stochastic. If $\epsilon_2 \le \epsilon_1$ then
we see that 
\begin{align*}
    |\lambda| &\le \frac{1}{2}\left|p_{11}+p_{22}-\epsilon_2 \pm |(p_{11}+p_{22}+\epsilon_2 - 2)|\right|\\
    &\le \frac{1}{2}\max\left\{|2p_{11}+2p_{22}-2|, |2-\epsilon_2|)\right\}\le 1.
\end{align*} 
 Hence the spectral radius is less than 1 and $P_2$ is a contraction in this case.

\noindent{\bf Real roots and $p_{11}+p_{22}-\epsilon_2>0$}

In the case $\epsilon_2 >\epsilon_1$, $p_{12}-\epsilon_1\ge0$, $p_{22}-\epsilon_2\le 0$ then $|\lambda| \le 1$ if and only if 
\begin{align*}
(2\pm(p_{11}+p_{22}-\epsilon_2))^2 &\ge (p_{11}+p_{22}-\epsilon_2)^2-4(p_{11}(p_{22}-\epsilon_2)-(p_{12}-\epsilon_1)p_{21})\\
\Leftrightarrow 1\pm(p_{11}+p_{22}-\epsilon_2)&\ge -(p_{11}(p_{22}-\epsilon_2)-(p_{12}-\epsilon_1)p_{21})
\end{align*}
In the case where $p_{11}+p_{22}-\epsilon_2>0$, then the stricter bound is
\begin{align*}
    1-(p_{11}+p_{22}-\epsilon_2)&\ge -(p_{11}(p_{22}-\epsilon_2)-(p_{12}-\epsilon_1)p_{21})\\
 \Leftrightarrow   1+p_{11}(p_{22}-\epsilon_2)&\ge p_{11}+p_{22}-\epsilon_2 +(p_{12}-\epsilon_1)p_{21}\\
  \Leftrightarrow   (1-p_{11})(1-p_{22}+\epsilon_2)&\ge(p_{12}-\epsilon_1)p_{21}\\
  \Leftarrow (1-p_{11})(1-p_{22}+\epsilon_2)&\ge(1-p_{11}-\epsilon_1)p_{21}\\
  \Leftrightarrow (1-p_{11})(1-p_{22}+\epsilon_2-p_{21})+\epsilon_1p_{21}&\ge 0\\
  \Leftarrow (1-p_{11})\epsilon_2+\epsilon_1p_{21}&\ge 0,
\end{align*}
which is always true. Here we used that $P$ is substochastic to move from the third line to the fourth line, and to move from the fifth line to the sixth line. Hence the spectral radius is less than 1 and $P_2$ is a contraction in this case.

\noindent{\bf Real roots and $p_{11}+p_{22}-\epsilon_2<0$}

In the case where $p_{11}+p_{22}-\epsilon_2<0$, then the stricter bound is
\begin{align*}
    1+(p_{11}+p_{22}-\epsilon_2)&\ge -(p_{11}(p_{22}-\epsilon_2)-(p_{12}-\epsilon_1)p_{21})\\
 \Leftrightarrow   1+p_{11}(p_{22}-\epsilon_2)+p_{11}+p_{22}-\epsilon_2&\ge  (p_{12}-\epsilon_1)p_{21}\\
 \Leftarrow 1+p_{11}(p_{22}-\epsilon_2)+p_{11}+p_{22}-\epsilon_2&\ge  (1-p_{11}-\epsilon_1)(1-p_{22})\\
 \Leftrightarrow 2(p_{11}+p_{22})+\epsilon_1(1-p_{22})&\ge \epsilon_2(1+p_{11}).
\end{align*}
This is not always true, for example if $p_{11},p_{22},\epsilon_1$ are small compared to $\epsilon_2$. Here we used that $P$ is substochastic to move from the second line to the third line.  Hence the spectral radius may not be less than 1 and $P_2$ may not be a contraction in this case. The previous cases suggest this necessarily also requires $\epsilon_2\ge\epsilon_1$, $p_{12}\ge \epsilon_1$ and $\epsilon_2\ge p_{22}$.
\end{proof}

\begin{example}\label{ex:iteratingLhatdaggerpi01}
If we have $p_{11}=0.1, p_{12}=0.89, p_{21}=0.89, p_{22}=0.1$ and $\epsilon_2=0.9$, $\epsilon_1=0.1$ then $P_2$ is not a contraction. It has eigenvalues $0.6016$ and $-1.3016$ therefore has spectral radius of $1.3016>1$.

However, $\hat{L}_{\pi}^{\dagger,0}$ still acts as a contraction during simulations. For example, if we let the costs equal to $0.01$ then iterating we converge to $x=(0.019694135768511, 0.010892287380350)$. If we iterated $\hat{L}_{\pi}$ we would converge to $J^*=(1,1)$. From \Cref{lem:existenceFixedPoint} any fixed point of, and indeed the image of, $\hat{L}_{\pi}^{\dagger,0}$ is contained in the set 
\[ [0.01,1]\times [0.01,1] .\] 
Plotting this set, the points $x,J^*$ and the costs in \Cref{fig:exampleContraction}, we see that due to $\epsilon_2$ being large, then $x$ is close to the costs and away from $J^*$. It appears that $\hat{L}_{\pi}^{\dagger,0}$ contracts as it avoids being in the active piece $P_2$ (the red shaded region) by moving to other active pieces (the blue shaded region) initially and then staying there.

\begin{figure}
    \centering
    \includegraphics[width=\linewidth]{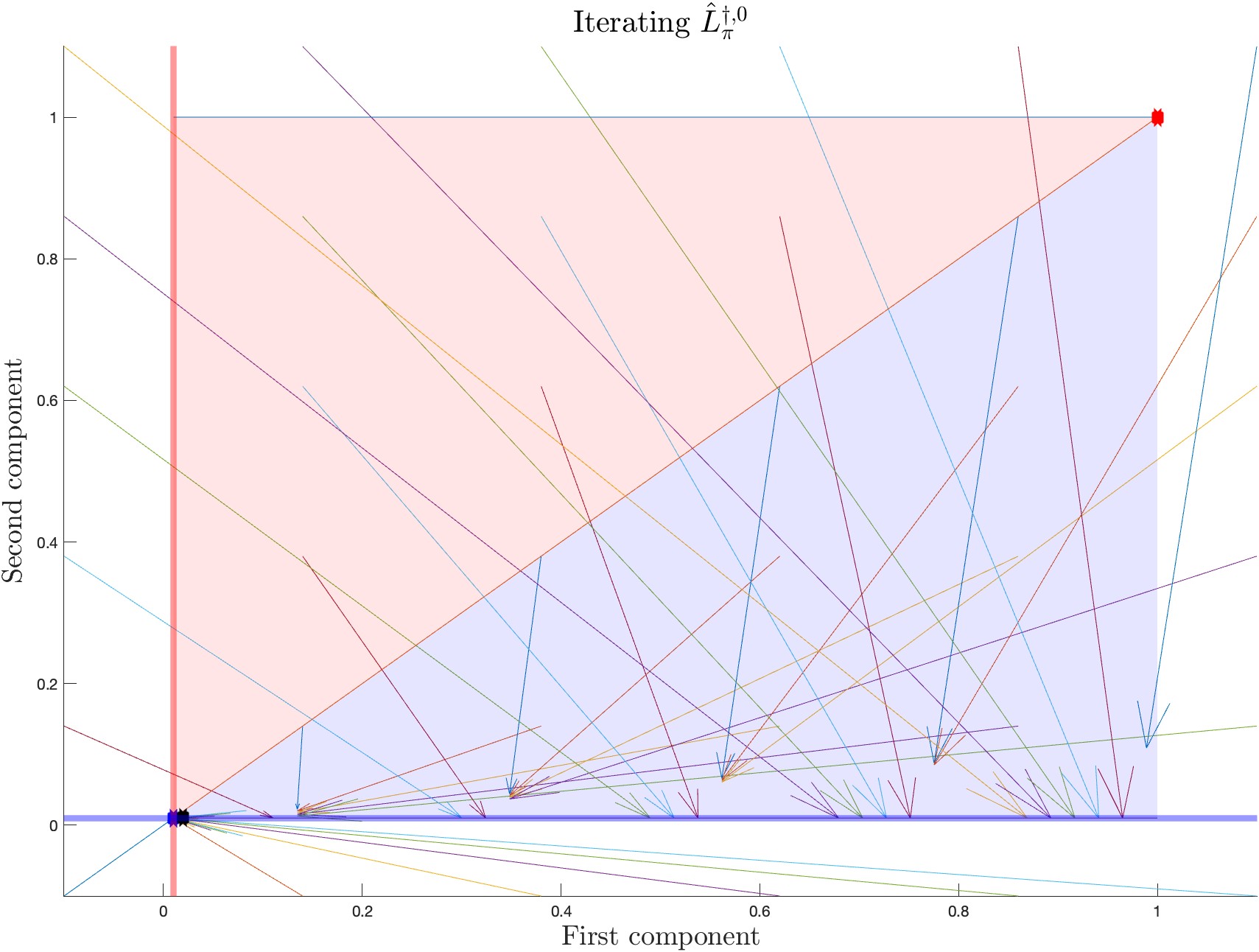}
    \caption{Here we illustrate iterating $\hat{L}_{\pi}^{\dagger,0}$ where $p_{11}=0.1, p_{12}=0.89, p_{21}=0.89, p_{22}=0.1$ and $\epsilon_2=0.9$, $\epsilon_1=0.1$. The black dot plots $x=(0.019694135768511, 0.010892287380350)$, the fixed point of $\hat{L}_{\pi}^{\dagger,0}$. The red dot plots $J^* = (1,1)$, the fixed point of $\hat{L}_{\pi}$. The blue dot, and the intersection of the red and blue lines, is the costs vector $(0.01,0.01)$. The shaded region is the set $\mathcal{I}=[0.01,1]\times [0.01,1]$ which contains any fixed points of $\hat{L}_{\pi}^{\dagger,0}$, as in \Cref{lem:existenceFixedPoint}. The red shaded region is where $P_{2},P_{21},P_{22}$ are active, while the blue shaded region is where $P_{1},P_{11},P_{12}$ are active. \\
    We have also plotted the first iteration of  $\hat{L}_{\pi}^{\dagger,0}$ when applied to vectors in the grid $\{-0.1,0.14,0.38,0.62,0.86,1.1\}\times\{-0.1,0.14,0.38,0.62,0.86,1.1\}$ in various colours using arrows. The start of the arrow is at the initial point in the grid, while the end point (where the arrow head lies) is the image of the starting point under  $\hat{L}_{\pi}^{\dagger,0}$. We can see that all elements of the grid outside $\mathcal{I}$ land in $\mathcal{I}$. Also, it appears that elements in the red shaded region land in the blue shaded region, and that these slowly converge to $x$.}
    \label{fig:exampleContraction}
\end{figure}
\end{example}

Unlike in \Cref{ex:iteratingLhatdaggerpi01}, the behaviour of $\hat{L}_{\pi}^{\dagger,0}$ may involve iterating between the different active pieces many times before convergence, as in the following example.

\begin{example}
If we have $p_{11}=0.00001, p_{12}=0.999 p_{21}=0.999, p_{22}=0.00001$ and $\epsilon_2=0.01$, $\epsilon_1=0.01$ then all active components are contractions.

We observe $\hat{L}_{\pi}^{\dagger,0}$ acting as a contraction during simulations, however it displays alternating behaviour, where it alternatives between the active pieces. For example, if we let the costs equal to $0.01$ then iterating we converge to $x=(0.90991810737, 0.90991810737)$. If we iterated $\hat{L}_{\pi}$ we would converge to $J^*=(10.100556144346518,10.100556144346518)$. From \Cref{lem:existenceFixedPoint} any fixed point of $\hat{L}_{\pi}^{\dagger,0}$ is contained in the set 
\[ [0.01,10.100556144346518]\times [0.01,10.100556144346518] .\] 
We plot this set, the points $x,J^*$ and the costs in \Cref{fig:exampleContraction2} and \Cref{fig:exampleContraction3}.

\begin{figure}
    \centering
    \includegraphics[width=\linewidth]{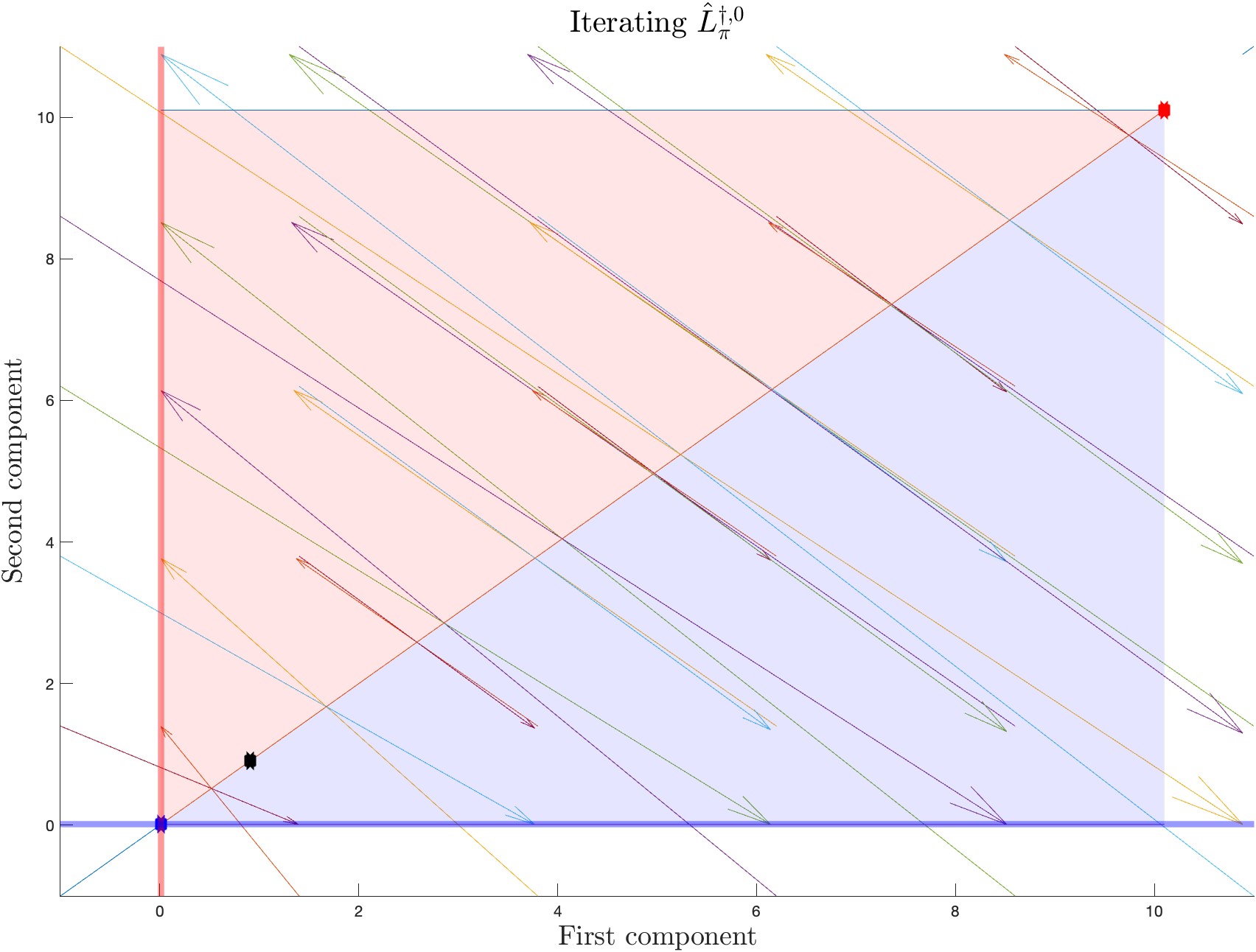}
    \caption{Here we illustrate iterating $\hat{L}_{\pi}^{\dagger,0}$ where $p_{11}=0.00001, p_{12}=0.999, p_{21}=0.999, p_{22}=0.00001$ and $\epsilon_2=0.01$, $\epsilon_1=0.01$. The black dot plots $x=(0.90991810737, 0.90991810737)$, the fixed point of $\hat{L}_{\pi}^{\dagger,0}$. The red dot plots $J^*=(10.100556144346518,10.100556144346518)$, the fixed point of $\hat{L}_{\pi}$. The blue dot, and the intersection of the red and blue lines, is the costs vector $(0.01,0.01)$. The shaded region is the set $\mathcal{I}=[0.01,10.1006]\times [0.01,10.1006]$ which is where any fixed point must lie. The red shaded region is where $P_{2},P_{21},P_{22}$ are active, while the blue shaded region is where $P_{1},P_{11},P_{12}$ are active. \\
    We have also plotted the first iteration of  $\hat{L}_{\pi}^{\dagger,0}$ when applied to vectors in the grid $\{-1,1.4,3.8,6.2,8.6,11\}\times\{-1,1.4,3.8,6.2,8.6,11\}$ in various colours using arrows. The start of the arrow is at the initial point in the grid, while the end point (where the arrow head lies) is the image of the starting point under  $\hat{L}_{\pi}^{\dagger,0}$. We can see that all elements of the grid outside $\mathcal{I}$ land in $\mathcal{I}$. Also, it appears that elements in the red shaded region land in the blue shaded region, and that these slowly converge to $x$.}
    \label{fig:exampleContraction2}
\end{figure}

\begin{figure}
    \centering
    \includegraphics[width=\linewidth]{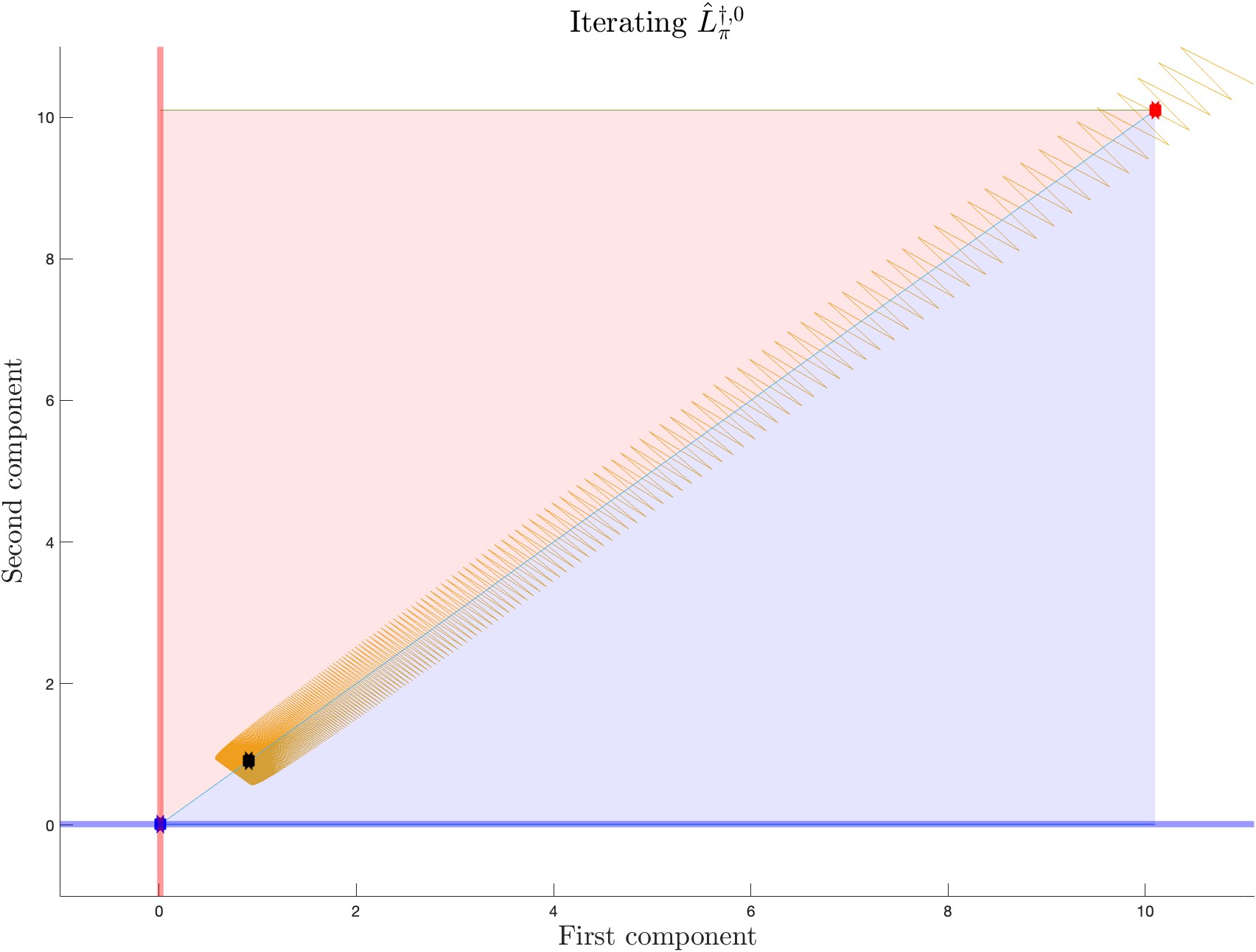}
    \caption{Here we illustrate iterating $\hat{L}_{\pi}^{\dagger,0}$ where $p_{11}=0.00001, p_{12}=0.999, p_{21}=0.999, p_{22}=0.00001$ and $\epsilon_2=0.01$, $\epsilon_1=0.01$. The black dot plots $x=(0.90991810737, 0.90991810737)$, the fixed point of $\hat{L}_{\pi}^{\dagger,0}$. The red dot plots $J^*=(10.100556144346518,10.100556144346518)$, the fixed point of $\hat{L}_{\pi}$. The blue dot, and the intersection of the red and blue lines, is the costs vector $(0.01,0.01)$. The shaded region is the set $\mathcal{I}=[0.01,10.1006]\times [0.01,10.1006]$ which is where any fixed point must lie. The red shaded region is where $P_{2},P_{21},P_{22}$ are active, while the blue shaded region is where $P_{1},P_{11},P_{12}$ are active. \\
    We have also plotted all iterations of $(11.1,10.468)$ under $\hat{L}_{\pi}^{\dagger,0}$ in orange. It oscillates between the red and blue regions before converging to $x$.}
    \label{fig:exampleContraction3}
\end{figure}
\begin{figure}
    \centering
    \includegraphics[width=\linewidth]{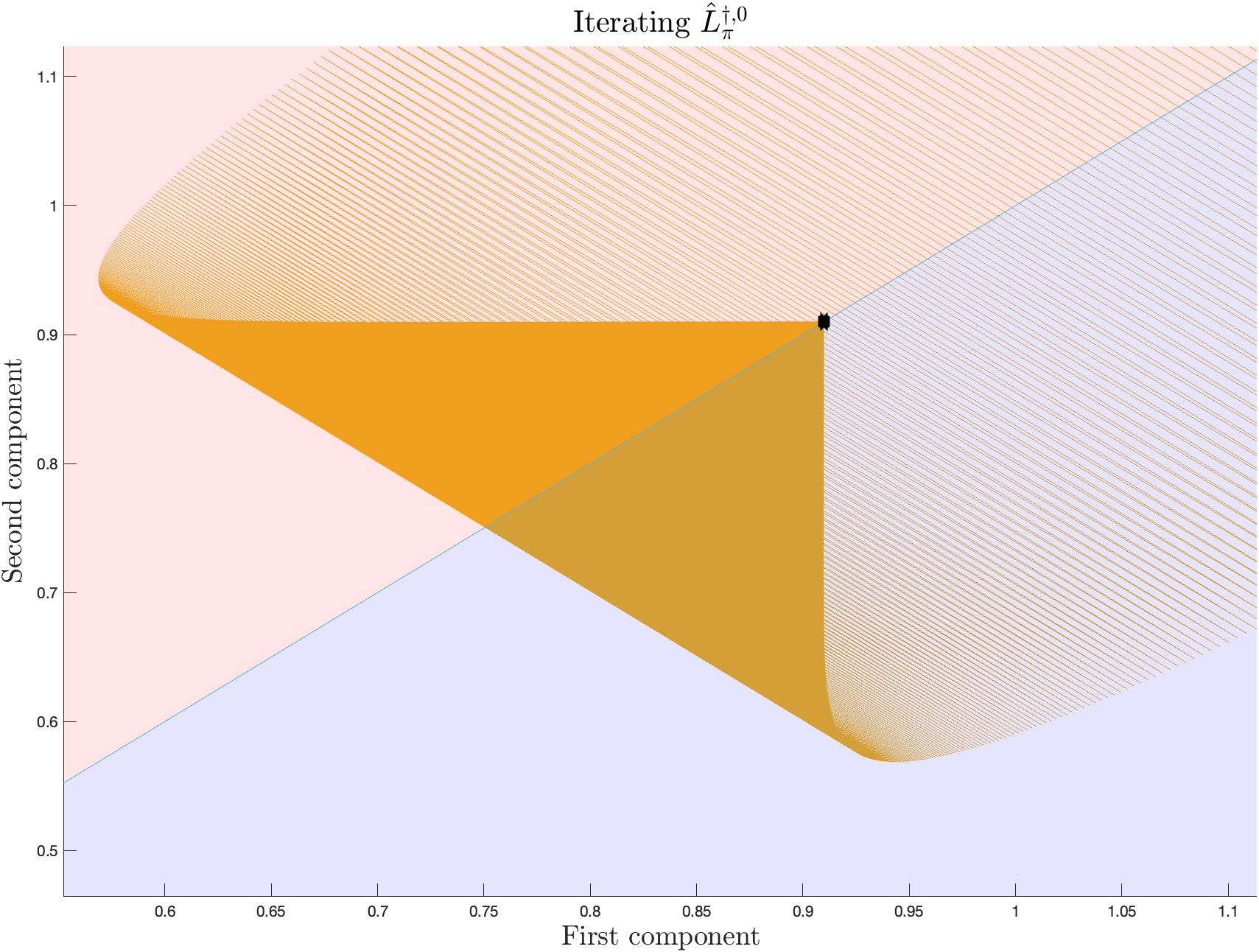}
    \caption{Close up of last few iterations of $\hat{L}_{\pi}^{\dagger,0}$ applied to $(11.1,10.468)$ before it converges to $x=(0.90991810737, 0.90991810737)$. The iterations continue past the fixed point before oscillating back to $x$.}
    \label{fig:exampleContraction4}
\end{figure}
\end{example}

The active pieces corresponding to $P_1,P_{11},P_{12},P_{21},P_{22}$ can be considered in the same way as \Cref{prop:whencontraction} by either swapping $p_{12}$ and $p_{21}$ for $P_1,P_{11},P_{21}$ and by taking $p_{11},p_{12},p_{21},\allowbreak p_{22},\epsilon_1,\epsilon_2$ to be equal to zero where appropriate.
\begin{corollary}
The active pieces corresponding to $P_{21}$ and $P_{12}$ are always contractions. 
\end{corollary}
\begin{proof}
The corresponding equation to \Cref{eqn:conditionsOnP2} is
\[ 1+p_{11} <  0 \]
for $P_{21}$, which is always false, and 
\[ 1+p_{22} <  0 \]
for $P_{12}$, which is again always false, so the active pieces corresponding to these matrices are contractions.
\end{proof}

\begin{proposition}
Only one of the fixed points of each of the pairs $(P_1,P_2)$, $(P_{11},P_{21})$ and $(P_{21},P_{22})$ are in the active piece of $\hat{L}_{\pi}^{\dagger,0}$, unless the fixed points are equal and are a multiple of $x=(1,1)$.
\end{proposition}

\begin{proof}
Firstly take any piece $P_k$ and consider that, provided a fixed point exists, it is equal to $(I-P_k)^{-1}c$ where $c$ is the costs vector and $I$ is the $N\times N$ identity matrix.

We now compare the fixed points of $P_{1}$ and $P_{2}$. 

The fixed point of $P_1$ is 
\[\begin{bmatrix} x_1 \\ x_2 \end{bmatrix}= \frac{1}{(1-p_{11}+\epsilon_1)(1-p_{22})-(p_{21}-\epsilon_2)p_{12}} \begin{bmatrix} 1-p_{22} & p_{12}\\ p_{21}-\epsilon_2 & 1-p_{11}+\epsilon_1\end{bmatrix} \begin{bmatrix} c_1 \\ c_2 \end{bmatrix} .\] This fixed point is in the active piece if and only if $x_1\ge x_2$, which is if and only if 
\begin{equation} (1-p_{22}-p_{21}+\epsilon_2)c_1 + (p_{12}+p_{11}-\epsilon_1-1)c_2\ge 0. \label{eqn: fixedpoint1}\end{equation}

The fixed point of $P_2$ is 
\[\begin{bmatrix} y_1 \\ y_2 \end{bmatrix}= \frac{1}{(1-p_{11})(1-p_{22}+\epsilon_2)-p_{21}(p_{12}-\epsilon_1)} \begin{bmatrix} 1-p_{22}+\epsilon_2 & p_{12}-\epsilon_1\\ p_{21} & 1-p_{11}\end{bmatrix} \begin{bmatrix} c_1 \\ c_2 \end{bmatrix} .\] This fixed point is in the active piece if and only if $y_2\ge y_1$ which is if and only if 
\begin{equation} (1-p_{22}-p_{21}+\epsilon_2)c_1 + (p_{12}+p_{11}-\epsilon_1-1)c_2\le 0.\label{eqn: fixedpoint2}\end{equation}
Notice that either exactly one of \Cref{eqn: fixedpoint1} or \Cref{eqn: fixedpoint2} is true or they are both true. In the former case, this means that one one of the fixed points of $P_1$ and $P_2$ occurs in the active piece. 

In the latter case, we have $x_1=x_2$ and $y_1=y_2$ and
\begin{equation}(1-p_{22}-p_{21}+\epsilon_2)c_1 =(1+\epsilon_1 -p_{12}-p_{11})c_2.\label{eqn: fixedpointequality}\end{equation}
We can write $x_1$ as
\begin{align*}
    x_1 & = \frac{(1-p_{22})c_1+p_{12}c_2}{\det(1-P_1)}\\
    & =  \frac{c_2}{1-p_{22}-p_{21}+\epsilon_2}\frac{(1-p_{22})(1+\epsilon_1 -p_{12}-p_{11})+p_{12}(1-p_{22}-p_{21}+\epsilon_2)}{\det(1-P_1)}\\
     & =  \frac{c_2}{1-p_{22}-p_{21}+\epsilon_2}\frac{\det(1-P_1)}{\det(1-P_1)}\\
     & = \frac{c_2}{1-p_{22}-p_{21}+\epsilon_2},
\end{align*}
using \Cref{eqn: fixedpointequality} to move from the first line to the second. Then we can similarly write $y_2$ as
\begin{align*}
    y_2 & = \frac{p_{21}c_1+(1-p_{11})c_2}{\det(1-P_2)}\\
    & =  \frac{c_2}{1-p_{22}-p_{21}+\epsilon_2}\frac{p_{21}(1+\epsilon_1 -p_{12}-p_{11})+(1-p_{11})(1-p_{22}-p_{21}+\epsilon_2)}{\det(1-P_2)}\\
     & =  \frac{c_2}{1-p_{22}-p_{21}+\epsilon_2}\frac{\det(1-P_2)}{\det(1-P_2)}\\
     & = \frac{c_2}{1-p_{22}-p_{21}+\epsilon_2},
\end{align*}
again using \Cref{eqn: fixedpointequality} to move from the first line to the second. So we have 
\[x_2=x_1=y_2=y_1.\] This means the fixed points are equal and are a multiple of $(1,1)$.

We can consider $(P_{11},P_{21})$  and $(P_{21},P_{22})$ in the same way, by taking $p_{11},p_{12},p_{21},\allowbreak p_{22},\epsilon_1,\epsilon_2$ to be equal to zero where appropriate.
\end{proof}

This gives some intuition for determining the fixed point(s) of  $\hat{L}_{\pi}^{\dagger,0}$. We expect that the fixed point of  $\hat{L}_{\pi}^{\dagger,0}$ is the fixed point corresponding whichever of $P_1$ and $P_2$ are in the active piece. However, they may not be possible fixed points of  $\hat{L}_{\pi}^{\dagger,0}$ if they fall outside the feasible region for active points bounded by the costs and the fixed point $J^*$ of  $\hat{L}_{\pi}$. In this case, the remaining fixed points may be a possible fixed point of  $\hat{L}_{\pi}^{\dagger,0}$. We detail this below.

\begin{conjecture} \label{conj:FPprocedure}
The fixed point to $\hat{L}_{\pi}^{\dagger,0}$ can be found by following this procedure:
\begin{itemize}
    \item Calculate $x_0,x_1,x_2,x_{11},x_{12},x_{21},x_{22}$ the fixed points of $P_0,P_1,P_2,P_{11},P_{12},P_{21},P_{22}$ respectively. Calculate the fixed point $J^*$ of $\hat{L}_{\pi}$.
    \item Discard any $x_i$ that does not satisfy $x_0\le x_i \le J^*$.
    \item Discard any $x_i$ whose fixed point is not in the corresponding active piece of $P_i$.
    \item The fixed point of $\hat{L}_{\pi}^{\dagger,0}$ is either whichever of $P_1$ or $P_2$ remains, or if neither remains, it is the remaining fixed point whose elements have the highest sum.
\end{itemize}
\end{conjecture}
Through simulations we have seen such a procedure hold for dimensions greater than 2. Note that we assume in this procedure that there is a unique fixed point, given we already know that one exists from \Cref{lem:existenceFixedPoint}.

\begin{example}\label{ex:oscillatingBehaviour}
The $\hat{L}_{\pi}^{\dagger,0}$ operator may oscillate between two points. We see this in \Cref{fig:OscillatingPoints}, where starting at the magenta point and iterating the operator follows the orange line oscillating between different points in the blue and red regions which correspond to the active pieces $P_1$ and $P_2$ respectively. Eventually it oscillates between the two black points. 

Note that we observed this oscillating behaviour between these two points for various other starting points, except the green point which a fixed point of $\hat{L}_{\pi}^{\dagger,0}$ found through \Cref{conj:FPprocedure}. This fixed point is associated to the active piece corresponding to $P_2$. The eigenvalues of $P_2$ are $-1.0529,0.85295$, so the spectral radius is greater than one. This means that the operator expands rather than contracts around this fixed point. We understand that this oscillating behaviour occurs because of this.
\begin{figure}
    \centering
    \includegraphics[width=\linewidth]{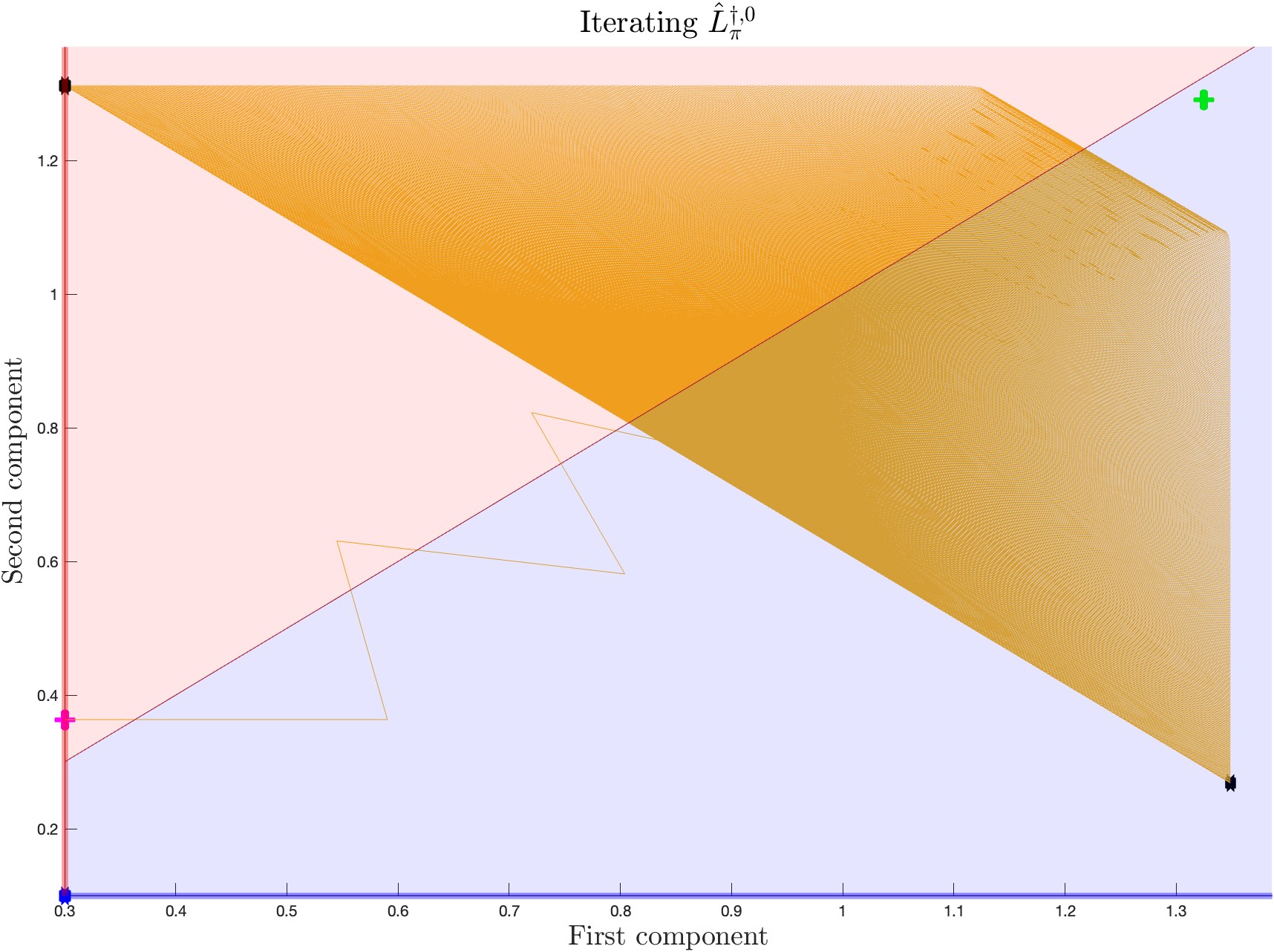}
    \caption{Example of oscillating behaviour for $\hat{L}_{\pi}^{\dagger,0}$. Here $p_{11}=0.00001, p_{12}=0.999, p_{21}=0.999, p_{22}=0.00001$, $\epsilon_1=0.2$, $\epsilon_2=0.1$, and $c_1=0.3,c_2=0.1$. The black dots plot $x=(0.3, 1.3124)$ and $x=(1.34862, 0.26847)$, and the operator $\hat{L}_{\pi}^{\dagger,0}$ oscillates between these two points. The magenta cross plots $J^*=(0.3,0.363367)$ where we start iterating $\hat{L}_{\pi}^{\dagger,0}$. All subsequent iterations are plotted using the orange line until the operator starts to oscillate between the black points. \\
    The green cross plots the fixed point of $\hat{L}_{\pi}^{\dagger,0}$ as found through \Cref{conj:FPprocedure}. The blue dot is the costs vector $(0.3,0.1)$. The red shaded region is where $P_{2},P_{21},P_{22}$ are active, while the blue shaded region is where $P_{1},P_{11},P_{12}$ are active.}
    \label{fig:OscillatingPoints}
\end{figure}

In general we found that cases were such behaviour occurred were found only very rarely in our simulations. As an aside, we note that if we modified $\hat{L}_{\pi}^{\dagger,0}$ to be bounded below by $0$ (instead of by the costs), as we discuss in \Cref{rmk:whyboundbycosts}, we observed this type of osciallating behaviour occurring very frequently in simulations, which is why we did not study this operator further.
\end{example}

\subsection{Discussion of results}
We have devoted this section into trying to understand the contractive properties of the operators $\hat{L}_{\pi}^{\dagger,0}$ and $\hat{U}^{\dagger,0}$. We focused much of our attention on $\hat{L}_{\pi}^{\dagger,0}$, as in general $\hat{U}^{\dagger,0} = \min_{\pi}\hat{L}_{\pi}^{\dagger,0}$. 

We see that $\hat{L}_{\pi}^{\dagger,0}$ has at least one fixed point by \Cref{lem:existenceFixedPoint}, and is always a contraction in the $|\Sc|=1$ case. We also see that $\hat{L}_{\pi}^{\dagger,0}$ is not in general monotone. 

In the $|\Sc|=2$ case, we found that there was a case (detailed in \Cref{prop:whencontraction} and \Cref{ex:oscillatingBehaviour}) where the operator was not a contraction and developed oscillating behaviour. This case required $\epsilon$ to be large enough, and for a non-contractive piece to be active around its fixed point. We still found that a fixed point existed in this example, and that this fixed point was equal to the optimal cost-to-go found by solving \Cref{CP primalUnknowndaggerl1}.

In general, we have empirically observed that the fixed point can be found by following \Cref{conj:FPprocedure}; that it equal to iterating $\hat{L}_{\pi}^{\dagger,0}$ whenever no-oscillating behaviour exists; and that it is equal to optimal cost-to-go found by solving \Cref{CP primalUnknowndaggerl1}. This validates \Cref{conj:l1dagger}. Further work is need to prove \Cref{conj:l1dagger} in the $|\Sc|=1$ case and in general. 

Some of this further work may be in developing understanding the contractice properties of piece-wise linear operators, particularly non-monotonic ones. This research could benefit from some of the literature on piece-wise PDE theory, such as \citet{nogueira2015} and \citet{GRIEWANK2015}.

Overall, the aim of understanding the contractive properties of $\hat{U}^{\dagger,0}$ was to substitute iterating the more computationally complex EVI operator $\hat{U}$ from \Cref{subsubsec:EVI} with the easy to calculate $\hat{L}_{\pi}^{\dagger,0}$. In all cases except the oscillating case, we have seen this work empirically. In the oscillating case, which occurs infrequently, there may be ways to adjust the algorithm. For example, to find the fixed point through \Cref{conj:FPprocedure}, or find the optimal point through solving \Cref{CP primalUnknowndaggerl1}, or even deviating back to iterating $\hat{U}$. We also observed that taking $x_s = \max_ix_{is}$ where $x_{i_1},x_{i_2}$ are the oscillating points results in $x$ being quite close to the fixed point found through \Cref{conj:FPprocedure}, which may be an option to investigate.

\section{\texorpdfstring{Alternative modified $\hat{P}$ for the KL-divergence from \Cref{subsec:KLdiv} }{Different modified P-hat for the KL-divergence}}\label{appn:alternativeKLdiv}
Adapting \Cref{subsec:KLdiv}, here we use the \emph{modified empirical transition probability} as in \citet[pg.~21]{neuPikeBurke} defined by 
\[\hat{P}^{+}(s'|s,a) = \frac{\max\{1,N(s,a,s')\}}{N(s,a)}.\]
This redefines any zeros of $\hat{P}$ to be equal to $1/N(s,a)$, however this does not ensure that $\hat{P}^+$ is a valid probability distribution. Let $\mathcal{S}_{0}\subseteq\mathcal{S}$ be the corresponding states where $\hat{P}$ is zero for a given $s,a$.

We now use the unnormalised relative KL-divergence to account for $\hat{P}^+$ not necessarily being a valid probability distribution. Here 
\[ D_{RKL}(\tilde{P}(\cdot|s,a),\hat{P}(\cdot|s,a)) = \sum_{s'}\tilde{P}(s'|s,a)\log\frac{\tilde{P}(s'|s,a)}{\hat{P}^+(s'|s,a)} + \sum_{s'} (\hat{P}^+(s'|s,a)-P(s'|s,a)).\]

We then have that
\begin{align*}
&\CB_{\min}(s,a)(x) \\
& = \min_{\tilde{P}\in \Delta} \{\langle x,\tilde{P}-\hat{P}\rangle \mid D_{RKL}(\tilde{P},\hat{P}) \leq \epsilon(s,a) \} \\
& =  \max_{\lambda \le 0} \min_{\tilde{P}\in \Delta} \{\langle x,\tilde{P}-\hat{P}\rangle - \lambda\left( D_{RKL}(\tilde{P},\hat{P})-\epsilon(s,a)\right) \} \\
& = - \min_{\lambda \le 0} \max_{\tilde{P}\in \Delta} \{-\left(\langle x,\tilde{P}-\hat{P}\rangle - \lambda \left(D_{RKL}(\tilde{P},\hat{P})-\epsilon(s,a)\right)\right) \}\\
& = - \min_{\lambda \le 0} \max_{\tilde{P}\in \Delta} \{-\left(\langle x,\tilde{P}-\hat{P}\rangle  - \lambda \left(D_{KL}(\tilde{P},\hat{P}^+)+\langle 1, \hat{P}^+ - \tilde{P}\rangle-\epsilon(s,a)\right)\right) \}\\
& = - \min_{\lambda \le 0} \max_{\tilde{P}\in \Delta} \{-\left(\langle x,\tilde{P}-\hat{P}\rangle  - \lambda \left(D_{KL}(\tilde{P},\hat{P}^+)-\epsilon'(s,a)\right)\right) \ \\
& = - \min_{\lambda \ge 0} \max_{\tilde{P}\in \Delta} \{\left(\langle x,\hat{P}-\tilde{P}\rangle  - \lambda \left(D_{KL}(\tilde{P},\hat{P})+\sum_{s'\in \mathcal{S}_0}\tilde{P}\log\frac{\tilde{P}}{\hat{P}^+}-\epsilon'(s,a)\right)\right) \} \\
& \ge - \min_{\lambda \ge 0} \max_{\tilde{P}\in \Delta} \{\left(\langle x,\hat{P}-\tilde{P}\rangle  - \lambda \left(\frac{1}{2\log2}\lVert\tilde{P}-\hat{P}\rVert_{1}^2+\sum_{s'\in \mathcal{S}_0}\tilde{P}\log\frac{\tilde{P}}{\hat{P}^+}-\epsilon'(s,a)\right)\right) \}  \\
& = - \min_{\lambda \ge 0} \max_{\tilde{P}\in \Delta} \{\left(\langle \frac{1}{\sqrt{\lambda}}x,\sqrt{\lambda}(\hat{P}-\tilde{P})\rangle  -  \frac{1}{2}\lVert\sqrt{\lambda}(\tilde{P}-\hat{P})\rVert_{1}^2-\lambda\log2\sum_{s'\in \mathcal{S}_0}\tilde{P}\log\frac{\tilde{P}}{\hat{P}^+}+\lambda\epsilon'(s,a)\log2\right) \} \\
& \ge - \min_{\lambda \ge 0}  \{\left(\frac{1}{2\lambda}\lVert x\rVert_{\infty}^2-\lambda\log2\sum_{s'\in \mathcal{S}_0}\min_{\tilde{P}\in \Delta}\tilde{P}\log\frac{\tilde{P}}{\hat{P}^+}+\lambda\epsilon'(s,a)\log2\right) \} \\
& \ge - \min_{\lambda \ge 0}  \{\left(\frac{1}{2\lambda}\lVert x\rVert_{\infty}^2+\lambda\frac{\log2}{e}\sum_{s'\in \mathcal{S}_0}\hat{P}^+(s'|s,a)+\lambda\epsilon'(s,a)\log2\right) \} \\
& = - 2\lVert x\rVert_{\infty}\sqrt{\frac{\log2}{2}\Big(\frac{1}{e}\sum_{s'\in \mathcal{S}_0}\hat{P}^+(s'|s,a)+\epsilon'(s,a)\Big)}.
\end{align*}
Here we have defined $\epsilon'(s,a) = \epsilon(s,a) +1 - \langle 1,\hat{P}^{+} \rangle $. We use Pinkser's inequality (see \citet[pg.~26]{Yeung2008}) to move from line 6 to line 7, and we use the relationship between convex conjugates of norms from \cite[Ex.~3.27]{Boyd} to move from line 8 to line 9. We also use the result that the minimum with respect to $x$ of $x\log(\frac{x}{a})$ is $-\frac{a}{e}$ occurring at $x=\frac{a}{e}$ as in \Cref{appn:minofLogfunction}, and that the minimum with respect to $\lambda$ of $a\lambda + \frac{b}{\lambda}$ is $\sqrt{ab}$ occurring at $\lambda = \sqrt{\frac{b}{a}}$ as in \Cref{appn:minAlambdaBoneOnLambda}. Both of these minima can be found by finding the first derivative and setting equal to zero and noting that the functions are convex.

Instead of using Pinkser's inequality, we can use that the convex conjugate of a KL divergence is the cumulant function as in \citet[Cor.~4.14]{Boucheron2012}. We can then proceed from line 6 above as follows
\begin{align*}
&\CB_{\min}(s,a)(x) \\
& = - \min_{\lambda \ge 0} \max_{\tilde{P}\in \Delta} \{\langle -x,\tilde{P}-\hat{P}\rangle  - \lambda \left(D_{KL}(\tilde{P},\hat{P})+\sum_{s'\in \mathcal{S}_0}\tilde{P}\log\frac{\tilde{P}}{\hat{P}^+}-\epsilon'(s,a)\right) \} \\
& \ge - \min_{\lambda \ge 0}  \{\lambda\log\left(\sum_{s'}\hat{P}(s')e^{\frac{-1}{\lambda}(x_{s'}-\langle \hat{P},x\rangle)}\right)-\lambda\sum_{s'\in \mathcal{S}_0}\min_{\tilde{P}\in \Delta} \tilde{P}\log\frac{\tilde{P}}{\hat{P}^+}+\lambda\epsilon'(s,a) \} \\
& \ge - \min_{\lambda \ge 0, \lambda \ge \lVert x-\langle \hat{P},x\rangle \rVert_{\infty})}  \{\frac{1}{\lambda}\sum_{s'}\hat{P}(s')(x_{s'}-\langle \hat{P},x\rangle)^2+\lambda\frac{\log2}{e}\sum_{s'\in \mathcal{S}_0}\hat{P}^+(s'\mid s,a) +\lambda\epsilon'(s,a) \} \\
& = - \begin{cases} 2\sqrt{\hat{V}(x)\left( \frac{\log2}{e}\sum_{s'\in \mathcal{S}_0}\hat{P}^+(s'\mid s,a) +\epsilon'(s,a)\right)} & \text{ if } \epsilon'(s,a)\le f(x,\hat{P}),\\
\frac{1}{\lVert x-\langle \hat{P},x\rangle \rVert_{\infty}}\hat{V}(x)+\lVert x-\langle \hat{P},x\rangle \rVert_{\infty}\left(\frac{\log2}{e}\sum_{s'\in \mathcal{S}_0}\hat{P}^+(s'\mid s,a) +\epsilon'(s,a)\right) & \text{ if } \epsilon'(s,a)> f(x,\hat{P}).
\end{cases}
\end{align*}
Here $\hat{V}(x) = \sum_{s'}\hat{P}(s')(x_{s'}-\langle \hat{P},x\rangle)^2$ could be considered the variance of $x$ with respect to $\hat{P}(\cdot|s,a)$. To move from line 2 to line 3 we use result that $\lambda \log \mathbb{E}^+e^{X/\lambda}\le \mathbb{E}^+(X)+ \frac{1}{\lambda}\mathbb{E}^+(X^2)$ whenever $|X|\le \lambda$ with probability 1 (see \Cref{appn: proofBoundCumulant} for details), with $X=x - \langle \hat{P},x\rangle$ considered as a finite random variable with probability distribution $\hat{P}$. Then $|X|\le \lambda$ with probability 1 translates to 
\[\lambda \ge \sup_{s'\in\Sc}\{|x_{s'}-\langle \hat{P}(\cdot|s,a),x\rangle| : \hat{P}(s'|s,a)>0\}\] which we write as $\lambda \ge \lVert x-\langle \hat{P},x\rangle \rVert_{\infty}$.
The function $f(x,\hat{P})$ is 
\[ f(x,\hat{P}) = \hat{V}(x)\frac{1}{\lVert x - \langle \hat{P},x\rangle \rVert_{\infty}^2} -\frac{\log2}{e}\sum_{s'\in\mathcal{S}_0}\hat{P}^+(s'\mid s,a).\] This comes from the location of the minimum of the function $a\lambda + b\frac{1}{\lambda}$ as in \Cref{appn:minAlambdaBoneOnLambda}.

Alternatively, we could proceed using Hoeffding's lemma as in \citet[Lem.~2.2]{Boucheron2012}. Then we have 
\begin{align*}
\CB_{\min}(s,a)(x) & = - \min_{\lambda \ge 0} \max_{\tilde{P}\in \Delta} \{\langle -x,\tilde{P}-\hat{P}\rangle  - \lambda \left(D_{KL}(\tilde{P},\hat{P})+\sum_{s'\in \mathcal{S}_0}\tilde{P}\log\frac{\tilde{P}}{\hat{P}^+}-\epsilon'(s,a)\right) \} \\
& \ge - \min_{\lambda \ge 0}  \{\lambda\log\left(\sum_{s'}\hat{P}(s')e^{\frac{-1}{\lambda}(x_{s'}-\langle \hat{P},x\rangle)}\right)-\lambda\sum_{s'\in \mathcal{S}_0}\min_{\tilde{P}\in \Delta} \tilde{P}\log\frac{\tilde{P}}{\hat{P}^+}+\lambda\epsilon'(s,a) \} \\
& \ge - \min_{\lambda \ge 0}  \{\frac{1}{2\lambda}(\spn(x-\langle \hat{P},x\rangle))^2 +\lambda\frac{\log 2}{e}\sum_{s'\in \mathcal{S}_0}\hat{P}^+(s'\mid s,a)+\lambda\epsilon'(s,a) \} \\
& =-\frac{2}{\sqrt{2}}\sqrt{(\spn(x-\langle \hat{P},x\rangle))^2\left(\frac{\log 2}{e}\sum_{s'\in \mathcal{S}_0}\hat{P}^+(s'\mid s,a)+\epsilon'(s,a)\right)}. 
\end{align*}
We could use any of the above as $\CB_{\min}^{\dagger}$, or alternatively the maximum of the three, and also bounding by $-\hat{P}(\cdot|s,a)x(\cdot)$ as we did in \Cref{eqn:hatDaggerBalloon}.

\section{\texorpdfstring{Alternative approach to the $\chi^2$-divergence bound from \Cref{subsec:Chisquaredivbound}}{Alternative approach to the Chi-squared-divergence bound}} \label{appn:alternativeChisquared}
Here we do not modify $\epsilon$ and instead try to find a different bound for $\CB_{\min}$.
We have that 
\begin{align*}
&\CB_{\min}(s,a)(x) \\
& = \min_{\tilde{P}\in \Delta} \{\langle x,\tilde{P}-\hat{P}\rangle \mid D_{\chi^2}(\hat{P},\tilde{P}) \leq \epsilon(s,a) \} \\
& =\max_{\lambda \le 0} \min_{\tilde{P}\ge 0} \{\langle x-\lambda 1,\tilde{P}-\hat{P}\rangle \mid \left\lVert \frac{\tilde{P}-\hat{P}^+}{\sqrt{\hat{P}^+}} \right\rVert_2^2 \leq \epsilon(s,a) \} \\
&= \max_{\lambda \le 0} \min_{\tilde{P}\ge 0} \{\langle x-\lambda 1,\tilde{P}-\hat{P}^+\rangle + \langle x - \lambda 1,\hat{P}^+-\hat{P}\rangle \mid \left\lVert \frac{\tilde{P}-\hat{P}^+}{\sqrt{\hat{P}^+}} \right\rVert_2^2 \leq \epsilon(s,a) \} \\
& =-\min_{\lambda \le 0} \max_{\tilde{P}\ge 0 } \{-\left\langle \sqrt{\hat{P}^+}(x-\lambda 1),\frac{\tilde{P}-\hat{P}^+}{\sqrt{\hat{P}^+}}\right\rangle - \langle x - \lambda 1,\hat{P}^+-\hat{P}\rangle \mid \left\lVert \frac{\tilde{P}-\hat{P}^+}{\sqrt{\hat{P}^+}} \right\rVert_2 \leq \sqrt{\epsilon(s,a)} \} \\
& \ge -\min_{\lambda \le 0} \max_{\tilde{P}\ge 0 } \{\left\lVert \sqrt{\hat{P}^+}(x-\lambda 1)\right\rVert_2\left\lVert\frac{\tilde{P}-\hat{P}^+}{\sqrt{\hat{P}^+}}\right\rVert_2 - \langle x - \lambda 1,\hat{P}^+-\hat{P}\rangle\mid \left\lVert \frac{\tilde{P}-\hat{P}^+}{\sqrt{\hat{P}^+}} \right\rVert_2 \leq \sqrt{\epsilon(s,a)}\}\\
& \ge -\min_{\lambda \le 0}\{\left\lVert \sqrt{\hat{P}^+}( x-\lambda 1)\right\rVert_2\sqrt{\epsilon(s,a)} -\langle x - \lambda 1,\hat{P}^+-\hat{P}\rangle\}.
\end{align*}
Here we used the Cauchy-Swartz inequality to move from line 4 to line 5 (see for example \citet{Hardy1988}.) At this point, we can take the derivative with respect to $\lambda$ and set equal to zero to find the location of the minimum. In the case where $\hat{P}(\cdot\mid s,a)$ has no zeros, then this occurs at $\lambda = \langle x,\hat{P}^+\rangle$, which would give 
\[\CB_{\min}(s,a)(x) \ge - \sqrt{\epsilon(s,a) \hat{V}^+(x)}.\] However, this is not less than zero, so the minimum occurs at the boundary where $\lambda = 0$ and we have 
\[\CB_{\min}(s,a)(x) \ge - \sqrt{\epsilon(s,a) \sum_{s'}\hat{P}^+(s'\mid s,a)x_{s'}^2}.\]
When $\hat{P}(\cdot\mid s,a)$ has zeros, setting the derivative equal to zero reduces to solving the following quadratic
\[A\lambda^2+B\lambda +C =0\] where 
\begin{align*} A &= \sum_{s'}\hat{P}^+(s')\left (Q-\sum_{s'}\hat{P}^+(s')\right)\\
B & = (2-Q)\sum_{s'}\hat{P}^+(s')x_{s'}\\
C & = \sum_{s'}\hat{P}^+(s')x_{s'}\left(Q-\sum_{s'}\hat{P}^+(s')x_{s'}\right)\\
Q & = \frac{4}{\sqrt{\epsilon(s,a)}}\left ( \sum_{s'} (\hat{P}^+(s')-\hat{P}(s'))\right)^2 .\end{align*}
The usual quadratic formula gives the roots. One then needs to check whether the roots give (global) minima and whether they are less than or equal to zero. However this results in needing to understand the relationship between $\hat{P}^+-\hat{P}$ and $\epsilon$.

\section{\texorpdfstring{Alternative approach to the variance-weighted $\ell_{\infty}$-norm from \Cref{subsec:varweightLinfinitynorm}}{Alternative approach to the variance-weighted l-infinity-norm}} \label{appn:AlternativeVarWeightsupNorm}
Here we do not modify $\epsilon$ and instead try to find a different bound for $\CB_{\min}$. 
We have that
\begin{align*}
&\CB_{\min}(s,a)(x) \\
& = \min_{\tilde{P}\in \Delta} \{\langle x,\tilde{P}-\hat{P}\rangle \mid D_{\infty}(\hat{P},\tilde{P}) \leq \epsilon(s,a) \} \\
& =-\min_{\lambda \le 0} \max_{\tilde{P}\ge 0 } \{-\left\langle \sqrt{\hat{P}^+}(x-\lambda 1),\frac{\tilde{P}-\hat{P}^+}{\sqrt{\hat{P}^+}}\right\rangle - \langle x - \lambda 1,\hat{P}^+-\hat{P}\rangle \mid \left\lVert \frac{\tilde{P}-\hat{P}^+}{\sqrt{\hat{P}^+}} \right\rVert_\infty \leq \sqrt{\epsilon(s,a)} \} \\
& \ge -\min_{\lambda \le 0} \max_{\tilde{P}\ge 0 } \{\left\lVert \sqrt{\hat{P}^+}(x-\lambda 1)\right\rVert_1\left\lVert\frac{\tilde{P}-\hat{P}^+}{\sqrt{\hat{P}^+}}\right\rVert_\infty - \langle x - \lambda 1,\hat{P}^+-\hat{P}\rangle\mid \left\lVert \frac{\tilde{P}-\hat{P}^+}{\sqrt{\hat{P}^+}} \right\rVert_\infty \leq \sqrt{\epsilon(s,a)}\}\\
& \ge -\min_{\lambda \le 0}\{\left\lVert \sqrt{\hat{P}^+}( x-\lambda 1)\right\rVert_1\sqrt{\epsilon(s,a)} -\langle x - \lambda 1,\hat{P}^+-\hat{P}\rangle\}.
\end{align*}

This will likely end up with the minimum at $\lambda =0$, however this again relies on the relationship between $\epsilon$ and $\hat{P}^+ - \hat{P}$ as in \Cref{appn:alternativeChisquared}.

 \section{Lemma regarding maximums and minimums of vectors} \label{appendix2:rearrangeMaxs}
\begin{lemma} \label{lem:maxarrange}
If $x$ and $y$ are vectors with entries $x_{i}$ and $y_{i}$. Say w.l.o.g. we have $\min_ix_{i} \ge\min_i y_{i}$ then for $i'$ such that $y_{i'}$ attains its minimum then we have
\[ |x_{i'} - y_{i'}|=x_{i'} - y_{i'} \ge \min_ix_{i} -\min_i y_{i} = |\min_ix_{i} -\min_i y_{i}|.\]
Similarly, if $\max_ix_{i} \ge\max_i y_{i}$ then for $i'$ such that $x_{i'}$ attains its maximum then we have
\[ |x_{i'} - y_{i'}|=x_{i'} - y_{i'} \ge \max_ix_{i} -\max_i y_{i} = |\max_ix_{i} -\max_i y_{i}|.\]
Then in general we have that 
\begin{enumerate}
    \item 
    $\max_i|(x_{i}-y_{i})| \ge | \min_i(x) - \min_i(y)|$, and
    \item 
    $\max_i|(x_{i}-y_{i})| \ge | \max_i(x) - \max_i(y)|$.
\end{enumerate}

\end{lemma}

\begin{proof}
Say $x_{i_1}$, $y_{i_2}$ are where the minimums are attained, and we have  $x_{i_1} \ge y_{i_2}$. Then $x_{i_2}\ge x_{i_1}\ge y_{i_2}$, so \[|\max_i (x_{i}-y_{i})| \ge x_{i_2} - y_{i_2} \ge x_{i_1}-y_{i_2} = \min_ix_{i} -\min_i y_{i} = |\min_ix_{i} -\min_i y_{i}|.\] This proves the first statement for $i'=i_2$, and item 1. in general.

For the second statement, consider that 
\[ |\max_ix_{i} -\max_i y_{i}| =  |\min_i (-y_{i})-\min_i(-x_{i})|\] and that $\max_ix_{i} \ge\max_i y_{i}$ implies that $-\min_i-x_{i} \ge-\min_i -y_{i}$, so $\min_i-x_{i} \le \min_i -y_{i}$. Then we can apply the first statement to $-y, -x$ respectively, so that 
\begin{align*} |\max_ix_{i} -\max_i y_{i}| &=  |\min_i (-y_{i})-\min_i(-x_{i})| \\
& = \min_i (-y_{i})-\min_i(-x_{i}) \\
& = -y_{i_2}--x_{i_1} \\
&\le -y_{i_1}--x_{i_1} \\
&\le |\max_i(-y_{i}--x_{i})|\\
&= |\max_i(x_{i}-y_{i}|\\
& =|\min_i(y_{i}-x_{i}| \end{align*} where $x_{i_1}$ and $y_{i_2}$ are where the maximums of $x$ and $y$ are attained respectively.

\end{proof}

\section{Lemma on span} \label{appn:proofSpan}

\begin{lemma}
We have that for any vector $f\in \mathbb{R}^n=(f_1,\ldots, f_n)$ 
\[\min_{\lambda\in \R} \lVert f- \lambda 1\rVert_{\infty} = \spn(f) \] 
where $1=(1,\ldots,1)\in\R^n$ and the ``span" of the value function is defined as $\spn(f) = \frac{1}{2}(\max(f) - \min(f))$. We also have that if $f\ge 0$ then
\[\min_{\lambda\le 0} \lVert f- \lambda 1\rVert_{\infty} = \max(f) .\]
\end{lemma}

\begin{proof}
We have \[\min_{\lambda\in \R} \lVert f- \lambda 1\rVert_{\infty} = \min_{\lambda\in \R} \max_i|f_i - \lambda| . \]
If we consider $|f_i-\lambda|$ as functions in $\lambda$, we have that 
\[ |f_i - \lambda| =\begin{cases}\lambda - f_i &\text{ when }\lambda \ge f_i\\
f_i - \lambda & \text{ when } \lambda \le f_i \end{cases}\]

and then for some $x$ we must have
\[ \max_i|f_i - \lambda| =\begin{cases}\lambda - \min_i{f_i} &\text{ when }\lambda \ge x\\
\max{f_i} - \lambda & \text{ when } \lambda\le x \end{cases}.\] It follows that $x$ has the value of $\frac{1}{2}(\max_if_i+\min_if_i)$. Subsequently taking the minimum over $\lambda\in \mathbb{R}$ we see that this occurs precisely at $x$, where the value of this minimum is $\frac{1}{2}(\max_if_i - \min_if_i)=\spn(f)$, as required. We see this in \Cref{fig:span}.

If we consider the minimum over $\lambda \le 0$ and we have $f\ge 0$, then $x\ge 0$ and we see that the minimum occurs at $\lambda =0$ so we have
\[ \min_{\lambda \le 0}\{ \max_i|f_i - \lambda|\} =  \min_{\lambda \le 0} \{\max{f_i} - \lambda\} = \max_{f_i}f_i\]

\begin{figure}
    \centering
\begin{tikzpicture}[scale=0.80, vect/.style={->,shorten >=3pt,>=latex'}]
\tkzDefPoint(0,0){A}
\tkzDefPoint(0,6){B}
\tkzDefPoint(0,-6){C}
\tkzDefPoint(10,0){D}
\tkzLabelPoint[](D){$\lambda$}
\tkzDefPoint(-4,0){E}
\tkzDefPoint(0,3){F}
\tkzLabelPoint[left](F){\(f_i\)}
\tkzDefPoint(6,-2){G}
\tkzDrawSegment[gray,dashed](F,G)
\tkzDrawSegments[vect](A,B)
\tkzDrawSegment(A,C)
\tkzDrawSegments[vect](A,D)
\tkzDrawSegment(E,A)
\tkzDefPoint(0,4){H}
\tkzDefPoint(7,-2){I}
\tkzDrawSegment[gray, dashed](H,I)
\tkzDefPoint(0,-3){J}
\tkzLabelPoint[left](J){\(-\min_if_i\)}
\tkzDefPoint(6,2){K}
\tkzDefPoint(0,-4){L}
\tkzDefPoint(7,2){M}
\tkzDrawSegment[gray, dashed](L,M)
\tkzDefPoint(0,5){N}
\tkzLabelPoint[left](N){\(\max_i f_i\)}
\tkzDefPoint(8,-2){O}
\tkzDefPoint(0,-5){P}
\tkzLabelPoint[left](P){\(- f_i\)}
\tkzDefPoint(8,2){Q}
\tkzDrawSegment[gray,dashed](P,Q)
\tkzDefPoint(40/7,0){MM}
\tkzDefPoint(384/82, 37/41){GG}
\tkzDrawSegment(GG,N)
\tkzDrawSegment(GG,K)
\tkzDrawSegment[gray,dashed](O,GG)
\tkzDrawSegment[gray,dashed](J,GG)
\tkzDrawPoints(GG)
\tkzDefPoint(56/12,-1/8.7){XX}
\tkzLabelPoint[below](XX){\(x\)}
\end{tikzpicture}
    \caption{Graph $\max_i|f_i - \lambda|$ as a function of $\lambda$ and the point $x$ where the minimum occurs.}
    \label{fig:span}
\end{figure}
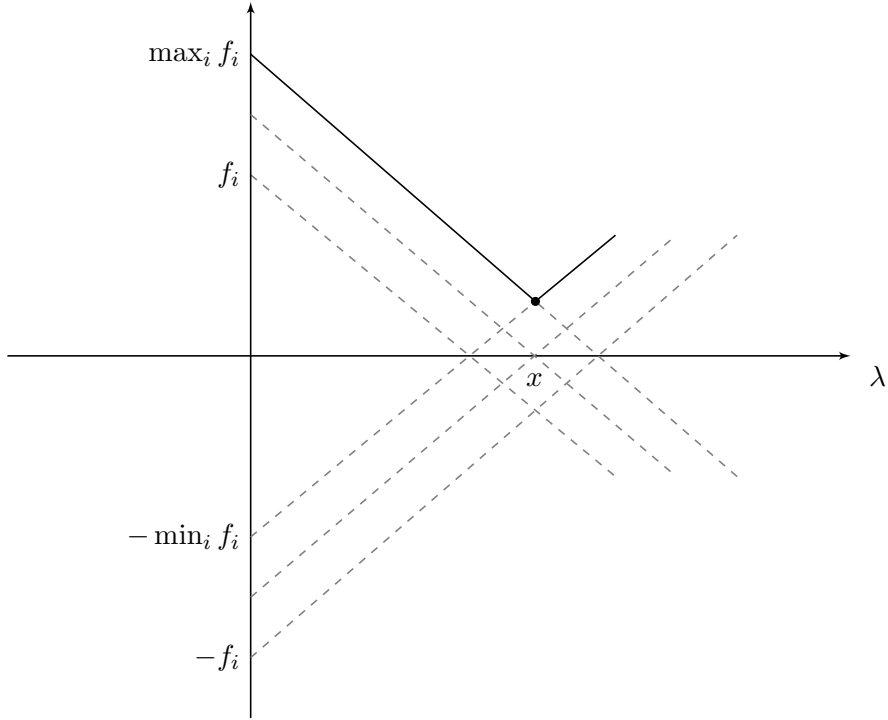

\end{proof}

\section{Proof that the cumulant is bounded}\label{appn: proofBoundCumulant}
We show here that $\lambda \log \mathbb{E}e^{X/\lambda}\le \mathbb{E}(X)+ \frac{1}{\lambda}\mathbb{E}(X^2)$ whenever $|X|\le \lambda$ with probability $1$, which we use in \Cref{subsec:KLdiv}. Thanks to correspondence from Gergely Neu for how to approach this.
\begin{proof}
First we show that $\exp(x)\le 1 + x + x^2$ for $x\in[-1,1]$. One can see this graphically, and we give an analytic proof below.

Define $h(x) = \exp(x)-1-x-x^2$ and note that $h(0) = 0$, $h(1) = -0.281718<0$, $h(-1) = -0.63212056<0$. Now $h'(x) = \exp(x)-1-2x$ and we can check that $h'(0) = 0$. Finally $h''(x) = \exp(x)-2$ and $h''(x)=0$ when $x=\log(2) = 0.69314718$, and where $x<\log(2)$ then $h''(x)<0$ (so $h$ is concave and $h'$ is strictly decreasing in this region) and $x>\log(2)$ is where $h''(x)>0$ (so $h$ is convex and $h'$ is strictly increasing in this region). This means that $h'$ can have at most two zeros, either in $[-1,\log(2)]$ (and we note that this is where $0$ lies, which is a zero of $h'$) or in $[\log(2),1]$. We note that on $[\log(2),1]$ $h'$ is increasing and $h'(\log(2))=-1-\log(2) <0$ and $h'(1)=\exp(1)-3 <0$, so there are no zeros for $h'$ in $[\log(2),1]$. Then for $h$ this means that there is a unique local maximum at $x=0$ with the value of zero, and $h<0$ on the region $[-1,\log(2)]$ and $h$ is concave there, and $h<0$ on $[\log(2),1]$ as $h$ is concave in this region so its maximum in this region must be either at $x=\log(2)$ (where $h(\log(2)) = -0.17360019$) or at $x=1$, but in both cases $h$ is negative. Therefore $h$ is negative at all $x\in [-1,1]$ except $x=0$ where $h(0)=0$, as in \Cref{fig:graph of e^x-1-x-x^2}. This implies the result.

\begin{figure}
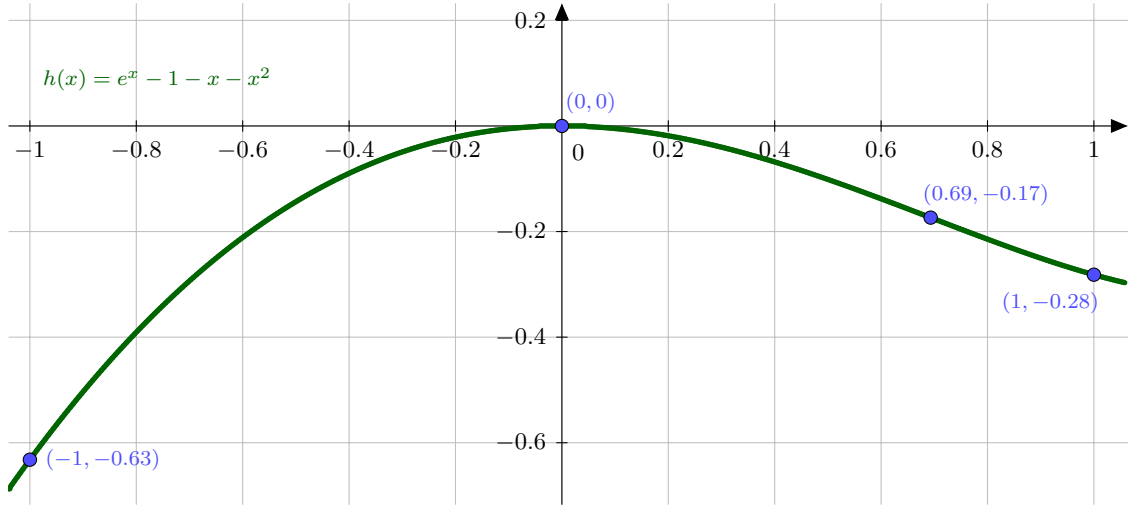
\centering
\definecolor{ududff}{rgb}{0.30196078431372547,0.30196078431372547,1.}
\definecolor{qqwuqq}{rgb}{0.,0.39215686274509803,0.}
\definecolor{cqcqcq}{rgb}{0.7529411764705882,0.7529411764705882,0.7529411764705882}

\caption{Graph of $h(x)=\exp(x)-1-x-x^2$ for $x\in [-1,1]$ with key points highlighted.}\label{fig:graph of e^x-1-x-x^2}
\end{figure}
This implies that \begin{equation}\log\left(\Eb\left(\exp\left(\frac{X}{\lambda}\right)\right)\right) \le \log\left( \Eb\left(1+\frac{X}{\lambda}+\frac{X^2}{\lambda^2}\right)\right) = \log\left(1+\frac{\Eb(X)}{\lambda}+ \frac{\Eb(X^2)}{\lambda^2}\right)\label{eqn:firstboundE}\end{equation} whenever $\frac{X}{\lambda}\in [-1,1]$ with probability 1, so whenever $|X|\le 1$ with probability 1.

Now we show that $\log(1+x)\le x$ for $x> -1$. We see this by considering that $\frac{1}{x}$ is decreasing, and therefore \[\log(1+x)=\int_{1}^{1+x}\frac{1}{t}dt \le (1+x-1)\frac{1}{t}\Big|_{t=1} = x\] by the definition of integration, as in \Cref{fig:graph of 1/x}.

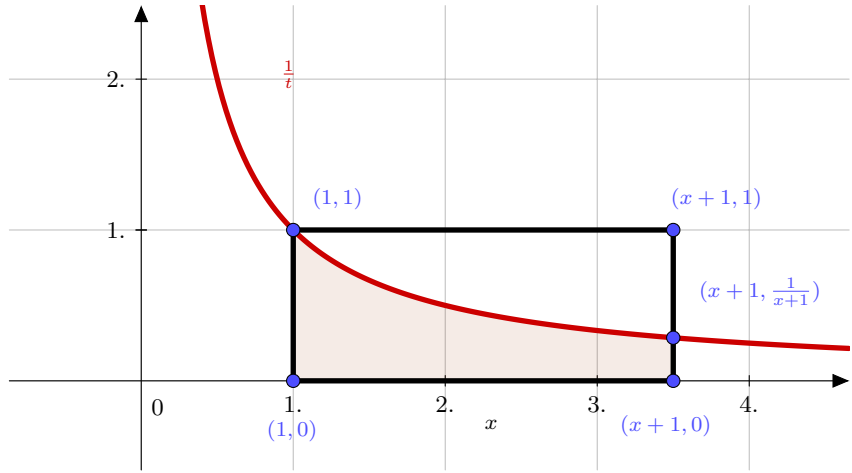
\begin{figure}
    \centering
\definecolor{zzttqq}{rgb}{0.6,0.2,0.}
\definecolor{ududff}{rgb}{0.30196078431372547,0.30196078431372547,1.}
\definecolor{ccqqqq}{rgb}{0.8,0.,0.}
\definecolor{cqcqcq}{rgb}{0.7529411764705882,0.7529411764705882,0.7529411764705882}
\begin{tikzpicture}[line cap=round,line join=round,>=triangle 45,x=2.0cm,y=2.0cm]
\draw [color=cqcqcq,, xstep=2.0cm,ystep=2.0cm] (-0.8677786276810572,-0.5908021433566658) grid (4.657149793221308,2.487536353112843);
\draw[->,color=black] (-0.8677786276810572,0.) -- (4.657149793221308,0.);
\foreach \x in {,1.,2.,3.,4.}
\draw[shift={(\x,0)},color=black] (0pt,2pt) -- (0pt,-2pt) node[below] {\footnotesize $\x$};
\draw[->,color=black] (0.,-0.5908021433566658) -- (0.,2.487536353112843);
\foreach \y in {,1.,2.}
\draw[shift={(0,\y)},color=black] (2pt,0pt) -- (-2pt,0pt) node[left] {\footnotesize $\y$};
\draw[color=black] (0pt,-10pt) node[right] {\footnotesize $0$};
\clip(-0.8677786276810572,-0.5908021433566658) rectangle (4.657149793221308,2.487536353112843);
\draw[line width=0.8pt,color=zzttqq,fill=zzttqq,fill opacity=0.10000000149011612, smooth,samples=50,domain=1.0:3.5] plot(\x,{1.0/\x}) -- (3.5,0.) -- (1.,0.) -- cycle;
\draw[line width=2.pt,color=ccqqqq,smooth,samples=100,domain=0.02:4.657149793221308] plot(\x,{1.0/(\x)});
\draw [line width=2.pt] (1.,1.)-- (1.,0.);
\draw [line width=2.pt] (1.,0.)-- (3.5,0.);
\draw [line width=2.pt] (3.5,0.)-- (3.5,0.2857142857142857);
\draw [line width=2.pt] (1.,1.)-- (3.5,1.);
\draw [line width=2.pt] (3.5,1.)-- (3.5,0.2857142857142857);
\begin{scriptsize}
\draw[color=ccqqqq] (0.9700354000619329,2.0338260150137923) node {$\frac{1}{t}$};
\draw [fill=ududff] (1.,1.) circle (2.5pt);
\draw[color=ududff] (1.2916528549169564,1.2068097025294469) node {$(1,1)$};
\draw [fill=ududff] (3.5,0.2857142857142857) circle (2.5pt);
\draw[color=ududff] (4.071346571878229,0.5865474681661876) node {$(x+1,\frac{1}{x+1})$};
\draw [fill=ududff] (3.5,1.) circle (2.5pt);
\draw[color=ududff] (3.784188130043387,1.2068097025294469) node {$(x+1,1)$};
\draw [fill=ududff] (1.,0.) circle (2.5pt);
\draw[color=ududff] (0.9930080754087203,-0.3323595457053075) node {$(1,0)$};
\draw [fill=ududff] (3.5,0.) circle (2.5pt);
\draw[color=ududff] (3.45108433751497,-0.29790053268512645) node {$(x+1,0)$};
\draw[color=black] (2.302450570175601,-0.28641419501173276) node {$x$};
\end{scriptsize}
\end{tikzpicture}
    \caption{Graph of $\frac{1}{t}$. The shaded area has the value $\int_{1}^{1+x}\frac{1}{t}dt=\log(1+x)$. This is less than the area of the square which has the value $x$.}
    \label{fig:graph of 1/x}
\end{figure}

Applying this to \Cref{eqn:firstboundE} we see that 
\begin{equation}\log\left(\Eb\left(\exp\left(\frac{X}{\lambda}\right)\right)\right) \le \log\left(1+\frac{\Eb(X)}{\lambda}+ \frac{\Eb(X^2)}{\lambda^2}\right) \le \frac{\Eb(X)}{\lambda}+ \frac{\Eb(X^2)}{\lambda^2} \end{equation} whenever $|X|\le 1$ with probability 1. Multiplying by $\lambda$ gives the result. Also, there is no issue using a substochastic probability distribution, which we use in \Cref{subsubsec:cumulantAppendixD}.
\end{proof}

\section{\texorpdfstring{Minimum of $a\lambda+\frac{1}{\lambda}b$}{Minimum of a function of lambda}}\label{appn:minAlambdaBoneOnLambda}
Here we show that the minimum of $a\lambda+\frac{1}{\lambda}b$ for $\lambda >0$ is $2\sqrt{ab}$ and it occurs at $\sqrt{\frac{b}{a}}$. We use this in \Cref{subsec:KLdiv,subsec:reverseKLdiv}.
\begin{proof}
Let $f(\lambda) = a\lambda+\frac{1}{\lambda}b$. Then \[f'(\lambda)= a -b\frac{1}{\lambda^2}.\] So $f'(\lambda) = 0$ when $\lambda =\sqrt{\frac{b}{a}}$ and $f(\sqrt{\frac{b}{a}}) = 2\sqrt{ab}$. Also, \[f''(\lambda) = 2b\frac{1}{\lambda^3} >0,\] for every $\lambda$ which means that $\lambda =\sqrt{\frac{b}{a}}$ is the global minimum.
\end{proof}
\begin{figure} \centering
\definecolor{ududff}{rgb}{0.30196078431372547,0.30196078431372547,1.}
\definecolor{qqwuqq}{rgb}{0.,0.39215686274509803,0.}
\definecolor{cqcqcq}{rgb}{0.7529411764705882,0.7529411764705882,0.7529411764705882}
\begin{tikzpicture}[line cap=round,line join=round,>=triangle 45,x=1.0cm,y=1.0cm]
\draw [color=cqcqcq,, xstep=1.0cm,ystep=1.0cm] (-0.76,-0.48) grid (10.84,8.78);
\draw[->,color=black] (-0.76,0.) -- (10.84,0.);
\foreach \x in {,1.,2.,3.,4.,5.,6.,7.,8.,9.,10.}
\draw[shift={(\x,0)},color=black] (0pt,2pt) -- (0pt,-2pt) node[below] {\footnotesize $\x$};
\draw[->,color=black] (0.,-0.48) -- (0.,8.78);
\foreach \y in {,1.,2.,3.,4.,5.,6.,7.,8.}
\draw[shift={(0,\y)},color=black] (2pt,0pt) -- (-2pt,0pt) node[left] {\footnotesize $\y$};
\draw[color=black] (0pt,-10pt) node[right] {\footnotesize $0$};
\clip(-0.76,-0.48) rectangle (10.84,8.78);
\draw[line width=2.pt,color=qqwuqq,smooth,samples=100,domain=-0.7600000000000002:10.840000000000003] plot(\x,{(\x)+2.0/(\x)});
\begin{scriptsize}
\draw[color=qqwuqq] (1.12,7.59) node {$\lambda+\frac{2}{\lambda}$};
\draw [fill=ududff] (1.4142135623730951,2.8284271247461903) circle (2.5pt);
\draw[color=ududff] (2.22,2.43) node {$(\sqrt{2}, 2\sqrt{2})$};
\end{scriptsize}
\end{tikzpicture}
\caption{Graph of $\lambda+2\frac{1}{\lambda}$, which has minimum at $(\sqrt{2},2\sqrt{2}$).}
\end{figure}
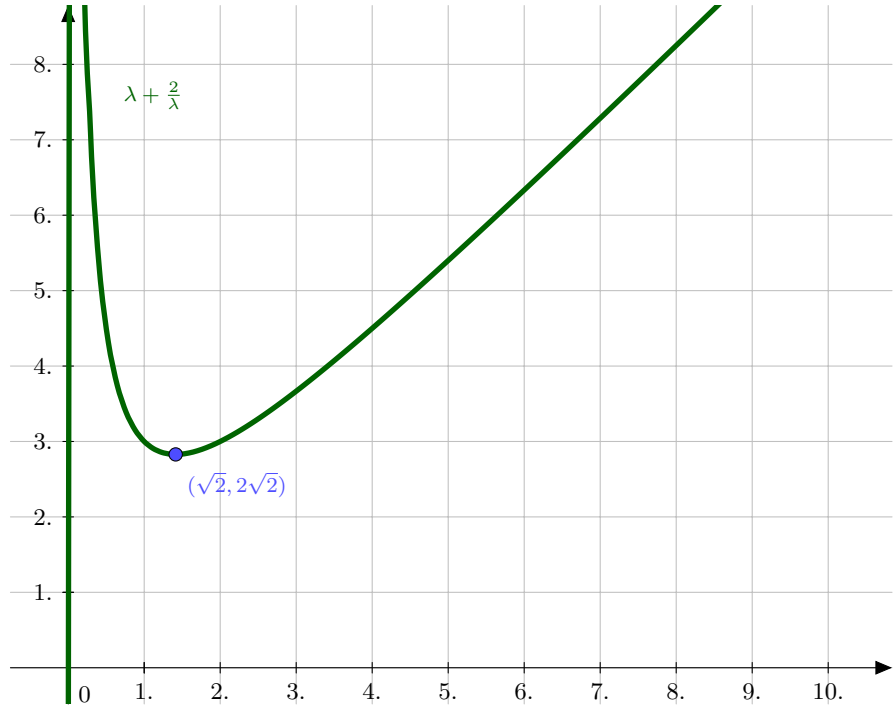

\section{\texorpdfstring{Minimum of $x\log(\frac{x}{a})$}{Minimum of a log function}}\label{appn:minofLogfunction}
Here we show that the minimum of $x\log(\frac{x}{a})$ for $x >0$ is $-\frac{a}{e}$ and it occurs at $x=\frac{a}{e}$. We use this in \Cref{appn:alternativeKLdiv}.
\begin{proof}
Let $g(x) = x\log(\frac{x}{a})$. Then \[ g'(x) = \log(\frac{x}{a}) + 1\] and $g'(x) = 0$ when $x = \frac{a}{e}$ and $g(\frac{a}{e}) = -\frac{a}{e}$. Also \[g''(x) = \frac{1}{x^2}\ge 0\] so $x=\frac{a}{e}$ is a global minimum.
\end{proof}.

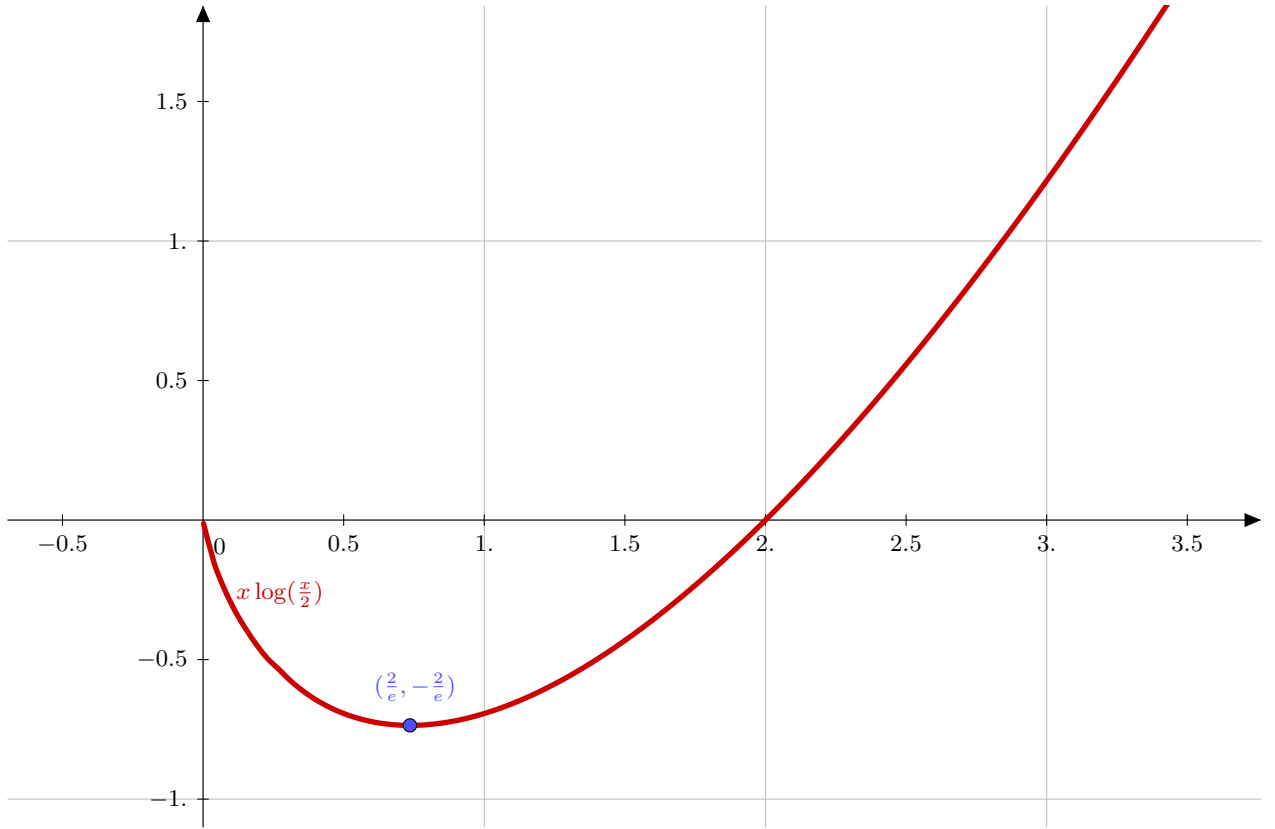
\begin{figure}
    \centering
\definecolor{ududff}{rgb}{0.30196078431372547,0.30196078431372547,1.}
\definecolor{ccqqqq}{rgb}{0.8,0.,0.}
\definecolor{cqcqcq}{rgb}{0.7529411764705882,0.7529411764705882,0.7529411764705882}
\begin{tikzpicture}[line cap=round,line join=round,>=triangle 45,x=3.7cm,y=3.7cm]
\draw [color=cqcqcq,, xstep=3.7cm,ystep=3.7cm] (-0.6945669886362705,-1.099596459435472) grid (3.762768632530062,1.8437820628174606);
\draw[->,color=black] (-0.6945669886362705,0.) -- (3.762768632530062,0.);
\foreach \x in {-0.5,0.5,1.,1.5,2.,2.5,3.,3.5}
\draw[shift={(\x,0)},color=black] (0pt,2pt) -- (0pt,-2pt) node[below] {\footnotesize $\x$};
\draw[->,color=black] (0.,-1.099596459435472) -- (0.,1.8437820628174606);
\foreach \y in {-1.,-0.5,0.5,1.,1.5}
\draw[shift={(0,\y)},color=black] (2pt,0pt) -- (-2pt,0pt) node[left] {\footnotesize $\y$};
\draw[color=black] (0pt,-10pt) node[right] {\footnotesize $0$};
\clip(-0.6945669886362705,-1.099596459435472) rectangle (3.762768632530062,1.8437820628174606);
\draw[line width=2.pt,color=ccqqqq,smooth,samples=100,domain=1.7918675944066492E-3:3.762768632530062] plot(\x,{(\x)*ln((\x)/2.0)});
\begin{scriptsize}
\draw[color=ccqqqq] (0.27356234656594076,-0.2645019418697383) node {\footnotesize $x\log(\frac{x}{2})$};
\draw [fill=ududff] (0.7357588823428847,-0.7357588823428847) circle (2.5pt);
\draw[color=ududff] (0.7538315842564183,-0.59715358555262184) node {\footnotesize $(\frac{2}{e},-\frac{2}{e})$};
\end{scriptsize}
\end{tikzpicture}
    \caption{Graph of $ x\log(\frac{x}{2})$ with minimum at $(\frac{2}{e},-\frac{2}{e})$.}
    \label{fig:my_label}
\end{figure}

\section{\texorpdfstring{Minimum of $\lVert a(b-\lambda)\rVert_1$}{Minimum of weighted span}}\label{appn:weightedspan}
Here we consider a weighted version of \Cref{appn:proofSpan} and use the $\ell_1$ norm rather than the $\ell_\infty$ norm. We assume $a$ is a positive vector in $\R^N$ and $b\in \R^N$, for a positive integer $N$.

In \Cref{fig:weightedspan}, note that $\lVert a(b-\lambda)\rVert_1$ is convex, so there is a unique minimum. As it is also piece-wise linear, then the minimum must occur at $\lambda = b_s$ for some $s=1,\ldots, N$. 
\begin{lemma}
Assume $a$ is a positive vector in $\R^N$ and $b\in \R^N$, for a positive integer $N$. Then reorder the vectors $a$ and $b$ so that the smallest element of $b$ is $b_1$ and the largest is $b_N$, with $a$ corresponding to the correct $b$. Then find the smallest $i$ such that \[\sum_{s=1}^ia_i\ge \frac{1}{2}\sum_{s=1}^Na_i.\] Then the minimum of $\lVert a(b-\lambda)\rVert_1$ occurs at $\lambda = b_i$.

If $b$ are non-negative and we require $\lambda\le0$, then the minimum occurs at $\lambda = 0$.
\end{lemma}

\begin{proof}
Say $b_k\le \lambda\le b_{k+1}$ for some $k=1,2,\ldots, N-1$ then 
\begin{align*} \lVert a(b-\lambda)\rVert_1 & = \sum_{s}|a_s(b_s-\lambda)|\\
&=  \sum_{s=1}^ka_s(\lambda-b_s) + \sum_{s=k+1}^Na_s(b_s-\lambda)\\
&=\lambda ( \sum_{s=1}^ka_s- \sum_{s=k+1}^ka_s) - \sum_{s=1}^ka_sb_s + \sum_{s=k+1}^Na_sb_s\\
&=\lambda ( 2\sum_{s=1}^ka_s- \sum_{s=N}^ka_s) + 2\sum_{s=k+1}^Na_sb_s - \sum_{s=1}^Na_sb_s\\
&=2\lambda ( \sum_{s=1}^ka_s- \frac{1}{2}\sum_{s=1}^Na_s) + 2\sum_{s=k+1}^Na_sb_s - \sum_{s=1}^Na_sb_s\\
\end{align*}
Then the minimum over $\lambda$ in $[b_k, b_{k+1}]$ occurs at $b_k$ whenever $ \sum_{s=1}^ka_s> \frac{1}{2}\sum_{s=1}^Na_s$, and it occurs at $b_{k+1}$ otherwise. As the function is piecewise linear and convex, the minimum must occur at the only $k$ such that  $ \sum_{s=1}^{k-1}a_s< \frac{1}{2}\sum_{s=1}^Na_s$ and  $ \sum_{s=1}^ka_s> \frac{1}{2}\sum_{s=1}^Na_s$, which is equivalent to the smallest $k$ such that \[\sum_{s=1}^ka_k\ge \frac{1}{2}\sum_{s=1}^Na_k,\] as required.

The second point follows by convexity of $\lVert a(b-\lambda)\rVert_1$.
\end{proof}
 
\begin{figure}
\definecolor{xdxdff}{rgb}{0.49019607843137253,0.49019607843137253,1.}
\definecolor{qqwuqq}{rgb}{0.,0.39215686274509803,0.}
\definecolor{cqcqcq}{rgb}{0.7529411764705882,0.7529411764705882,0.7529411764705882}
\begin{tikzpicture}[line cap=round,line join=round,>=triangle 45,x=1.8cm,y=1.8cm]
\draw [color=cqcqcq,, xstep=1.8cm,ystep=1.8cm] (-0.8091953776251293,-0.6640320560981714) grid (7.290693713999502,5.961848949948123);
\draw[->,color=black] (-0.8091953776251293,0.) -- (7.290693713999502,0.);
\foreach \x in {,1.,2.,3.,4.,5.,6.,7.}
\draw[shift={(\x,0)},color=black] (0pt,2pt) -- (0pt,-2pt) node[below] {\footnotesize $\x$};
\draw[->,color=black] (0.,-0.6640320560981714) -- (0.,5.961848949948123);
\foreach \y in {,1.,2.,3.,4.,5.}
\draw[shift={(0,\y)},color=black] (2pt,0pt) -- (-2pt,0pt) node[left] {\footnotesize $\y$};
\draw[color=black] (0pt,-10pt) node[right] {\footnotesize $0$};
\clip(-0.8091953776251293,-0.6640320560981714) rectangle (7.290693713999502,5.961848949948123);
\draw[line width=2.pt,color=qqwuqq,smooth,samples=100,domain=-0.8091953776251293:7.290693713999502] plot(\x,{0.3*abs((\x)-1.0)+0.2*abs((\x)-5.0)+0.2*abs((\x)-3.0)+0.4*abs((\x)-6.0)});
\begin{scriptsize}
\draw[color=qqwuqq] (2.425035955779759,4.437753210976351) node {$0.3|x - 1| + 0.2|x - 5| + 0.2|x - 3| + 0.4|x - 6|$};
\draw [fill=xdxdff] (1.,3.2) circle (2.5pt);
\draw[color=xdxdff] (0.8222310471897261,2.805976172587743) node {$(1,3.2)$};
\draw [fill=xdxdff] (3.,2.2) circle (2.5pt);
\draw[color=xdxdff] (2.7541833923652126,1.8900876533934386) node {$(3,2.2)$};
\draw [fill=xdxdff] (5.,2.) circle (2.5pt);
\draw[color=xdxdff] (4.8721755930020425,1.7612908303817394) node {$(5,2)$};
\draw [fill=xdxdff] (6.,2.3) circle (2.5pt);
\draw[color=xdxdff] (6.474980501592075,1.99331952345175486) node {$(6,2.3)$};
\end{scriptsize}
\end{tikzpicture}
\caption{Plot of $\lVert a(b-\lambda)\rVert_1$ with $b=(1,3,5,6)$ and $a=(0.3,0.2,0.2,0.4)$. We can see the minimum occurs at $(5,2)$ and that $\sum_{s=1}^ia_i\ge \frac{1}{2}\sum_{s=1}^Na_i$ is satisfied for $i=3$, $b_i=5$.}\label{fig:weightedspan}
\end{figure}
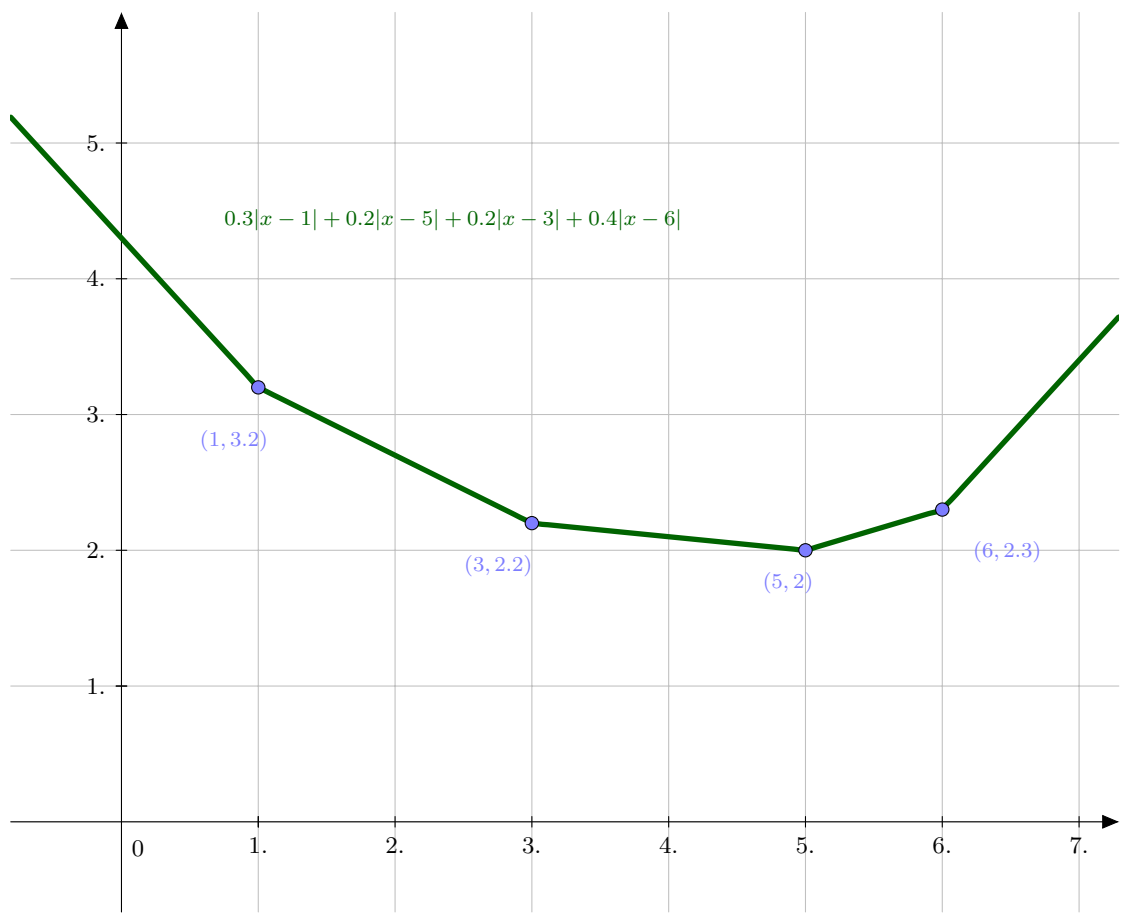
 
\end{document}